\documentclass{article}

\usepackage{microtype}
\usepackage{graphicx}
\usepackage{subcaption}
\usepackage{booktabs} %

\usepackage{hyperref}

\usepackage[preprint]{icml}

\usepackage{amsmath}
\usepackage{amssymb}
\usepackage{mathtools}
\usepackage{amsthm}

\usepackage[capitalize,noabbrev]{cleveref}

\theoremstyle{plain}
\newtheorem{theorem}{Theorem}[section]
\newtheorem{proposition}[theorem]{Proposition}
\newtheorem{lemma}[theorem]{Lemma}

\theoremstyle{definition}

\theoremstyle{remark}

\usepackage[textsize=tiny]{todonotes}

\usepackage[table]{xcolor}

\colorlet{lgreen}{green!10}
\colorlet{lblue}{blue!10}
\colorlet{lred}{red!10}

\usepackage{listings}
\definecolor{Numbers}{RGB}{0,137,84}      %
\definecolor{Definitions}{RGB}{0,0,255}   %
\definecolor{Functions}{RGB}{126,93,23}   %
\definecolor{Variables}{RGB}{0,17,134}    %
\definecolor{Comments}{RGB}{0,100,0}      %
\definecolor{Strings}{RGB}{178,0,2}       %
\definecolor{Operators}{RGB}{0,0,0}       %
\definecolor{BaseFunctions}{RGB}{192,0,227} %
\definecolor{Bakgroung}{RGB}{255,255,255} %

\lstdefinelanguage{PythonVSCodeLight}{
    language=Python,
    morekeywords={
        torch, nn, autograd, optim, utils, cuda, device, Tensor, einops, opt_einsum,
        tensor, from_numpy, zeros, ones, randn, manual_seed, no_grad, save, load,
        Linear, Conv2d, ReLU, Sigmoid, Tanh, Dropout, Sequential, Module,
        SGD, Adam, lr_scheduler,
        CrossEntropyLoss, MSELoss, BCELoss, NLLLoss,
    },
    keywords=[2]{abs, norm, kthvalue, where, einsum, rearrange, contract_path, svd, sort, view, reshape, reshape_as, flip, cumsum, square, std, expand_as, gather, pow, sign, view_as, flatten, clamp, lu, solve_triangular, allclose, mse_loss, contract_expression, pop, get, AdamW, zero_grad, backward, step, item, copy_},
    keywordstyle=[2]\color{Functions}, %
    sensitive=true
}

\usepackage{tikz}
\usetikzlibrary{tikzmark, shapes.geometric, arrows.meta, fit, positioning, calc, matrix}

\NewDocumentCommand{\progressbar}{m O{2} O{red!70} O{0.3}}{%
  \begin{tikzpicture}
    \fill[gray!30] (0,0) rectangle (#2,#4); %
    \fill[#3] (0,0) rectangle ({#1*#2},#4); %
  \end{tikzpicture}%
}

\icmltitlerunning{EinSort: Sorting is All We Need for Tensorizing LLM}

\begin{document}

\twocolumn[
  \icmltitle{EinSort: Sorting is All We Need for Tensorizing LLM}

  \icmlsetsymbol{equal}{*}

  \begin{icmlauthorlist}
    \icmlauthor{Toshiaki Koike-Akino}{equal,merl}
    \icmlauthor{Jing Liu}{merl}
    \icmlauthor{Ye Wang}{merl}
  \end{icmlauthorlist}

  \icmlaffiliation{merl}{Mitsubishi Electric Research Laboratories (MERL), 201 Broadway, Cambridge, MA 02139, USA.}

  \icmlcorrespondingauthor{Toshiaki Koike-Akino}{koike@merl.com}

  \icmlkeywords{Machine Learning, LLM, Tensor Decomposition, Tensor Networks}

  \vskip 0.3in
]

\printAffiliationsAndNotice{}  %

\begin{abstract}
Tensor networks provide efficient representations for compressing large neural networks. 
By carefully designing shapes and topologies, they can significantly reduce memory and computational costs. 
However, identifying implicit low-rank structures in large foundation models remains challenging due to their enormous scale and unstructured weight distributions.
We propose an adaptive tensorization method that discovers inherent low-rank structure in a target tensor by index ordering. 
Experiments on weight and KV-cache compression demonstrate improved reconstruction quality compared to baselines.
\end{abstract}

\section{Introduction}

Large foundation models achieve remarkable performance across a variety
of  tasks~\citep{katz2024gpt}. 
Nonetheless, their large parameter counts and memory requirements make deployment expensive~\citep{schwartz2020green}. 
Numerous compression techniques have therefore been proposed~\citep{xu2023survey, zhu2024survey, bai2024beyond}, including partial activation~\citep{lin2024moe}, weight pruning~\citep{bai2024sparsellm},
quantization~\citep{wang2024q}, knowledge distillation~\citep{hwang2024pc}, and rank reduction~\citep{yuan2023asvd, hwang2024pc, saxena2024eigen}.

Among them, tensor decomposition methods~\citep{lebedev2014speeding, sidiropoulos2017tensor} provide a flexible framework for representing large neural networks with significantly fewer parameters~\citep{novikov2015tensorizing,denil2013predicting, sainath2013low, jaderberg2014speeding}.
Existing tensor network approaches~\citep{roberts2019tensornetwork,huggins2019towards} primarily focus on optimizing topology, contraction, and tensor ranks~\citep{luo2024adaptive}.
However, one important degree of freedom has received relatively little attention: the ordering of tensor indices.
In this work, we show that index ordering can dramatically alter the effective rank structure of tensors. In particular, simple sorting operations can expose latent low-rank structure that is otherwise hidden in pretrained LLM weights and KV caches. Motivated by this observation, we propose Einstein sorted sum (\textbf{EinSort}), a tensorization framework that augments Einstein summation with reversible permutation operators.

Our contributions are summarized as follows:
(1) We introduce EinSort, a tensor decomposition framework with adaptive index ordering for LLM compression.
(2) We theoretically and empirically demonstrate that sorting operations can substantially reduce effective tensor rank.
(3) We propose practical low-overhead permutation schemes for tensor compression.
(4) We demonstrate improved compression quality on LLM weight and KV-cache compression tasks.

\section{Tensorizing LLM for compression}

\begin{figure}[t]
    \centering
    \includegraphics[width=0.9\linewidth]{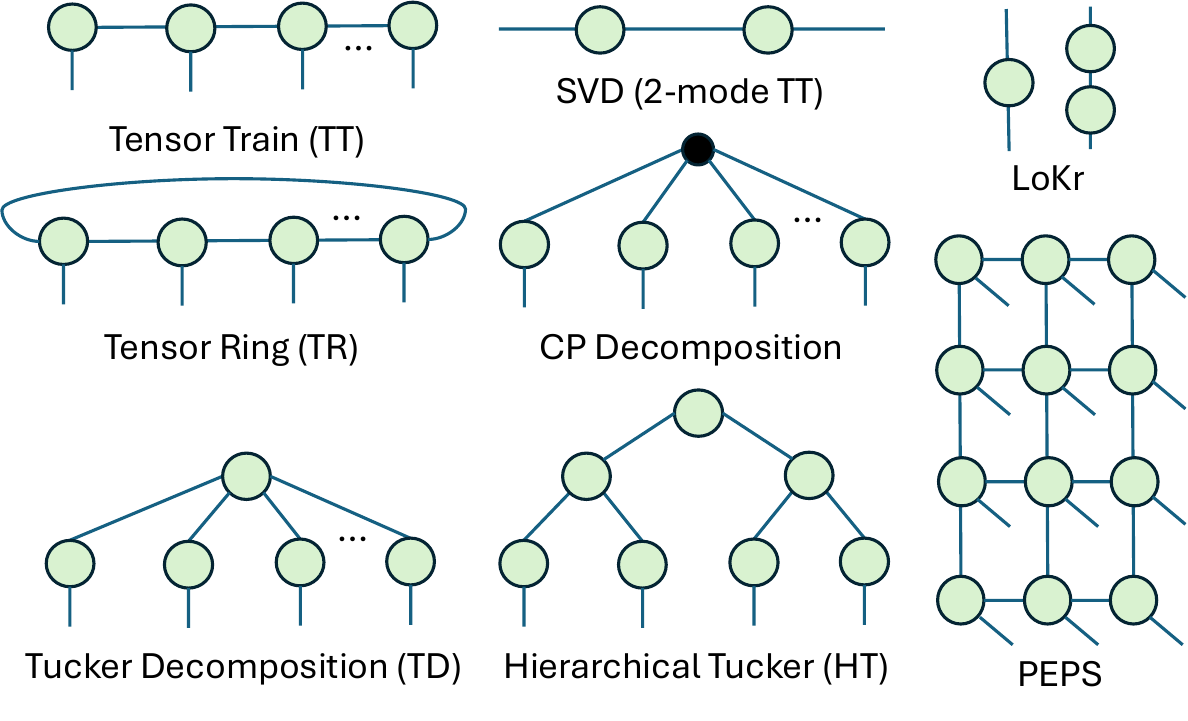}
    \caption{Various tensor networks in tensor diagrams.}
    \label{fig:tensornet}
\end{figure}

\subsection{Einstein summation as general tensor networks}

Model decomposition is a widely used approach to compress large models, e.g., low-rank factorization via singular-value decomposition (SVD).
More generally, tensorizing models can be viewed as a tensor network framework, including tensor train (TT), tensor ring (TR), canonical polyadic (CP), Tucker decomposition (TD), hierarchical Tucker (HT), and projected entangled-pair states (PEPS) as depicted in Fig~\ref{fig:tensornet}.
Almost all tensor networks can be represented by {Einstein summation} (\textbf{einsum}), which is a powerful format for linear algebra.
Given tensor cores with proper shapes, a tensor network is factorized in the einsum form, e.g., for 4-mode TT:
$
X_{i,j,k,l} =
\sum_{p,q,r} A_{i,p} B_{p,j,q} C_{q,k,r} D_{r,l},
$
where $A_{i,p}$, $B_{p,j,q}$, $C_{q,k,r}$, and $D_{r,l}$ are factorized cores.
The set $\{i,j,k,l\}$ are physical site indices, and $\{p,q,r\}$ are bond indices whose maximum ranges are called as tensor ranks.
If we reduce the tensor ranks, it can considerably reduce the memory footprint to represent a tensor $X$.
For example, if $X$ is of shape $(d,d,d,d)$, the total number of parameters is a quartic order of $d^{4}$, while it can be reduced to $2d(R+1)R$ for the maximum rank of $R$. 
When $R=\rho d$ for $0< \rho<1$, it will be reduced to a cubic order of $2\rho d^2(\rho d+1)$.

Several examples are listed in the pseudo code below:
\begin{lstlisting}
TT = torch.einsum("ip,pjq,qkr,rl->ijkl", A, B, C, D)
TR = torch.einsum("mip,pjq,qkr,rlm->ijkl", A, B, C, D)
CP = torch.einsum("r,ri,rj,rk->ijk", A, B, C, D)
TD = torch.einsum("pqr,pi,qj,rk->ijk", A, B, C, D)
X = einops.rearrange("i p j q -> (i p) (j q)", TT) # unfolding 4-mode to 2-mode
\end{lstlisting}
Here we use \href{https://einops.rocks/} {\texttt{einops}} for folding/unfolding. 
See Appendix~\ref{sec:einsum} for more examples.
As shown above, the einsum is a powerful tool to manipulate multi-linear operations.
It gives an opportunity to optimize topologies for tensor networks by just designing the einsum equation (a string like \texttt{"ip,pj->ij"}) with properly shaped tensor cores.

In addition to the parameter reduction, the computational cost can be also reduced by designing the contraction order.
For example with the 4-mode TT for a site dimension of $d=32$ with a rank of $R=20$, a na\"{i}ve contraction from left to right requires $3.4\times 10^{10}$ FLOPs, whereas an optimized contraction order reduces it to $4.4\times 10^{7}$, nearly 3 orders of magnitude when using \href{https://optimized-einsum.readthedocs.io/en/stable/index.html}{\texttt{opt\_einsum}}.

\subsection{Gauge fixing}

Given a large tensor, we can numerically factorize into any tensor networks, while exact solutions generally do not exist, except in special cases such as matrix SVD.
As einsum is a multi-linear operation, gradient optimization can be used to minimize the tensorization error.
For TT, a recursive SVD per bond cutting offers a good initialization, while alternating least-squares (ALS) is often used to refine the estimate.
We use \href{https://tensorly.org/stable/index.html}{\texttt{tensorly}}.

Actually, the tensorization solution is not unique because they have \textbf{gauge} freedom~\citep{evenbly2018gauge}. 
For example, consider factorizing $X$ as $X\simeq A B$ for low-rank cores $A$ and $B$. 
Optimal factors are given by truncated SVD, i.e., $A$ is based on left-singular and $B$ is on right-singular vectors of $X$.
However, any full-rank junction matrix $J$ has no impact when injecting it together with its inverse, i.e., $AB = A (J J^{-1}) B = (A J) (J^{-1} B) = A' B'$.
Hence, there are infinite solutions for $A, B$.
By fixing the gauge freedom such as canonicalization, the numerical stability~\citep{phan2020stable} for tensor factorization can be improved.

More importantly, such gauge fixing can notably reduce the number of parameters when using a specific junction~\citep{koike2026latentllm}.
For instance, we can choose a junction to convert adjacent tensor cores into block identity form.
In consequence, the number of parameters for 4-mode TT will be further reduced to $2d(R+1)R-2R^2$ for example.

\begin{figure}[t]
    \centering
    \includegraphics[width=\linewidth]{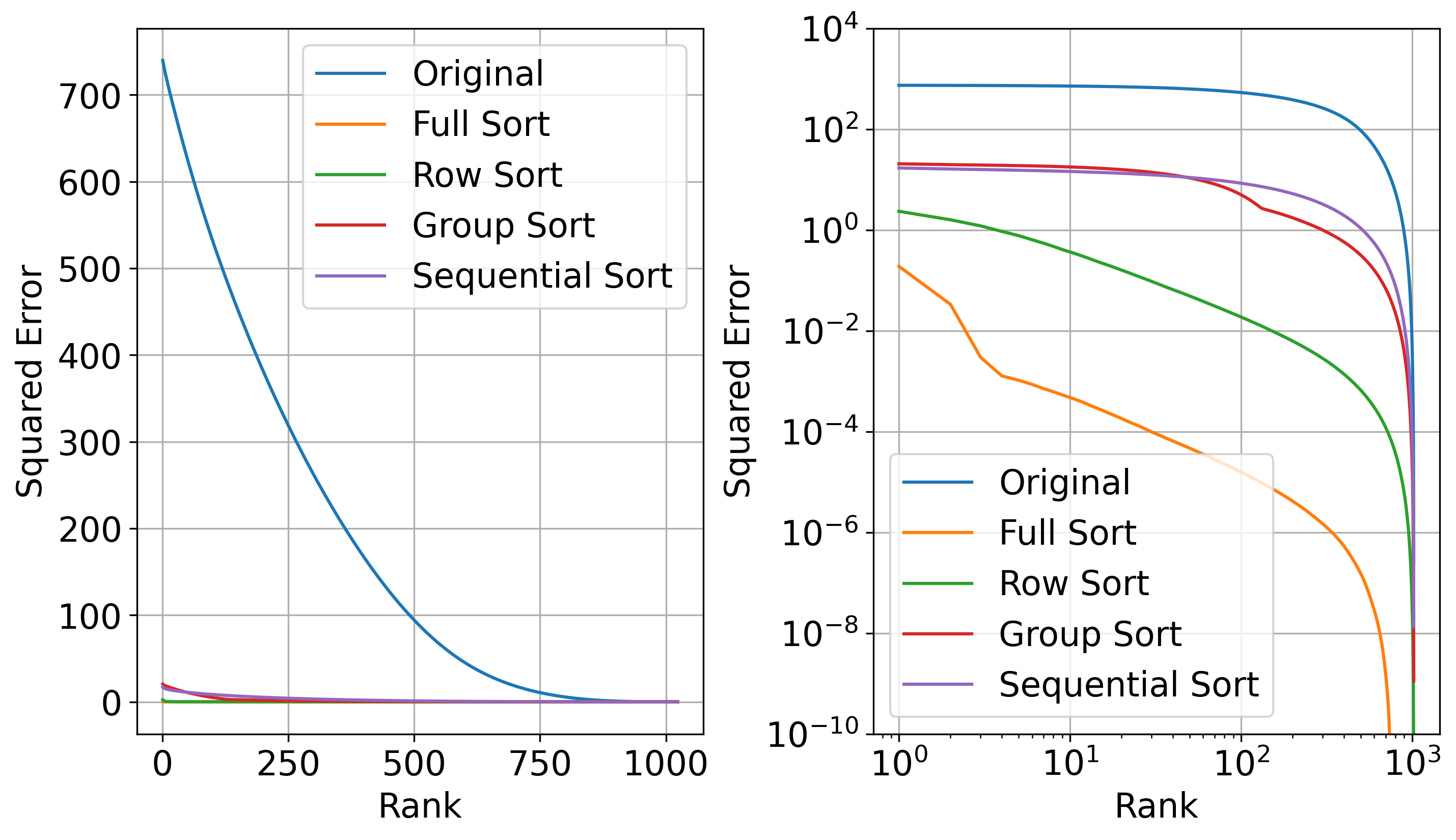}
    \caption{Tensorization error vs.\ rank for v\_proj weights at the first layer of Qwen3-0.6B model.
    Left: linear; right: log-log plots.}
    \label{fig:spect}
\end{figure}

\subsection{Index ordering}

Using einsum, we can adaptively design tensor networks to factorize LLMs: determining einsum equation, tensor core shapes, folding/unfolding operations, numerical solvers, and gauge fixing for a target tensor to decompose (e.g., pretrained weights).
However, there are few papers on tensor network designs, which consider optimizing tensor indexing beyond a simple axis ordering like \texttt{"ijkl->kjli"}. 
In this paper, we show that index ordering plays a critical role in exposing low-rank structure in random tensors.

Let us first consider a toy example, assuming a $2\times 2$ tensor $W_0$ and its re-ordered version $W_\pi$:
\vspace{-1em}
\begin{align}
    W_0 =
    \begin{tikzpicture}[baseline=(m.center),remember picture]
    \node (m) at (0,0) {
    $
    \begin{bmatrix}
        2 & 1 \\
        \tikzmarknode{a}{1} & \tikzmarknode{b}{2}
    \end{bmatrix}
    $
    };
    \draw[->, blue] (-0.6,0.5) -- (-0.6,-0.5) node[midway,left] {$i$};
    \draw[->, blue] (-0.5,0.6) -- (0.5,0.6) node[midway,above] {$j$};
      \draw[<->, red, thick, bend left=-35]
    ([yshift=-2pt]a.south) to node[below, sloped] {\scriptsize swap} ([yshift=-2pt]b.south);
    \end{tikzpicture}
    ,
    \qquad
    W_\pi =
    \begin{tikzpicture}[baseline=(m.center)]
    \node (m) at (0,0) {
    $
    \begin{bmatrix}
        2 & 1 \\
        2 & 1
    \end{bmatrix}
    $
    };
    \draw[->, blue] (-0.6,0.5) -- (-0.6,-0.5) node[midway,left] {$i$};
    \draw[->, blue] (-0.5,0.6) -- (0.5,0.6) node[midway,above] {$j\oplus i$};
    \end{tikzpicture}
    .
    \label{eq:2x2}
\end{align}
The first matrix has singular values of $[3,1]$, while the second one has $[\sqrt{10},0]$ and a reduced-rank of $1$.
Here $W_\pi$ has an index re-ordering of $j\mapsto j\oplus i$ for indices $(i,j)$.
In fact, this index ordering is known as \textbf{entanglement} operations in quantum tensor networks and the above example is exactly same as controlled-NOT (CNOT) gates.
Appendix~\ref{sec:quantum} discusses their connections in more details.

We next show empirical results suggesting that the sorting operation can greatly reduce the required rank.
Fig.~\ref{fig:spect} shows the tensorization loss in squared norm $\|W-\hat{W}\|^2$, where $W$ is the original weights matrix and $\hat{W}$ is a reconstructed version through a 2-mode TT.
Specifically, we analyze the eigen spectrum for pretrained weights of an LLM.
We also evaluate two sorting options, as given in the pseudo code:
\begin{lstlisting}
S = torch.linalg.svdvals(W)
loss = S.square().flip(-1).cumsum(-1)
Ws, _ = torch.sort(W.view(-1)) # full-sort
S = torch.linalg.svdvals(Ws.reshape_as(W))
Wr, _ = torch.sort(W, dim=-1) # row-sort
\end{lstlisting}

As shown in Fig.~\ref{fig:spect}, increasing the rank improves the accuracy, while the improvement is slow when original weights are decomposed by SVD without sorting.
Notably, tensor sorting dramatically reduces reconstruction error, enabling near rank-1 reconstruction.
From log-log plot, we can see that row sorting 
has relatively higher error than full sorting, whereas it is still more than 10 times lower than the one without sorting. 
Therefore, even with imperfect sorting, the index ordering has a great potential to improve the accuracy.

\subsection{Theoretical analysis}

We next provide a theoretical foundation supporting the reason why sorting can induce low-rank structure.

\begin{lemma}
\label{th:uniform}
Let $X_0,\ldots,X_{L-1}$ be i.i.d. random variables uniformly distributed on
$[a,b]$, where $a<b$. 
Let $X_{(k)}$ for $k\in\mathbb{Z}_L$ denote their order statistics as
$
    X_{(0)}\leq X_{(1)}\leq \cdots \leq X_{(L-1)}$,
and reshape them into an $M\times N$ matrix
$X$ by
$
X_{i,j}=X_{(Ni+j)}
$, 
for $i\in\mathbb{Z}_M$,
$j\in\mathbb{Z}_N$,
and $L=MN$. 
Then $X$ is asymptotically approximated, in
probability, by a matrix of rank at most $3$. 
\end{lemma}

See proof in Appendix~\ref{sec:theory}, which uses a well-known order statistics~\citep{david2004order} to derive the mean and variance.
As its mean is linear and the variance is bounded, we can find that a sorted matrix from uniformly distributed random variables is at most rank of $3$ asymptotically, up to vanishing stochastic error.
Although this does not guarantee that an arbitrary random matrix admits a low-rank factorization, we discuss more general cases in Appendix~\ref{sec:theory}.

\subsection{EinSort: Index ordered tensor networks}

Those results motivate us to consider sorting for tensor network designs.
We then propose \textbf{EinSort} framework.
Conceptually, we use a reversible  permutation operation $\pi[\cdot]$ inside tensor decomposition.
Let $\mathcal{T}_\theta[\cdot]$ denote the tensor decomposition with hyperparameters $\theta$ which determine einsum equation, shapes, topology, etc.
The sorted tensor decomposition is written as $\mathcal{T}_\theta^\pi :=\pi^{-1}[\mathcal{T}_\theta[\pi[X]]]$.
It generalizes the einsum to find low-rank structure.

Although sorting may reduce the number of parameters for tensor cores with few rank, the memory overhead of storing permutations is not negligible.
In fact, sorting $L$ values requires at most $\lceil \log_2(L!) \rceil$ bits using factoradic/Lehmer code~\citep{lehmer1960teaching}.
For example, full sorting for a tensor of shape $(1024, 1024)$ requires roughly $18.6$ bits per entry, which often exceeds the memory footprint of the original tensors in FP16.
(Note that the row sorting requires $8.6$ bits.)
Therefore, we need to seek optimizing a tradeoff between sorting accuracy and tensor rank reduction.

We introduce a simple sorting method reusing a shared permutation for multiple tensors.
For a 4-mode tensor $X$ with a site dimension of $d$,
we may use a shared permutation across the first mode slice: $X_{:,j,k,l}$ to re-order the last three modes.
It reduces the memory of permutation from $\lceil \log_2(d^4!)\rceil$ to $\lceil \log_2(d^3!) \rceil$ bits, i.e., more than $d$-fold memory reduction.
Therefore, we can easily adjust the total memory.
For example, with $d=32$, we have $\lceil \log_2(d^3!)\rceil / d^4 \simeq 0.4$ bits per weight.
For sorting metric, we consider some options for nonlinear mapping and reduction, including power scaling, and mean/max/min/median/std/prod operations.
An example is given in the pseudo code below.
\begin{lstlisting}
S = X.abs().pow(p) # nonlinear map score
S = S.std(dim=[0], keepdim=True).expand_as(X) # std reduction at the first mode
perm = torch.argsort(S.flatten(1), dim=-1) 
X = X.flatten(1).gather(-1,perm).view_as(S) 
\end{lstlisting}
We search for effective folding, slicing modes, power exponent and reduction options to increase the sorting accuracy.
Besides the sorting metric, we introduce an approach to improve the accuracy by employing non-negative tensorization.
Specifically, we project the original tensor into non-negative values before decomposition.
To recover the negative values, we keep the sign bit of the original tensor, requiring only one additional bit.
It is illustrated in Fig.~\ref{fig:slicing}.
More details and example pseudo codes are found in Appendix~\ref{sec:einsort}.

\begin{figure}[t]
    \centering
    \includegraphics[width=0.85\linewidth]{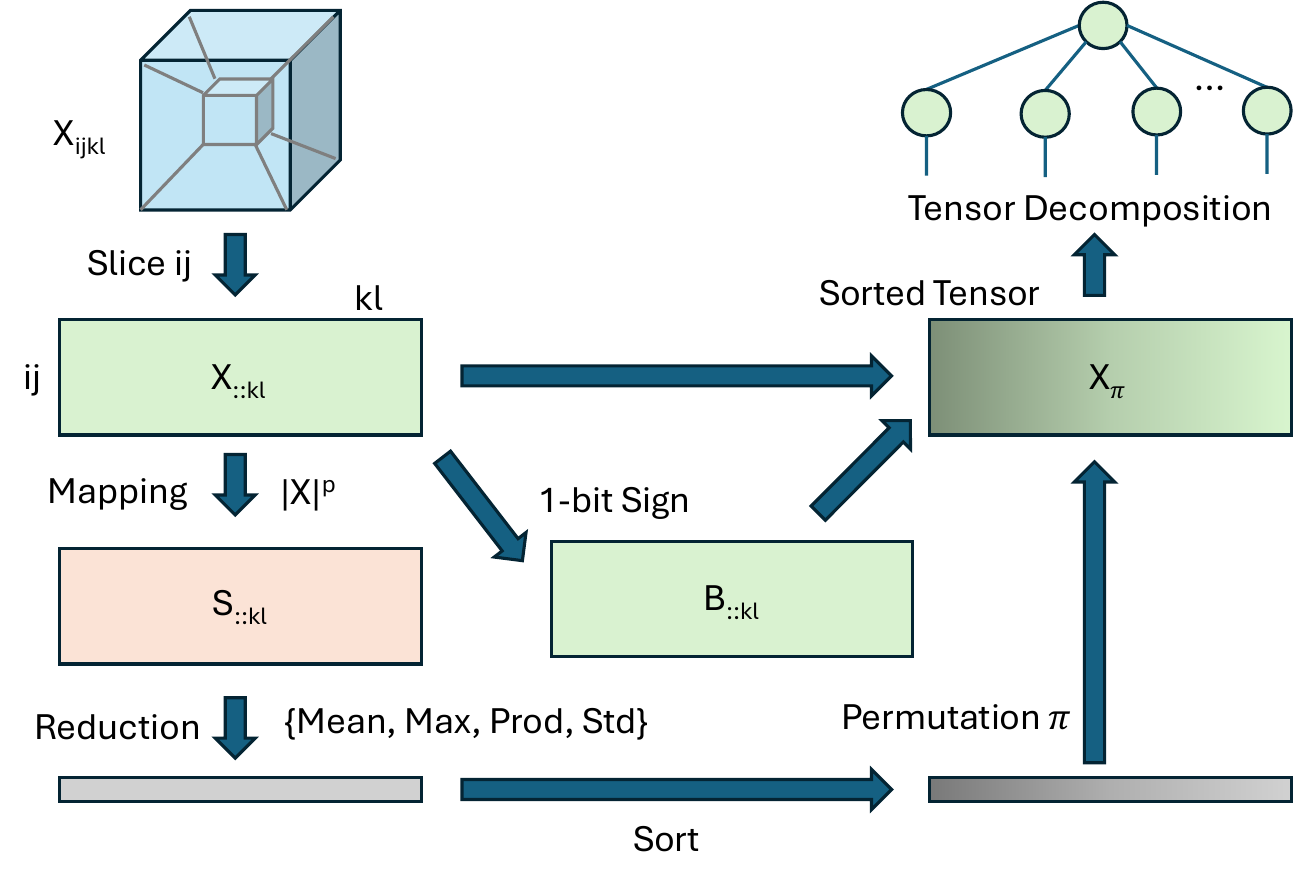}
    \caption{Sliced sorting with nonlinear mapping and reduction.
    Non-negative tensor decomposition keeps one-bit sign information.}
    \label{fig:slicing}
\end{figure}

\begin{figure}[t]
    \centering
    \includegraphics[width=\linewidth]{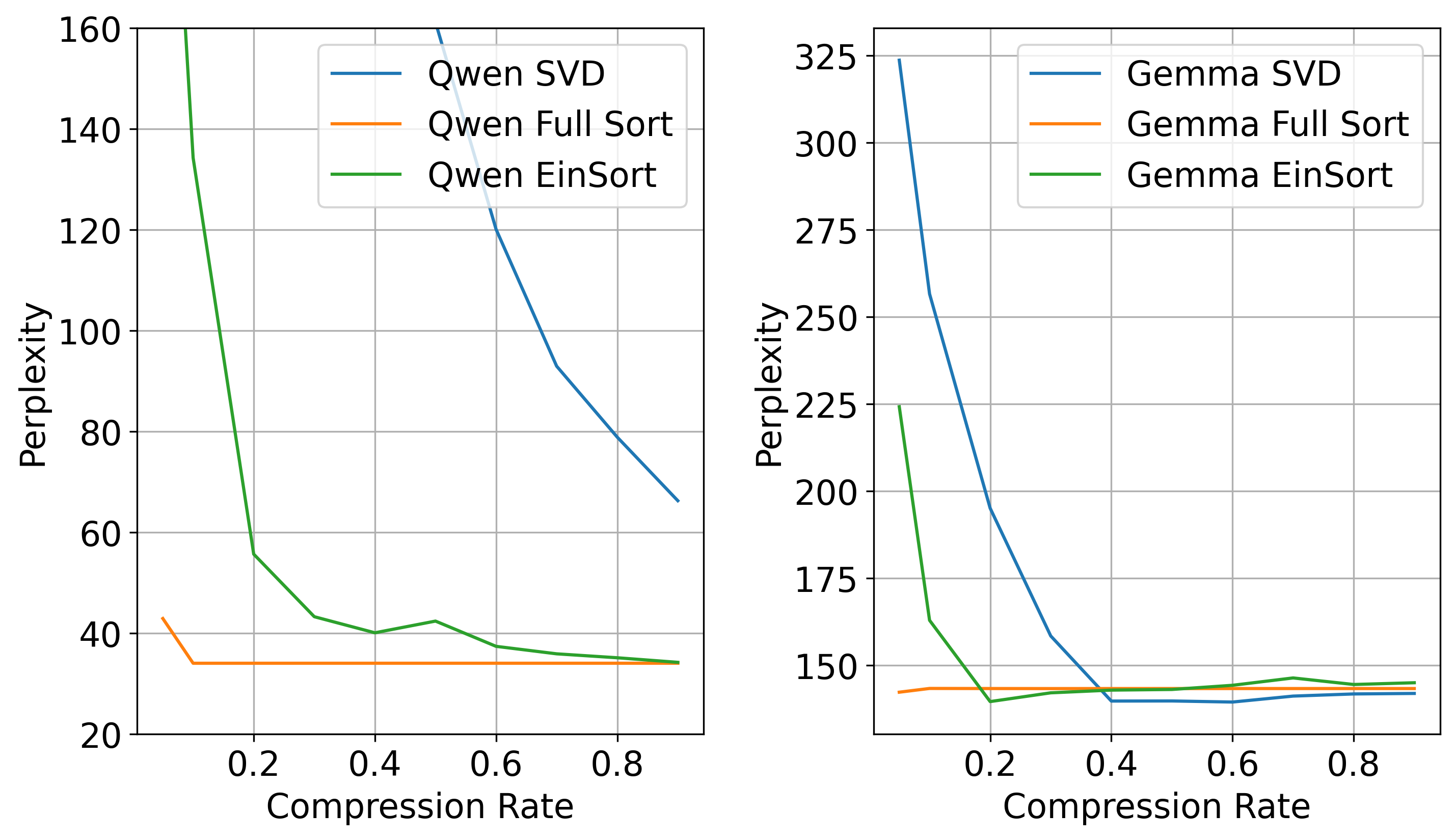}
    \caption{PPL for Qwen3-0.6B and Gemma3-4B models.}
    \label{fig:wt2}
\end{figure}

\begin{figure}[t]
    \centering
    \includegraphics[width=0.95\linewidth]{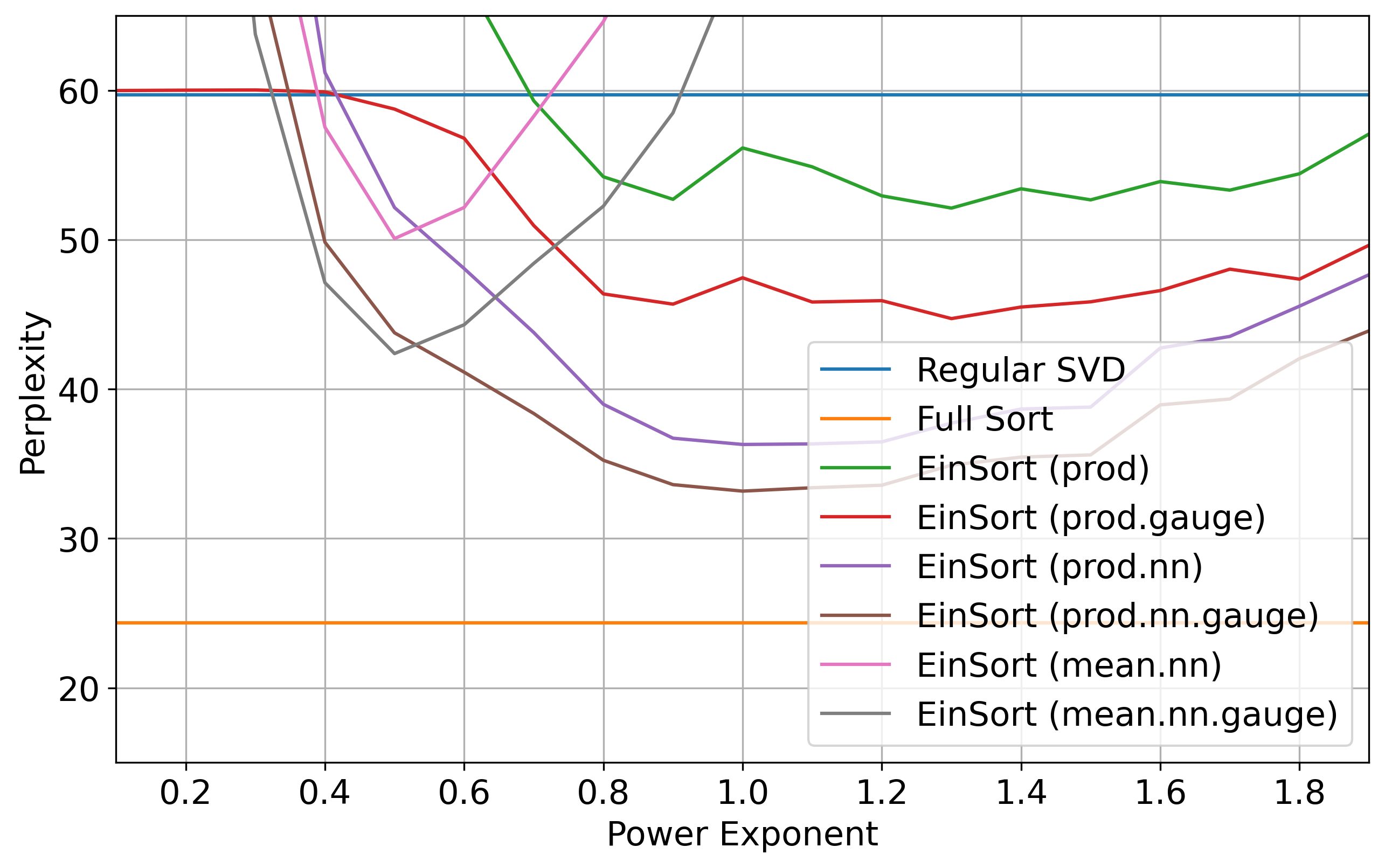}
    \caption{PPL vs.\ exponent for Qwen3-1.7B at 50\% compression.
    Gauge fixing and non-negative tensorization improve the accuracy.
    }
    \label{fig:pow}
\end{figure}

\begin{table}[t]
    \centering
    \caption{Reduction comparisons at the best power exponent.}
    \label{tab:red}
    \renewcommand{\tabcolsep}{3.5pt}
    \begin{tabular}{ccc cccc}
    \toprule
       Reduction  & mean & max & min & median & std & prod \\
       \midrule
       PPL & 42.38 & 40.25 & 52.70 & 44.21
       & 46.79 & \textbf{33.17}
         \\
       pow(p) & 0.5 & 0.5 & 0.7 & 0.5 &
     0.6 & 1.0
       \\
         \bottomrule
    \end{tabular}
\end{table}

\begin{table}[t]
    \centering
    \caption{Tensor network topology comparison (Qwen3-0.6B).}
    \label{tab:tn}
    \begin{tabular}{cc c c c c}
    \toprule
     Compression & SVD & TT & TR & CP & TD \\
      \midrule
     90\% & 47.38   & 27.04 & 36.94 & \textbf{26.88} & 27.04
      \\
      80\% & 53.24 & 29.81 & 36.95 & \textbf{27.99} & 29.81
      \\
      \bottomrule
    \end{tabular}
\end{table}

\begin{table}[t]
    \centering
    \caption{Runtime analysis for Qwen3-0.6B decoding throughput relative to the original LLM without KV cache decomposition.}
    \label{tab:runtime}
    \begin{tabular}{cc c c c c}
    \toprule
     Device & SVD & TT & TR & CP & TD \\
      \midrule
     CPU & 78.7\%   & 62.0\% & 62.7\% & 3.3\% & 32.5\%
      \\
      GPU & 79.6\% & 61.9\% & 55.6\% & 3.7\% & 28.6\%
      \\
      \bottomrule
    \end{tabular}
\end{table}

\section{Experiments}

We evaluate some LLMs, including Qwen3~\citep{yang2025qwen3} and Gemma3~\citep{team2025gemma} on the witkitext-2 (WT2)~\citep{merity2016pointer} benchmark.
See Appendix~\ref{sec:model} for details of LLM models, and Appendices~\ref{sec:dataset}/\ref{sec:library} for datasets and libraries we use.
We focus on KV cache decomposition to reduce memory footprint for decoding.
We partition the WT2 dataset into 64-token segments for prefilling, and then we decompose the KV cache for decoding up to 64 new tokens.
For sorting slice, we use KV head axis for reduction.
For Qwen3 models having 8 heads, the permutation memory can be at least 8-fold lower than the full sorting.
For our case, the required amount of bits is just about 1.4~bit per weights to record the shared permutation.

Fig.~\ref{fig:wt2} shows the perplexity (PPL) for Qwen3-0.6B and Gemma3-4B-it models when KV cache is compressed.
We observe that full sorting achieves nearly un-compressed performance as expected.
The regular SVD without sorting can rapidly degrade, 
whereas EinSort can achieve better PPL.
We also observe that Gemma3 has higher tolerance while the base PPL is worse than Qwen3.

Fig.~\ref{fig:pow} shows the impact of nonlinear power scaling.
We observe that the exponent near $1$ is best for product reduction, implying that an absolute mapping is sufficient without exponent.
However, we see that the mean reduction has a best exponent around $0.5$.
It is also confirmed that the gauge fixing can improve the accuracy because parameter redundancy can be eliminated.
More importantly, non-negative tensorization is effective to further improve the performance at a cost of 1-bit additional memory.
Table~\ref{tab:red} compares different reduction methods at the same setting.
It lists the PPL at the best power exponent for each reduction with gauge fixing and non-negative tensorization.
We observe that product reduction is best across the reduction methods. 
Other reductions had lower best exponents around $0.5$--$0.7$.

Table~\ref{tab:tn} compares different tensor networks with square-root mapping, product reduction, non-negative factorization and gauge fixing.
We fold KV cache into 3 mode tensor except for SVD.
For this setting, CP decomposition is slightly better than other topologies, and TR is the worst one.
Table~\ref{tab:runtime} lists the runtime results for the same setting. 
KV compression inevitably decreases the LLM decoding throughput.
For both Apple M1 CPU and NVIDIA A40 GPU devices, TT/TR had around 40\% speed down, whereas TD had about 70\% loss even though it has nearly the same performance as TT.
More importantly, CP is way slower because it requires iterative Khatri--Rao product for all cores.
Note that the current implementation prioritizes validating the compression principle rather than optimized inference kernels yet.

See Appendix~\ref{sec:bench} for more benchmark results, including different LLMs, GSM8K math reasoning~\cite{cobbe2021training}, TextVQA visual reasoning~\cite{singh2019towards} and LIBERO robot manipulation tasks~\cite{liu2023libero}.

\section{Conclusion}

We proposed a new EinSort framework, which employs index ordering for tensor decomposition to discover implicit low-rank structure for LLM compression.
We revealed that sorting is a powerful tool to reduce the tensor rank.
We then introduced a simple sorting mechanism and
demonstrated a potential advantage on some  models.

\section*{Impact Statement}

This paper's goal is to advance the field of machine
learning. 
There are many potential societal consequences of our work, none of which we feel must be specifically highlighted.

\bibliography{ref}

\begin{thebibliography}{103}
\providecommand{\natexlab}[1]{#1}
\providecommand{\url}[1]{\texttt{#1}}
\expandafter\ifx\csname urlstyle\endcsname\relax
  \providecommand{\doi}[1]{doi: #1}\else
  \providecommand{\doi}{doi: \begingroup \urlstyle{rm}\Url}\fi

\bibitem[Abdin et~al.(2024)Abdin, Aneja, Behl, Bubeck, Eldan, Gunasekar,
  Harrison, Hewett, Javaheripi, Kauffmann, et~al.]{abdin2024phi}
Abdin, M., Aneja, J., Behl, H., Bubeck, S., Eldan, R., Gunasekar, S., Harrison,
  M., Hewett, R.~J., Javaheripi, M., Kauffmann, P., et~al.
\newblock Phi-4 technical report.
\newblock \emph{arXiv preprint arXiv:2412.08905}, 2024.

\bibitem[Abouelenin et~al.(2025)Abouelenin, Ashfaq, Atkinson, Awadalla, Bach,
  Bao, Benhaim, Cai, Chaudhary, Chen, et~al.]{abouelenin2025phi}
Abouelenin, A., Ashfaq, A., Atkinson, A., Awadalla, H., Bach, N., Bao, J.,
  Benhaim, A., Cai, M., Chaudhary, V., Chen, C., et~al.
\newblock Phi-4-mini technical report: Compact yet powerful multimodal language
  models via mixture-of-{LoRAs}.
\newblock \emph{arXiv preprint arXiv:2503.01743}, 2025.

\bibitem[Abronin et~al.(2024)Abronin, Naumov, Mazur, Bystrov, Tsarova,
  Melnikov, Dolgov, Brasher, and Perelshein]{abronin2024tqcompressor}
Abronin, V., Naumov, A., Mazur, D., Bystrov, D., Tsarova, K., Melnikov, A.,
  Dolgov, S., Brasher, R., and Perelshein, M.
\newblock {TQCompressor}: improving tensor decomposition methods in neural
  networks via permutations.
\newblock In \emph{2024 IEEE 7th International Conference on Multimedia
  Information Processing and Retrieval (MIPR)}, pp.\  503--506. IEEE, 2024.

\bibitem[Aneja et~al.(2026)Aneja, Harrison, Joshi, LaBonte, Langford, and
  Salinas]{aneja2026phi}
Aneja, J., Harrison, M., Joshi, N., LaBonte, T., Langford, J., and Salinas, E.
\newblock Phi-4-reasoning-vision-{15B} technical report.
\newblock \emph{arXiv preprint arXiv:2603.03975}, 2026.

\bibitem[Arai \& Ichikawa(2025)Arai and Ichikawa]{arai2025quantization}
Arai, Y. and Ichikawa, Y.
\newblock Quantization error propagation: Revisiting layer-wise post-training
  quantization.
\newblock \emph{arXiv preprint arXiv:2504.09629}, 2025.

\bibitem[Assran et~al.(2025)Assran, Bardes, Fan, Garrido, Howes, Muckley,
  Rizvi, Roberts, Sinha, Zholus, et~al.]{assran2025v}
Assran, M., Bardes, A., Fan, D., Garrido, Q., Howes, R., Muckley, M., Rizvi,
  A., Roberts, C., Sinha, K., Zholus, A., et~al.
\newblock {V-JEPA 2}: Self-supervised video models enable understanding,
  prediction and planning.
\newblock \emph{arXiv preprint arXiv:2506.09985}, 2025.

\bibitem[Bai et~al.(2024{\natexlab{a}})Bai, Chai, Ling, Wang, Lu, Zhang, Shi,
  Yu, Zhu, Zhang, et~al.]{bai2024beyond}
Bai, G., Chai, Z., Ling, C., Wang, S., Lu, J., Zhang, N., Shi, T., Yu, Z., Zhu,
  M., Zhang, Y., et~al.
\newblock Beyond efficiency: A systematic survey of resource-efficient large
  language models.
\newblock \emph{arXiv preprint arXiv:2401.00625}, 2024{\natexlab{a}}.

\bibitem[Bai et~al.(2024{\natexlab{b}})Bai, Li, Ling, Kim, and
  Zhao]{bai2024sparsellm}
Bai, G., Li, Y., Ling, C., Kim, K., and Zhao, L.
\newblock {SparseLLM}: Towards global pruning for pre-trained language models.
\newblock \emph{arXiv preprint arXiv:2402.17946}, 2024{\natexlab{b}}.

\bibitem[Bai et~al.(2025)Bai, Cai, Chen, Chen, Chen, Cheng, Deng, Ding, Gao,
  Ge, et~al.]{bai2025qwen3}
Bai, S., Cai, Y., Chen, R., Chen, K., Chen, X., Cheng, Z., Deng, L., Ding, W.,
  Gao, C., Ge, C., et~al.
\newblock {Qwen3-VL} technical report.
\newblock \emph{arXiv preprint arXiv:2511.21631}, 2025.

\bibitem[Bershatsky et~al.(2024)Bershatsky, Cherniuk, Daulbaev, Mikhalev, and
  Oseledets]{bershatsky2024lotr}
Bershatsky, D., Cherniuk, D., Daulbaev, T., Mikhalev, A., and Oseledets, I.
\newblock {LoTR}: Low tensor rank weight adaptation.
\newblock \emph{arXiv preprint arXiv:2402.01376}, 2024.

\bibitem[Beyer et~al.(2024)Beyer, Steiner, Pinto, Kolesnikov, Wang, Salz,
  Neumann, Alabdulmohsin, Tschannen, Bugliarello, et~al.]{beyer2024paligemma}
Beyer, L., Steiner, A., Pinto, A.~S., Kolesnikov, A., Wang, X., Salz, D.,
  Neumann, M., Alabdulmohsin, I., Tschannen, M., Bugliarello, E., et~al.
\newblock {PaliGemma}: A versatile {3B} {VLM} for transfer.
\newblock \emph{arXiv preprint arXiv:2407.07726}, 2024.

\bibitem[Brayton et~al.(1984)Brayton, Hachtel, McMullen, and
  Sangiovanni-Vincentelli]{brayton1984logic}
Brayton, R.~K., Hachtel, G.~D., McMullen, C., and Sangiovanni-Vincentelli, A.
\newblock \emph{Logic minimization algorithms for {VLSI} synthesis}, volume~2.
\newblock Springer Science \& Business Media, 1984.

\bibitem[Cai et~al.(2025)Cai, Zhang, Gao, Liu, Li, Liu, Lu, Xiong, Dong, Hu,
  and Xiao]{cai2025pyramidkv}
Cai, Z., Zhang, Y., Gao, B., Liu, Y., Li, Y., Liu, T., Lu, K., Xiong, W., Dong,
  Y., Hu, J., and Xiao, W.
\newblock Pyramid{KV}: Dynamic {KV} cache compression based on pyramidal
  information funneling.
\newblock In \emph{Second Conference on Language Modeling}, 2025.
\newblock URL \url{https://openreview.net/forum?id=ayi7qezU87}.

\bibitem[Chang et~al.(2024)Chang, Lin, Lin, Chen, Hu, Wang, Huang, Ceze,
  Abdelfattah, and Wu]{chang2024palu}
Chang, C.-C., Lin, W.-C., Lin, C.-Y., Chen, C.-Y., Hu, Y.-F., Wang, P.-S.,
  Huang, N.-C., Ceze, L., Abdelfattah, M.~S., and Wu, K.-C.
\newblock Palu: Compressing {KV}-cache with low-rank projection.
\newblock \emph{arXiv preprint arXiv:2407.21118}, 2024.

\bibitem[Chavan et~al.(2023)Chavan, Liu, Gupta, Xing, and Shen]{chavan2023one}
Chavan, A., Liu, Z., Gupta, D., Xing, E., and Shen, Z.
\newblock One-for-all: Generalized {LoRA} for parameter-efficient fine-tuning.
\newblock \emph{arXiv preprint arXiv:2306.07967}, 2023.

\bibitem[Chen et~al.(2024{\natexlab{a}})Chen, Liu, Wang, Wang, Brand, Wang, and
  Koike-Akino]{chen2024superlora}
Chen, X., Liu, J., Wang, Y., Wang, P., Brand, M., Wang, G., and Koike-Akino, T.
\newblock {SuperLoRA}: Parameter-efficient unified adaptation for large vision
  models.
\newblock In \emph{Proceedings of the IEEE/CVF Conference on Computer Vision
  and Pattern Recognition}, pp.\  8050--8055, 2024{\natexlab{a}}.

\bibitem[Chen et~al.(2024{\natexlab{b}})Chen, Dangovski, Loh, Dugan, Luo, and
  Solja{\v{c}}i{\'c}]{chen2024quanta}
Chen, Z., Dangovski, R., Loh, C., Dugan, O., Luo, D., and Solja{\v{c}}i{\'c},
  M.
\newblock {QuanTA}: Efficient high-rank fine-tuning of {LLMs} with
  quantum-informed tensor adaptation.
\newblock \emph{Advances in Neural Information Processing Systems},
  37:\penalty0 92210--92245, 2024{\natexlab{b}}.

\bibitem[Cobbe et~al.(2021)Cobbe, Kosaraju, Bavarian, Chen, Jun, Kaiser,
  Plappert, Tworek, Hilton, Nakano, et~al.]{cobbe2021training}
Cobbe, K., Kosaraju, V., Bavarian, M., Chen, M., Jun, H., Kaiser, L., Plappert,
  M., Tworek, J., Hilton, J., Nakano, R., et~al.
\newblock Training verifiers to solve math word problems.
\newblock \emph{URL https://arxiv. org/abs/2110.14168}, 9, 2021.

\bibitem[David \& Nagaraja(2004)David and Nagaraja]{david2004order}
David, H.~A. and Nagaraja, H.~N.
\newblock \emph{Order statistics}.
\newblock John Wiley \& Sons, 2004.

\bibitem[Denil et~al.(2013)Denil, Shakibi, Dinh, Ranzato, and
  De~Freitas]{denil2013predicting}
Denil, M., Shakibi, B., Dinh, L., Ranzato, M., and De~Freitas, N.
\newblock Predicting parameters in deep learning.
\newblock \emph{Advances in neural information processing systems}, 26, 2013.

\bibitem[Denton et~al.(2014)Denton, Zaremba, Bruna, LeCun, and
  Fergus]{denton2014exploiting}
Denton, E.~L., Zaremba, W., Bruna, J., LeCun, Y., and Fergus, R.
\newblock Exploiting linear structure within convolutional networks for
  efficient evaluation.
\newblock \emph{Advances in neural information processing systems}, 27, 2014.

\bibitem[Ding et~al.(2022)Ding, Xiao, Codella, Luo, Wang, and
  Yuan]{ding2022davit}
Ding, M., Xiao, B., Codella, N., Luo, P., Wang, J., and Yuan, L.
\newblock {DaViT}: Dual attention vision transformers.
\newblock In \emph{European conference on computer vision}, pp.\  74--92.
  Springer, 2022.

\bibitem[Dong et~al.(2017)Dong, Chen, and Pan]{dong2017learning}
Dong, X., Chen, S., and Pan, S.
\newblock Learning to prune deep neural networks via layer-wise optimal brain
  surgeon.
\newblock \emph{Advances in neural information processing systems}, 30, 2017.

\bibitem[Dong et~al.(2019)Dong, Yao, Gholami, Mahoney, and
  Keutzer]{dong2019hawq}
Dong, Z., Yao, Z., Gholami, A., Mahoney, M.~W., and Keutzer, K.
\newblock {HAWQ}: Hessian aware quantization of neural networks with
  mixed-precision.
\newblock In \emph{Proceedings of the IEEE/CVF International Conference on
  Computer Vision}, pp.\  293--302, 2019.

\bibitem[Edalati et~al.(2022)Edalati, Tahaei, Kobyzev, Nia, Clark, and
  Rezagholizadeh]{edalati2022krona}
Edalati, A., Tahaei, M., Kobyzev, I., Nia, V.~P., Clark, J.~J., and
  Rezagholizadeh, M.
\newblock {KronA}: Parameter efficient tuning with kronecker adapter.
\newblock \emph{arXiv preprint arXiv:2212.10650}, 2022.

\bibitem[Evenbly(2018)]{evenbly2018gauge}
Evenbly, G.
\newblock Gauge fixing, canonical forms, and optimal truncations in tensor
  networks with closed loops.
\newblock \emph{Physical Review B}, 98\penalty0 (8):\penalty0 085155, 2018.

\bibitem[Frantar \& Alistarh(2023)Frantar and Alistarh]{frantar2023sparsegpt}
Frantar, E. and Alistarh, D.
\newblock {SparseGPT}: Massive language models can be accurately pruned in
  one-shot.
\newblock In \emph{International Conference on Machine Learning}, pp.\
  10323--10337. PMLR, 2023.

\bibitem[Frantar et~al.(2022)Frantar, Ashkboos, Hoefler, and
  Alistarh]{frantar2022gptq}
Frantar, E., Ashkboos, S., Hoefler, T., and Alistarh, D.
\newblock {GPTQ}: Accurate post-training quantization for generative
  pre-trained transformers.
\newblock \emph{arXiv preprint arXiv:2210.17323}, 2022.

\bibitem[Fu et~al.(2023)Fu, Arora, Grogan, Johnson, Eyuboglu, Thomas, Spector,
  Poli, Rudra, and R{\'e}]{fu2023monarch}
Fu, D., Arora, S., Grogan, J., Johnson, I., Eyuboglu, E.~S., Thomas, A.,
  Spector, B., Poli, M., Rudra, A., and R{\'e}, C.
\newblock Monarch mixer: A simple sub-quadratic {GEMM}-based architecture.
\newblock \emph{Advances in Neural Information Processing Systems},
  36:\penalty0 77546--77603, 2023.

\bibitem[{Gemma Team} et~al.(2025){Gemma Team}, Kamath, Ferret, Pathak,
  Vieillard, Merhej, Perrin, Matejovicova, Ram{\'e}, Rivi{\`e}re,
  et~al.]{team2025gemma}
{Gemma Team}, Kamath, A., Ferret, J., Pathak, S., Vieillard, N., Merhej, R.,
  Perrin, S., Matejovicova, T., Ram{\'e}, A., Rivi{\`e}re, M., et~al.
\newblock Gemma 3 technical report.
\newblock \emph{arXiv preprint arXiv:2503.19786}, 2025.

\bibitem[Grover et~al.(2019)Grover, Wang, Zweig, and
  Ermon]{grover2019stochastic}
Grover, A., Wang, E., Zweig, A., and Ermon, S.
\newblock Stochastic optimization of sorting networks via continuous
  relaxations.
\newblock \emph{arXiv preprint arXiv:1903.08850}, 2019.

\bibitem[Haider et~al.(2024)Haider, Perez-Becker, Portet, Madan, Garg, Ashfaq,
  Majercak, Wen, Kim, Yang, et~al.]{haider2024phi}
Haider, E., Perez-Becker, D., Portet, T., Madan, P., Garg, A., Ashfaq, A.,
  Majercak, D., Wen, W., Kim, D., Yang, Z., et~al.
\newblock Phi-3 safety post-training: Aligning language models with a
  ``break-fix'' cycle.
\newblock \emph{arXiv preprint arXiv:2407.13833}, 2024.

\bibitem[Hassibi et~al.(1993)Hassibi, Stork, and Wolff]{hassibi1993optimal}
Hassibi, B., Stork, D., and Wolff, G.
\newblock Optimal brain surgeon: Extensions and performance comparisons.
\newblock \emph{Advances in neural information processing systems}, 6, 1993.

\bibitem[Hayou et~al.(2024)Hayou, Ghosh, and Yu]{hayou2024lora+}
Hayou, S., Ghosh, N., and Yu, B.
\newblock {LoRA}+: Efficient low rank adaptation of large models.
\newblock In \emph{International Conference on Machine Learning}, 2024.

\bibitem[Hsieh et~al.(2023)Hsieh, Li, Yeh, Nakhost, Fujii, Ratner, Krishna,
  Lee, and Pfister]{hsieh2023distilling}
Hsieh, C.-Y., Li, C.-L., Yeh, C.-K., Nakhost, H., Fujii, Y., Ratner, A.,
  Krishna, R., Lee, C.-Y., and Pfister, T.
\newblock Distilling step-by-step! outperforming larger language models with
  less training data and smaller model sizes.
\newblock \emph{arXiv preprint arXiv:2305.02301}, 2023.

\bibitem[Hu et~al.(2021)Hu, Wallis, Allen-Zhu, Li, Wang, Wang, Chen,
  et~al.]{hu2021lora}
Hu, E.~J., Wallis, P., Allen-Zhu, Z., Li, Y., Wang, S., Wang, L., Chen, W.,
  et~al.
\newblock {LoRA}: Low-rank adaptation of large language models.
\newblock In \emph{International Conference on Learning Representations}, 2021.

\bibitem[Huggins et~al.(2019)Huggins, Patil, Mitchell, Whaley, and
  Stoudenmire]{huggins2019towards}
Huggins, W., Patil, P., Mitchell, B., Whaley, K.~B., and Stoudenmire, E.~M.
\newblock Towards quantum machine learning with tensor networks.
\newblock \emph{Quantum Science and technology}, 4\penalty0 (2):\penalty0
  024001, 2019.

\bibitem[Hwang et~al.(2024)Hwang, Park, Lee, Yang, and Maeng]{hwang2024pc}
Hwang, I., Park, H., Lee, Y., Yang, J., and Maeng, S.
\newblock {PC-LoRA}: Low-rank adaptation for progressive model compression with
  knowledge distillation.
\newblock \emph{arXiv preprint arXiv:2406.09117}, 2024.

\bibitem[Intelligence et~al.(2025)Intelligence, Black, Brown, Darpinian,
  Dhabalia, Driess, Esmail, Equi, Finn, Fusai, et~al.]{intelligence2025pi05}
Intelligence, P., Black, K., Brown, N., Darpinian, J., Dhabalia, K., Driess,
  D., Esmail, A., Equi, M., Finn, C., Fusai, N., et~al.
\newblock $\pi_{0.5}$: a vision-language-action model with open-world
  generalization.
\newblock \emph{arXiv preprint arXiv:2504.16054}, 2025.

\bibitem[Jaderberg et~al.(2014)Jaderberg, Vedaldi, and
  Zisserman]{jaderberg2014speeding}
Jaderberg, M., Vedaldi, A., and Zisserman, A.
\newblock Speeding up convolutional neural networks with low rank expansions.
\newblock \emph{arXiv preprint arXiv:1405.3866}, 2014.

\bibitem[Katz et~al.(2024)Katz, Bommarito, Gao, and Arredondo]{katz2024gpt}
Katz, D.~M., Bommarito, M.~J., Gao, S., and Arredondo, P.
\newblock {GPT}-4 passes the bar exam.
\newblock \emph{Philosophical Transactions of the Royal Society A},
  382\penalty0 (2270):\penalty0 20230254, 2024.

\bibitem[Koike-Akino et~al.(2025{\natexlab{a}})Koike-Akino, Liu, and
  Wang]{koike2025mu}
Koike-Akino, T., Liu, J., and Wang, Y.
\newblock $\mu$-moe: Test-time pruning as micro-grained mixture-of-experts.
\newblock \emph{arXiv preprint arXiv:2505.18451}, 2025{\natexlab{a}}.

\bibitem[Koike-Akino et~al.(2025{\natexlab{b}})Koike-Akino, Tonin, Wu, Wu,
  Candogan, and Cevher]{koike2025quantum}
Koike-Akino, T., Tonin, F., Wu, Y., Wu, F.~Z., Candogan, L.~N., and Cevher, V.
\newblock {Quantum-PEFT}: Ultra parameter-efficient fine-tuning.
\newblock \emph{arXiv preprint arXiv:2503.05431}, 2025{\natexlab{b}}.

\bibitem[Koike-Akino et~al.(2026{\natexlab{a}})Koike-Akino, Chen, Liu, Wang,
  Wang, and Brand]{koike2026latentllm}
Koike-Akino, T., Chen, X., Liu, J., Wang, Y., Wang, P.~P., and Brand, M.
\newblock {LatentLLM}: Activation-aware transform to multi-head latent
  attention.
\newblock In \emph{Proceedings of the AAAI Conference on Artificial
  Intelligence}, volume~40, pp.\  22644--22652, 2026{\natexlab{a}}.

\bibitem[Koike-Akino et~al.(2026{\natexlab{b}})Koike-Akino, Liu, and
  Wang]{koike2026ttq}
Koike-Akino, T., Liu, J., and Wang, Y.
\newblock {TTQ}: Activation-aware test-time quantization to accelerate {LLM}
  inference on the fly.
\newblock \emph{arXiv preprint arXiv:2603.19296}, 2026{\natexlab{b}}.

\bibitem[Lebedev et~al.(2014)Lebedev, Ganin, Rakhuba, Oseledets, and
  Lempitsky]{lebedev2014speeding}
Lebedev, V., Ganin, Y., Rakhuba, M., Oseledets, I., and Lempitsky, V.
\newblock Speeding-up convolutional neural networks using fine-tuned
  {CP}-decomposition.
\newblock \emph{arXiv preprint arXiv:1412.6553}, 2014.

\bibitem[LeCun et~al.(1989)LeCun, Denker, and Solla]{lecun1989optimal}
LeCun, Y., Denker, J., and Solla, S.
\newblock Optimal brain damage.
\newblock \emph{Advances in neural information processing systems}, 2, 1989.

\bibitem[Lehmer(1960)]{lehmer1960teaching}
Lehmer, D.~H.
\newblock Teaching combinatorial tricks to a computer.
\newblock In \emph{Proceedings of Symposia in Applied Mathematics}, pp.\
  179--193. American Mathematical Society, 1960.

\bibitem[Li et~al.(1998)Li, Tromp, and Vit{\'a}nyi]{li1998reversible}
Li, M., Tromp, J., and Vit{\'a}nyi, P.
\newblock Reversible simulation of irreversible computation.
\newblock \emph{Physica D: Nonlinear Phenomena}, 120\penalty0 (1-2):\penalty0
  168--176, 1998.

\bibitem[Li et~al.(2026)Li, Lee, Yin, and Panda]{li2026optimal}
Li, Y., Lee, D., Yin, R., and Panda, P.
\newblock Optimal brain decomposition for accurate {LLM} low-rank
  approximation.
\newblock \emph{arXiv preprint arXiv:2604.00821}, 2026.

\bibitem[Lin et~al.(2024{\natexlab{a}})Lin, Tang, Ye, Cui, Zhu, Jin, Zhang,
  Ning, and Yuan]{lin2024moe}
Lin, B., Tang, Z., Ye, Y., Cui, J., Zhu, B., Jin, P., Zhang, J., Ning, M., and
  Yuan, L.
\newblock {MoE-LlaVa}: Mixture of experts for large vision-language models.
\newblock \emph{arXiv preprint arXiv:2401.15947}, 2024{\natexlab{a}}.

\bibitem[Lin et~al.(2023)Lin, Tang, Tang, Yang, Dang, and Han]{lin2023awq}
Lin, J., Tang, J., Tang, H., Yang, S., Dang, X., and Han, S.
\newblock {AWQ}: Activation-aware weight quantization for {LLM} compression and
  acceleration.
\newblock \emph{arXiv preprint arXiv:2306.00978}, 2023.

\bibitem[Lin et~al.(2024{\natexlab{b}})Lin, Tang, Tang, Yang, Chen, Wang, Xiao,
  Dang, Gan, and Han]{lin2024awq}
Lin, J., Tang, J., Tang, H., Yang, S., Chen, W.-M., Wang, W.-C., Xiao, G.,
  Dang, X., Gan, C., and Han, S.
\newblock {AWQ}: Activation-aware weight quantization for on-device {LLM}
  compression and acceleration.
\newblock \emph{Proceedings of Machine Learning and Systems}, 6:\penalty0
  87--100, 2024{\natexlab{b}}.

\bibitem[Liu et~al.(2024{\natexlab{a}})Liu, Feng, Xue, Wang, Wu, Lu, Zhao,
  Deng, Zhang, Ruan, et~al.]{liu2024deepseek}
Liu, A., Feng, B., Xue, B., Wang, B., Wu, B., Lu, C., Zhao, C., Deng, C.,
  Zhang, C., Ruan, C., et~al.
\newblock {DeepSeek}-v3 technical report.
\newblock \emph{arXiv preprint arXiv:2412.19437}, 2024{\natexlab{a}}.

\bibitem[Liu et~al.(2023{\natexlab{a}})Liu, Zhu, Gao, Feng, Liu, Zhu, and
  Stone]{liu2023libero}
Liu, B., Zhu, Y., Gao, C., Feng, Y., Liu, Q., Zhu, Y., and Stone, P.
\newblock {LIBERO}: Benchmarking knowledge transfer for lifelong robot
  learning.
\newblock \emph{Advances in Neural Information Processing Systems},
  36:\penalty0 44776--44791, 2023{\natexlab{a}}.

\bibitem[Liu et~al.(2023{\natexlab{b}})Liu, Koike-Akino, Wang, Brand, Wang, and
  Parsons]{liu2023loda}
Liu, J., Koike-Akino, T., Wang, P., Brand, M., Wang, Y., and Parsons, K.
\newblock {LoDA}: Low-dimensional adaptation of large language models.
\newblock In \emph{NeurIPS’23 Workshop on on Efficient Natural Language and
  Speech Processing}, 2023{\natexlab{b}}.

\bibitem[Liu et~al.(2025)Liu, Koike-Akino, Wang, Mansour, and
  Brand]{liu2025awp}
Liu, J., Koike-Akino, T., Wang, Y., Mansour, H., and Brand, M.
\newblock {AWP}: Activation-aware weight pruning and quantization with
  projected gradient descent.
\newblock \emph{arXiv preprint arXiv:2506.10205}, 2025.

\bibitem[Liu et~al.(2024{\natexlab{b}})Liu, Yuan, Jin, Zhong, Xu, Braverman,
  Chen, and Hu]{pmlr-v235-liu24bz}
Liu, Z., Yuan, J., Jin, H., Zhong, S., Xu, Z., Braverman, V., Chen, B., and Hu,
  X.
\newblock {KIVI}: A tuning-free asymmetric 2bit quantization for {KV} cache.
\newblock In Salakhutdinov, R., Kolter, Z., Heller, K., Weller, A., Oliver, N.,
  Scarlett, J., and Berkenkamp, F. (eds.), \emph{Proceedings of the 41st
  International Conference on Machine Learning}, volume 235 of
  \emph{Proceedings of Machine Learning Research}, pp.\  32332--32344. PMLR,
  21--27 Jul 2024{\natexlab{b}}.
\newblock URL \url{https://proceedings.mlr.press/v235/liu24bz.html}.

\bibitem[Luo et~al.(2024)Luo, Liu, Yu, Ren, and Bai]{luo2024adaptive}
Luo, S., Liu, M., Yu, Y., Ren, S., and Bai, Y.
\newblock An adaptive tensor-train decomposition approach for efficient deep
  neural network compression.
\newblock \emph{arXiv preprint arXiv:2408.01534}, 2024.

\bibitem[McCluskey(1956)]{mccluskey1956minimization}
McCluskey, E.~J.
\newblock Minimization of {Boolean} functions.
\newblock \emph{The Bell System Technical Journal}, 35\penalty0 (6):\penalty0
  1417--1444, 1956.

\bibitem[Mena et~al.(2018)Mena, Belanger, Linderman, and
  Snoek]{mena2018learning}
Mena, G., Belanger, D., Linderman, S., and Snoek, J.
\newblock Learning latent permutations with {Gumbel}-{Sinkhorn} networks.
\newblock \emph{arXiv preprint arXiv:1802.08665}, 2018.

\bibitem[Merity et~al.(2016)Merity, Xiong, Bradbury, and
  Socher]{merity2016pointer}
Merity, S., Xiong, C., Bradbury, J., and Socher, R.
\newblock Pointer sentinel mixture models.
\newblock \emph{arXiv preprint arXiv:1609.07843}, 2016.

\bibitem[Novikov et~al.(2015)Novikov, Podoprikhin, Osokin, and
  Vetrov]{novikov2015tensorizing}
Novikov, A., Podoprikhin, D., Osokin, A., and Vetrov, D.~P.
\newblock Tensorizing neural networks.
\newblock \emph{Advances in neural information processing systems}, 28, 2015.

\bibitem[Or{\'u}s(2014)]{orus2014practical}
Or{\'u}s, R.
\newblock A practical introduction to tensor networks: Matrix product states
  and projected entangled pair states.
\newblock \emph{Annals of physics}, 349:\penalty0 117--158, 2014.

\bibitem[Paszke et~al.(2019)Paszke, Gross, Massa, Lerer, Bradbury, Chanan,
  Killeen, Lin, Gimelshein, Antiga, et~al.]{paszke2019pytorch}
Paszke, A., Gross, S., Massa, F., Lerer, A., Bradbury, J., Chanan, G., Killeen,
  T., Lin, Z., Gimelshein, N., Antiga, L., et~al.
\newblock {PyTorch}: An imperative style, high-performance deep learning
  library.
\newblock \emph{Advances in neural information processing systems}, 32, 2019.

\bibitem[Peebles \& Xie(2023)Peebles and Xie]{peebles2023scalable}
Peebles, W. and Xie, S.
\newblock Scalable diffusion models with transformers.
\newblock In \emph{Proceedings of the IEEE/CVF international conference on
  computer vision}, pp.\  4195--4205, 2023.

\bibitem[Phan et~al.(2020)Phan, Sobolev, Sozykin, Ermilov, Gusak,
  Tichavsk{\`y}, Glukhov, Oseledets, and Cichocki]{phan2020stable}
Phan, A.-H., Sobolev, K., Sozykin, K., Ermilov, D., Gusak, J., Tichavsk{\`y},
  P., Glukhov, V., Oseledets, I., and Cichocki, A.
\newblock Stable low-rank tensor decomposition for compression of convolutional
  neural network.
\newblock In \emph{European Conference on Computer Vision}, pp.\  522--539.
  Springer, 2020.

\bibitem[Prillo \& Eisenschlos(2020)Prillo and Eisenschlos]{prillo2020softsort}
Prillo, S. and Eisenschlos, J.
\newblock {SoftSort}: A continuous relaxation for the argsort operator.
\newblock In \emph{International Conference on Machine Learning}, pp.\
  7793--7802. PMLR, 2020.

\bibitem[Roberts et~al.(2019)Roberts, Milsted, Ganahl, Zalcman, Fontaine, Zou,
  Hidary, Vidal, and Leichenauer]{roberts2019tensornetwork}
Roberts, C., Milsted, A., Ganahl, M., Zalcman, A., Fontaine, B., Zou, Y.,
  Hidary, J., Vidal, G., and Leichenauer, S.
\newblock {TensorNetwork}: A library for physics and machine learning.
\newblock \emph{arXiv preprint arXiv:1905.01330}, 2019.

\bibitem[Saha et~al.(2024)Saha, Sagan, Srivastava, Goldsmith, and
  Pilanci]{saha2024compressing}
Saha, R., Sagan, N., Srivastava, V., Goldsmith, A., and Pilanci, M.
\newblock Compressing large language models using low rank and low precision
  decomposition.
\newblock \emph{Advances in Neural Information Processing Systems},
  37:\penalty0 88981--89018, 2024.

\bibitem[Sainath et~al.(2013)Sainath, Kingsbury, Sindhwani, Arisoy, and
  Ramabhadran]{sainath2013low}
Sainath, T.~N., Kingsbury, B., Sindhwani, V., Arisoy, E., and Ramabhadran, B.
\newblock Low-rank matrix factorization for deep neural network training with
  high-dimensional output targets.
\newblock In \emph{2013 IEEE international conference on acoustics, speech and
  signal processing}, pp.\  6655--6659. IEEE, 2013.

\bibitem[Saxena et~al.(2024)Saxena, Saha, Choudhary, and Roy]{saxena2024eigen}
Saxena, U., Saha, G., Choudhary, S., and Roy, K.
\newblock Eigen attention: Attention in low-rank space for {KV} cache
  compression.
\newblock \emph{arXiv preprint arXiv:2408.05646}, 2024.

\bibitem[Schollw{\"o}ck(2011)]{schollwock2011density}
Schollw{\"o}ck, U.
\newblock The density-matrix renormalization group in the age of matrix product
  states.
\newblock \emph{Annals of physics}, 326\penalty0 (1):\penalty0 96--192, 2011.

\bibitem[Schwartz et~al.(2020)Schwartz, Dodge, Smith, and
  Etzioni]{schwartz2020green}
Schwartz, R., Dodge, J., Smith, N.~A., and Etzioni, O.
\newblock Green {AI}.
\newblock \emph{Communications of the ACM}, 63\penalty0 (12):\penalty0 54--63,
  2020.

\bibitem[Sidiropoulos et~al.(2017)Sidiropoulos, De~Lathauwer, Fu, Huang,
  Papalexakis, and Faloutsos]{sidiropoulos2017tensor}
Sidiropoulos, N.~D., De~Lathauwer, L., Fu, X., Huang, K., Papalexakis, E.~E.,
  and Faloutsos, C.
\newblock Tensor decomposition for signal processing and machine learning.
\newblock \emph{IEEE Transactions on signal processing}, 65\penalty0
  (13):\penalty0 3551--3582, 2017.

\bibitem[Singh et~al.(2019)Singh, Natarjan, Shah, Jiang, Chen, Parikh, and
  Rohrbach]{singh2019towards}
Singh, A., Natarjan, V., Shah, M., Jiang, Y., Chen, X., Parikh, D., and
  Rohrbach, M.
\newblock Towards {VQA} models that can read.
\newblock In \emph{Proceedings of the IEEE Conference on Computer Vision and
  Pattern Recognition}, pp.\  8317--8326, 2019.

\bibitem[Sinha \& Fleuret(2026)Sinha and Fleuret]{sinha2026aa}
Sinha, A.~K. and Fleuret, F.
\newblock {AA-SVD}: Anchored and adaptive {SVD} for large language model
  compression.
\newblock \emph{arXiv preprint arXiv:2604.02119}, 2026.

\bibitem[Sun et~al.(2026)Sun, Zhang, Qi, Ren, Liu, Zhu, Sun, Jin, and
  Chen]{sun2026vla}
Sun, J., Zhang, W., Qi, Z., Ren, S., Liu, Z., Zhu, H., Sun, G., Jin, X., and
  Chen, Z.
\newblock {VLA-JEPA}: Enhancing vision-language-action model with latent world
  model.
\newblock \emph{arXiv preprint arXiv:2602.10098}, 2026.

\bibitem[Sun et~al.(2023)Sun, Liu, Bair, and Kolter]{sun2023simple}
Sun, M., Liu, Z., Bair, A., and Kolter, J.~Z.
\newblock A simple and effective pruning approach for large language models.
\newblock \emph{arXiv preprint arXiv:2306.11695}, 2023.

\bibitem[Takeshita(2007)]{takeshita2007permutation}
Takeshita, O.~Y.
\newblock Permutation polynomial interleavers: An algebraic-geometric
  perspective.
\newblock \emph{IEEE Transactions on Information Theory}, 53\penalty0
  (6):\penalty0 2116--2132, 2007.

\bibitem[Team et~al.(2024)Team, Mesnard, Hardin, Dadashi, Bhupatiraju, Pathak,
  Sifre, Rivi{\`e}re, Kale, Love, et~al.]{team2024gemma}
Team, G., Mesnard, T., Hardin, C., Dadashi, R., Bhupatiraju, S., Pathak, S.,
  Sifre, L., Rivi{\`e}re, M., Kale, M.~S., Love, J., et~al.
\newblock Gemma: Open models based on gemini research and technology.
\newblock \emph{arXiv preprint arXiv:2403.08295}, 2024.

\bibitem[Vidal(2008)]{vidal2008class}
Vidal, G.
\newblock Class of quantum many-body states that can be efficiently simulated.
\newblock \emph{Physical review letters}, 101\penalty0 (11):\penalty0 110501,
  2008.

\bibitem[Wang et~al.(2024{\natexlab{a}})Wang, Wang, Xu, Tang, Zhou, and
  Lu]{wang2024q}
Wang, C., Wang, Z., Xu, X., Tang, Y., Zhou, J., and Lu, J.
\newblock {Q-VLM}: Post-training quantization for large vision-language models.
\newblock \emph{arXiv preprint arXiv:2410.08119}, 2024{\natexlab{a}}.

\bibitem[Wang et~al.(2024{\natexlab{b}})Wang, Zheng, Wan, and
  Zhang]{wang2024svd}
Wang, X., Zheng, Y., Wan, Z., and Zhang, M.
\newblock {SVD-LLM}: Truncation-aware singular value decomposition for large
  language model compression.
\newblock \emph{arXiv preprint arXiv:2403.07378}, 2024{\natexlab{b}}.

\bibitem[Williams \& Aletras(2023)Williams and Aletras]{williams2023impact}
Williams, M. and Aletras, N.
\newblock On the impact of calibration data in post-training quantization and
  pruning.
\newblock \emph{arXiv preprint arXiv:2311.09755}, 2023.

\bibitem[Wolf(2019)]{wolf2019huggingface}
Wolf, T.
\newblock Huggingface's transformers: State-of-the-art natural language
  processing.
\newblock \emph{arXiv preprint arXiv:1910.03771}, 2019.

\bibitem[Xiao et~al.(2024{\natexlab{a}})Xiao, Wu, Xu, Dai, Hu, Lu, Zeng, Liu,
  and Yuan]{xiao2024florence}
Xiao, B., Wu, H., Xu, W., Dai, X., Hu, H., Lu, Y., Zeng, M., Liu, C., and Yuan,
  L.
\newblock Florence-2: Advancing a unified representation for a variety of
  vision tasks.
\newblock In \emph{Proceedings of the IEEE/CVF Conference on Computer Vision
  and Pattern Recognition}, pp.\  4818--4829, 2024{\natexlab{a}}.

\bibitem[Xiao et~al.(2024{\natexlab{b}})Xiao, Tian, Chen, Han, and
  Lewis]{xiao2024efficient}
Xiao, G., Tian, Y., Chen, B., Han, S., and Lewis, M.
\newblock Efficient streaming language models with attention sinks.
\newblock In \emph{The Twelfth International Conference on Learning
  Representations}, 2024{\natexlab{b}}.
\newblock URL \url{https://openreview.net/forum?id=NG7sS51zVF}.

\bibitem[Xu \& McAuley(2023)Xu and McAuley]{xu2023survey}
Xu, C. and McAuley, J.
\newblock A survey on model compression and acceleration for pretrained
  language models.
\newblock In \emph{Proceedings of the AAAI Conference on Artificial
  Intelligence}, volume~37, pp.\  10566--10575, 2023.

\bibitem[Yan et~al.(2025)Yan, Li, Zhang, Qin, Kong, Zhang, and
  Yang]{yan2025recalkvlowrankkvcache}
Yan, X., Li, Z., Zhang, T., Qin, H., Kong, L., Zhang, Y., and Yang, X.
\newblock Recalkv: Low-rank kv cache compression via head reordering and
  offline calibration, 2025.
\newblock URL \url{https://arxiv.org/abs/2505.24357}.

\bibitem[Yang et~al.(2025)Yang, Li, Yang, Zhang, Hui, Zheng, Yu, Gao, Huang,
  Lv, et~al.]{yang2025qwen3}
Yang, A., Li, A., Yang, B., Zhang, B., Hui, B., Zheng, B., Yu, B., Gao, C.,
  Huang, C., Lv, C., et~al.
\newblock Qwen3 technical report.
\newblock \emph{arXiv preprint arXiv:2505.09388}, 2025.

\bibitem[Yeh et~al.(2024)Yeh, Hsieh, Gao, Yang, Oh, and
  Gong]{yeh2024navigating}
Yeh, S.-Y., Hsieh, Y.-G., Gao, Z., Yang, B. B.~W., Oh, G., and Gong, Y.
\newblock Navigating text-to-image customization: From ly{CORIS} fine-tuning to
  model evaluation.
\newblock In \emph{International Conference on Learning Representations}, 2024.

\bibitem[Yuan et~al.(2023)Yuan, Shang, Song, Wu, Yan, and Sun]{yuan2023asvd}
Yuan, Z., Shang, Y., Song, Y., Wu, Q., Yan, Y., and Sun, G.
\newblock {ASVD}: Activation-aware singular value decomposition for compressing
  large language models.
\newblock \emph{arXiv preprint arXiv:2312.05821}, 2023.

\bibitem[Yuan et~al.(2024)Yuan, Shang, Zhou, Dong, Zhou, Xue, Wu, Li, Gu, Lee,
  et~al.]{yuan2024llm}
Yuan, Z., Shang, Y., Zhou, Y., Dong, Z., Zhou, Z., Xue, C., Wu, B., Li, Z., Gu,
  Q., Lee, Y.~J., et~al.
\newblock {LLM} inference unveiled: Survey and roofline model insights.
\newblock \emph{arXiv preprint arXiv:2402.16363}, 2024.

\bibitem[Zaheer et~al.(2017)Zaheer, Kottur, Ravanbakhsh, Poczos, Salakhutdinov,
  and Smola]{zaheer2017deep}
Zaheer, M., Kottur, S., Ravanbakhsh, S., Poczos, B., Salakhutdinov, R.~R., and
  Smola, A.~J.
\newblock Deep sets.
\newblock \emph{Advances in neural information processing systems}, 30, 2017.

\bibitem[Zandieh et~al.(2024)Zandieh, Daliri, and Han]{zandieh2024qjl}
Zandieh, A., Daliri, M., and Han, I.
\newblock {QJL}: 1-bit quantized {JL} transform for {KV} cache quantization
  with zero overhead, 2024.
\newblock URL \url{https://openreview.net/forum?id=xHPVGmLXjd}.

\bibitem[Zandieh et~al.(2026)Zandieh, Daliri, Hadian, and
  Mirrokni]{zandieh2026turboquant}
Zandieh, A., Daliri, M., Hadian, M., and Mirrokni, V.
\newblock Turboquant: Online vector quantization with near-optimal distortion
  rate.
\newblock In \emph{The Fourteenth International Conference on Learning
  Representations}, 2026.
\newblock URL \url{https://openreview.net/forum?id=tO3ASKZlok}.

\bibitem[Zhai et~al.(2023)Zhai, Mustafa, Kolesnikov, and
  Beyer]{zhai2023sigmoid}
Zhai, X., Mustafa, B., Kolesnikov, A., and Beyer, L.
\newblock Sigmoid loss for language image pre-training.
\newblock In \emph{Proceedings of the IEEE/CVF international conference on
  computer vision}, pp.\  11975--11986, 2023.

\bibitem[Zhang et~al.(2025)Zhang, Wang, Liu, Wang, Cheng, Zhang, and yelong
  shen]{zhang2025lorc}
Zhang, R., Wang, K., Liu, L., Wang, S., Cheng, H., Zhang, C., and yelong shen.
\newblock Lo{RC}: Low-rank compression for {LLM}s {KV} cache with a progressive
  compression strategy, 2025.
\newblock URL \url{https://openreview.net/forum?id=NI8AUSAc4i}.

\bibitem[Zhang et~al.(2023)Zhang, Sheng, Zhou, Chen, Zheng, Cai, Song, Tian,
  R\'{e}, Barrett, Wang, and Chen]{NEURIPS2023_6ceefa7b}
Zhang, Z., Sheng, Y., Zhou, T., Chen, T., Zheng, L., Cai, R., Song, Z., Tian,
  Y., R\'{e}, C., Barrett, C., Wang, Z.~A., and Chen, B.
\newblock H2o: Heavy-hitter oracle for efficient generative inference of large
  language models.
\newblock In Oh, A., Naumann, T., Globerson, A., Saenko, K., Hardt, M., and
  Levine, S. (eds.), \emph{Advances in Neural Information Processing Systems},
  volume~36, pp.\  34661--34710. Curran Associates, Inc., 2023.
\newblock URL
  \url{https://proceedings.neurips.cc/paper_files/paper/2023/file/6ceefa7b15572587b78ecfcebb2827f8-Paper-Conference.pdf}.

\bibitem[Zheng et~al.(2025)Zheng, Li, Wang, Liu, Kang, Feng, Zheng, Zou, Chen,
  Zeng, et~al.]{zheng2025x}
Zheng, J., Li, J., Wang, Z., Liu, D., Kang, X., Feng, Y., Zheng, Y., Zou, J.,
  Chen, Y., Zeng, J., et~al.
\newblock {X-VLA}: Soft-prompted transformer as scalable cross-embodiment
  vision-language-action model.
\newblock \emph{arXiv preprint arXiv:2510.10274}, 2025.

\bibitem[Zhu et~al.(2024)Zhu, Li, Liu, Ma, and Wang]{zhu2024survey}
Zhu, X., Li, J., Liu, Y., Ma, C., and Wang, W.
\newblock A survey on model compression for large language models.
\newblock \emph{Transactions of the Association for Computational Linguistics},
  12:\penalty0 1556--1577, 2024.

\bibitem[Zi et~al.(2023)Zi, Qi, Wang, Wang, Wong, and Zhang]{zi2023delta}
Zi, B., Qi, X., Wang, L., Wang, J., Wong, K.-F., and Zhang, L.
\newblock Delta-{LoRa}: Fine-tuning high-rank parameters with the delta of
  low-rank matrices.
\newblock \emph{arXiv preprint arXiv:2309.02411}, 2023.

\end{thebibliography}
\bibliographystyle{icml}

\newpage
\appendix
\onecolumn

\section{Related work}
\label{sec:related}

\subsection{Model compression}
The field of model compression for LLMs has aimed at mitigating the substantial computation and memory requirements~\cite{zhu2024survey, yuan2024llm}. 
Such methods primarily fall into four categories: weight quantization~\cite{lin2024awq, frantar2022gptq, wang2024q}, network pruning~\cite{lecun1989optimal, hassibi1993optimal, frantar2023sparsegpt, bai2024sparsellm},
knowledge distillation~\cite{hsieh2023distilling, hwang2024pc}, and rank reduction~\cite{yuan2023asvd, liu2024deepseek, hwang2024pc, saxena2024eigen, saha2024compressing}. 

\subsection{Low-rank compression}

Low-rank compression for neural networks builds on the empirical observation that trained weight matrices and convolutional kernels are often highly redundant and admit accurate low-rank approximations. 
Early work by \citet{denil2013predicting} showed that a large fraction of network parameters can be predicted from a small subset, motivating structured factorization approaches. 
\citet{sainath2013low} and \citet{jaderberg2014speeding} popularized the use of SVD and related factorization techniques, achieving substantial speedups with limited accuracy loss. 
This line was extended by \citet{denton2014exploiting}, who applied low-rank approximations to convolutional filters, and \citet{lebedev2014speeding, sidiropoulos2017tensor}, who introduced tensor decomposition methods such as Tucker and CP decompositions,
though care is required to avoid instability or accuracy degradation~\cite{phan2020stable}.

\subsection{KV cache reduction}
KV-cache memory is a major bottleneck in long-context LLM inference, growing linearly with sequence length, batch size, number of layers, and KV hidden dimension. Existing approaches to reduce KV-cache can be broadly categorized into token eviction, quantization, and representation compression.

\textbf{Token eviction} methods reduce memory by selectively discarding cached tokens. StreamingLLM~\citep{xiao2024efficient} preserves initial attention-sink tokens together with a sliding window of recent tokens for stable streaming decoding, while H2O~\citep{NEURIPS2023_6ceefa7b} retains a balance of recent tokens and heavy-hitter tokens that have significant contribution to attention scores. PyramidKV~\citep{cai2025pyramidkv} further improves eviction-based compression by allocating nonuniform cache budgets across layers, retaining more KV entries in lower layers and fewer in higher layers. Although effective, eviction-based approaches may discard information needed for long-range reasoning.

\textbf{Quantization} methods compress KV tensors without removing tokens. KIVI~\citep{pmlr-v235-liu24bz} proposes tuning-free asymmetric 2-bit KV-cache quantization, using per-channel quantization for keys and per-token quantization for values. QJL~\citep{zandieh2024qjl} applies a Johnson-Lindenstrauss transform followed by sign-bit quantization, eliminating the need to store scale and zero-point quantization metadata while providing an unbiased, low-distortion estimator for query-key inner products. TurboQuant~\cite{zandieh2026turboquant} extends this direction with online vector quantization, combining an MSE-oriented quantizer with QJL-based residual correction to improve inner-product estimation.

\textbf{Low-rank compression }reduces KV-cache memory by exploiting redundancy in hidden dimensions or KV projections. LoRC~\citep{zhang2025lorc} applies low-rank approximation to KV weight matrices with layer-wise sensitivity and progressive compression. ReCalKV~\citep{yan2025recalkvlowrankkvcache} introduces a post-training low-rank KV-cache compression method with tailored strategies for keys and values: Head-wise similarity-aware reordering for keys, which clusters structurally similar heads and applies grouped SVD to the key projection matrix, and Offline Calibration and Matrix Fusion for values, which calibrates the low-rank value projection
matrix using a small calibration dataset.

\subsection{Low-rank adaptation}

LoRA~\citep{hu2021lora} updates the pretrained weight matrix through the addition of a product of two low-rank matrices
with widespread adoption~\citep{zi2023delta,chavan2023one,hayou2024lora+}. 
Many variants were introduced, e.g., 
CP decomposition with LoTR~\citep{bershatsky2024lotr} and Tucker decomposition with SuperLoRA~\citep{chen2024superlora}, and nonlinear low-rank mapping with LoDA~\citep{liu2023loda}, Hadamard \citep{yeh2024navigating} and Kronecker product \citep{edalati2022krona}, and quantum tensor networks~\cite{chen2024quanta, koike2025quantum}.

\subsection{Activation-aware compression}

\paragraph{Decomposition} ASVD~\citep{yuan2023asvd} improves the low-rank decomposition by dealing with activation statistics.
It was applied to SVD-LLM~\citep{wang2024svd} and Palu~\citep{chang2024palu}.
AA-SVD~\citep{sinha2026aa} incorporates error propagation factor motivated by QEP~\cite{arai2025quantization}.
And further OBD-LLM~\cite{li2026optimal} uses gradient information as an approximated Hessian to improve ASVD.

\paragraph{Quantization} 
HAWQ~\cite{dong2019hawq} uses layer-wise quantization based on optimal brain pruning~\cite{lecun1989optimal}.
Then, GPTQ~\cite{frantar2022gptq} extends it using zero-shot calibration, while activation-aware quantization (AWQ)~\cite{lin2023awq} uses activation-dependent scaling.
AWP~\cite{liu2025awp} uses projected gradient descent for activation aware quantization and pruning based on compressed sensing framework.
QEP~\cite{arai2025quantization} mitigates error propagation across layers.

\paragraph{Pruning} 
Activation-aware pruning methods~\cite{williams2023impact} include
SparseGPT~\cite{frantar2023sparsegpt}, which uses layer-wise optimal brain surgeon~\cite{dong2017learning, hassibi1993optimal, lecun1989optimal}. 
SparseLLM~\cite{bai2024sparsellm} extends with joint module compression.
Wanda~\cite{sun2023simple} greatly simplifies the pruning mechanism, and
$\mu$-MoE~\cite{koike2025mu} applied it for online pruning

\subsection{Permutation/sorting networks}

Permutation or sorting neural networks are architectures designed to process sets or sequences for ordering elements.
Deep Sets by \citet{zaheer2017deep} establishes that any permutation-invariant function over a set can be decomposed into a sum of elementwise embeddings followed by a global function, providing a theoretical foundation for set-based learning.
In parallel, differentiable sorting and ranking networks have emerged to learn orderings explicitly, such as NeuralSort \citep{grover2019stochastic} and SoftSort \citep{prillo2020softsort}, which provide continuous relaxations of sorting operators to enable gradient-based optimization. 
A closely related approach is Gumbel--Sinkhorn network~\cite{mena2018learning}, which leverages entropic optimal transport and Gumbel noise to produce differentiable approximations of permutation matrices via Sinkhorn normalization.

\subsection{Tensor network}
Tensor network~\citep{roberts2019tensornetwork} provides a way to represent/manipulate multi-dimensional arrays of data by factorizing into a network of lower-dimensional tensors. 
Many tensor decomposition methods are used for tensor networks, including matrix product state (MPS) and tree tensor network (TTN)~\citep{huggins2019towards}, based on tensor train (TT) decomposition and Hierarchical Tucker decomposition (HT), respectively. 
More sophisticated ones used in QML include multi-scale entanglement renormalization ansatz (MERA)~\citep{vidal2008class} and projected entangled-pair states (PEPS)~\citep{orus2014practical}.
Tensorization provides efficient parameterization of DNN architecture~\citep{novikov2015tensorizing}.

\section{Einstein summation}
\label{sec:einsum}

The einsum is a powerful tool to represent diverse set of tensor networks.
Several examples including TT/TR/CP/TD/HT/PEPS as well as LoHa/LoKr~\citep{yeh2024navigating} are listed in the pseudo code below:
\begin{lstlisting}
#TT: A(di,dp), B(dp,dj,dq), C(dq,dk,dr), D(dr,dl)
X = torch.einsum("ip,pjq,qkr,rl->ijkl", A, B, C, D)

#TR: A(dm,di,dp), B(dp,dj,dq), C(dq,dk,dr), D(dr,dl,dm)
X = torch.einsum("mip,pjq,qkr,rlm->ijkl", A, B, C, D)

#CP: A(dr), B(dr,di), C(dr,dj), D(dr,dk)
X = torch.einsum("r,ri,rj,rk->ijk", A, B, C, D)

#TD: A(dp,dq,dr), B(dp,di), C(dq,dj), D(dr,dk)
X = torch.einsum("pqr,pi,qj,rk->ijk", A, B, C, D)

#HT: A(di,dp), B(dj,dq), C(dk,dr), D(dl,ds), E(dp,dq,dt), F(dr,ds,du), G(dt,du)
X = torch.einsum("ip,jq,kr,ls,pqt,rsu,tu->ijkl", A, B, C, D, E, F, G)

#PEPS: A(di,dr,dp), B(dj,dp,ds,dq), C(dk,dq,dt), D(dl,dr,du), E(dm,du,ds,dv), F(dl,dv,dt)
X = torch.einsum("irp,jpsq,kqt,lru,musv,lvt->ijklmn", A, B, C, D, E, F)

#LoHa: A(di,dp), B(dp,dj), C(di,dq), D(dq,dj)
X = torch.einsum("ip,pj,iq,qj->ij", A, B, C, D)

#LoKr: A(di,dj), B(dp,dr), C(dr,dq)
X = torch.einsum("ij,pr,rq->ipjq", A, B, C)

X = einops.rearrange("i p j q -> (i p) (j q)", X) # unfolding 4-mode to 2-mode
\end{lstlisting}
Here we use \texttt{einops}\footnote{\url{https://einops.rocks/}} for folding/unfolding. 
It gives an opportunity to optimize topologies for tensor networks by just designing the einsum equation (a string like \texttt{"ip,pj->ij"}) with properly shaped tensor cores.
Even a Monarch mixer~\cite{fu2023monarch} can be represented by a series of einsum operations without specifying permutation.
The Monarch mixer uses dilated mixing with gradually expanded respective field:
\begin{align}
    W = P_{m+1} \mathrm{diag}[W_m] \cdots \mathrm{diag}[W_2] P_2 \mathrm{diag}[W_1] P_1 \mathrm{diag}[W_0] P_0,
\end{align}
where $P_k$'s are permutations using Cooley--Tukey like butterfly architecture, and $\mathrm{diag}[W_k]$ are block diagonal weights.
It generalizes many structured matrices like Fourier transform, Walsh--Hadamard transform, etc., while achieving low complexity with high parameter efficiency.
It can be realized by the einsum chain, in a surprisingly elegant fashion as follows:
\begin{lstlisting}
# 3-stage dilated Monarch mixer 
# Assume input activation X (..., d, d, d), Mixer W (3, d, d, d, d)
X = torch.einsum("...ij,...i->...j", W[0], X) # the first local mixing
X = torch.einsum("...iaj,...ia->...ja", W[1], X) # the second dilated mixing
X = torch.einsum("iabj,...iab->...jab", W[2], X) # the last dilated mixing
\end{lstlisting}

As given in the above example, using einsum, we can construct any tensor networks by specifying einsum equation and shapes.
The pseudo code for a tensor network module using arbitrary einsum equation and shapes can be listed as follows:
\begin{lstlisting}
import opt_einsum as oe

class TensorNet(torch.nn.Module):
    def __init__(self, shapes, equation, device=None):
        super().__init__()
        self.shapes = shapes
        self.equation = equation
        self.device = device
        self.cores = torch.nn.ParameterList(
            [
                torch.nn.Parameter(torch.empty(*shape, device=self.device))
                for shape in self.shapes
            ]
        )

        # contraction order optimization 
        self.expression = oe.contract_expression(self.equation, *self.shapes)

        self.reset_parameters() # initialize cores

    def contract(self, *args, **kwargs):
        # reconstruc a tensor from cores with optimized contraction expression
        return self.expression(*self.cores, *args, **kwargs)

    def fit_als(self, target, **kwargs):
        # fit cores to a target tensor with ALS
        with torch.inference_mode():
            cores = als(self.equation, self.cores, target, **kwargs)
            for k, core in enumerate(cores):
                self.cores[k].copy_(core)

    def fit_grad(self, target, **kwargs):
        # gradient optimization
        lr = kwargs.get("lr", 1e-3)
        opt = torch.optim.AdamW(self.cores, lr=lr)
        for k in range(kwargs.get("steps", 1000)):
            opt.zero_grad()
            output = self.contract()
            loss = torch.nn.functional.mse_loss(output, target)
            loss.backward()
            opt.step()

\end{lstlisting}
Here, we use \texttt{opt\_einsum}\footnote{\url{https://optimized-einsum.readthedocs.io/en/stable/}} to optimize the contraction order.
We use \texttt{tensorly}\footnote{\url{https://tensorly.org/stable/index.html}} to initialize cores, and we can fine-tune them with gradient optimization.

In addition to the parameter reduction and flexible expressivity of tensor networks, the computational cost can be also reduced by designing the contraction order.
For example with the 4-mode TT for a site dimension of $d=32$ with a rank of $R=20$, a na\"{i}ve contraction from left to right requires $3.4\times 10^{10}$ FLOPs, whereas an optimized contraction order reduces it to $4.4\times 10^{7}$, nearly 3 orders of magnitude when using \href{https://optimized-einsum.readthedocs.io/en/stable/index.html}{\texttt{opt\_einsum}}.
The pseudo code to demonstrate it is as follows:
\begin{lstlisting}
# tensor reconstruction can be faster with contraction optimization
shape = [(32,20),(20,32,20),(20,32,20),(20,32)]
_, info = opt_einsum.contract_path("ip,pjq,qkr,rl->ijkl", *shape, shapes=True, optimize="optimal")
\end{lstlisting}
Furthermore, when the weight is decomposed in such a tensor network, we often do not need to explicitly reconstruct the weight before multiplying with the input activation.
For example, if the input activation is batched as a shape of $(64, 1024)$, the $(1024, 1024)$ weight reconstruction and the weight-input multiplication require $2.7\times 10^{12}$ FLOPs in total.
Whereas, it can be significantly reduced to $6.9\times 10^{6}$ FLOPs by optimizing the contraction order, achieving more than 5 orders of magnitude speedup, which is even faster than the weight materialization alone.
The pseudo code to show its benefit is as follows:
\begin{lstlisting}
# weight-activation multiplication can be even faster with contraction optimization
# Weight: W (1024, 1024) -> (32, 32, 32, 32); Input X (64, 1024) -> (64, 32, 32)
shape = [(32,20),(20,32,20),(20,32,20),(20,32),(64,32,32)] 
_, info = opt_einsum.contract_path("ip,pjq,qkr,rl,...ij->..kl", *shape, shapes=True, optimize="optimal")
\end{lstlisting}

\section{Rank reduction with tensor sorting}
\label{sec:sorting}

We show more empirical results suggesting that the sorting operation can greatly reduce the required rank.
Fig.~\ref{fig:spect-all} shows the tensorization loss in squared norm $\|W-\hat{W}\|^2$, where $W$ is the original weights matrix and $\hat{W}$ is a reconstructed version through a 2-mode TT (i.e., SVD).
Specifically, we analyze the eigen spectrum for pretrained weights of some LLM models, including Qwen3-0.6B, Qwen3-1.7B, Gemma3-4B, and Phi3.5-mini.
We also evaluate several sorting options, as given in the pseudo code:
\begin{lstlisting}
# Original eigen spectrum
S = torch.linalg.svdvals(W)
loss = S.square().flip(-1).cumsum(-1) 

# Full sorting
Wf, _ = torch.sort(W.view(-1)) 
Wf = Wf.reshape_as(W)

# Row-wise sorting
Wr, _ = torch.sort(W, dim=-1) 

# Row-wise group sorting
_, perm = torch.sort(W, dim=-1) # row-wise sort
perm, _ = torch.sort(perm.view(-1, groupsize), dim=-1) # revert permutation within group
Wg = torch.gather(W, dim=-1, index=perm.view_as(W))

# Row-wise sequential axis sorting
primes = sympy.factorint(W.shape[1]) # prime factorization
shape = [-1] + [p for p, e in primes.items() for _ in range(e)] # like [dout, 2,2,...,2]
Ws = W.view(*shape)
for axis in range(1, len(shape)):
    Ws, _ = torch.sort(Ws, dim=axis)
Ws = Ws.view_as(W)
\end{lstlisting}
The group sort means that we sort the row-wise elements but we do not care the ordering within each group of size $g$.
For example, we sort $8$ elements, but we do not care the ordering inside the first smallest $4$ elements or the last smallest $4$ elements.
For this case, the required bits will be reduced from $\lceil \log_2(n!) \rceil / n$ to
$\lceil \log_2(n!) /(\log_2(g!))^{n/g}\rceil/n$ bits per element.
The sequential sorting does round-robin sorting across axis.
For $m$-mode tensors with site dimension $d$, the required memory is $\lceil \log_2(d!) \rceil m/d$ bits.

As shown in Fig.~\ref{fig:spect}, increasing the rank improves the accuracy, while the improvement is slow when original weights are decomposed by SVD without sorting.
Notably, tensor sorting dramatically reduces reconstruction error, enabling near rank-1 reconstruction.
From log-log plots, we can see that row-wise sorting, group sorting (with a groupsize of $g=64$) and sequential axis sorting
have relatively higher errors than full sorting, whereas they are still more than 10 times lower than the one without sorting. 
Note that the full sorting for Qwen3-0.6B requires  $\lceil\log_2(2^{20}!)\rceil/2^{20}\simeq 18.6$ bits per weight, the row sorting requires $\lceil\log_2(2^{10}!)\rceil/2^{10}\simeq 8.6$ bits, the group sorting requires $\lceil\log_2(2^{10}!/(g!)^{2^{10}/g})\rceil/2^{10}\simeq 3.9$ bits, and
the sequential sorting needs
$5$ bits per weight.
Therefore, even with such an imperfect sorting to reduce the permutation memory, the index ordering has a great potential to improve the accuracy.
Also we observe that the trend is similar to all LLM models we consider.

\begin{figure}[t]
    \centering
    \begin{subfigure}{0.47\linewidth}
        \includegraphics[width=\linewidth]{figs/eigen_Qwen-Qwen3-0-6B_duo.png}
        \caption{Qwen3-0.6B (layers.0.self\_attn.v\_proj)}
    \end{subfigure}
    \hfill
    \begin{subfigure}{0.47\linewidth}
        \includegraphics[width=\linewidth]{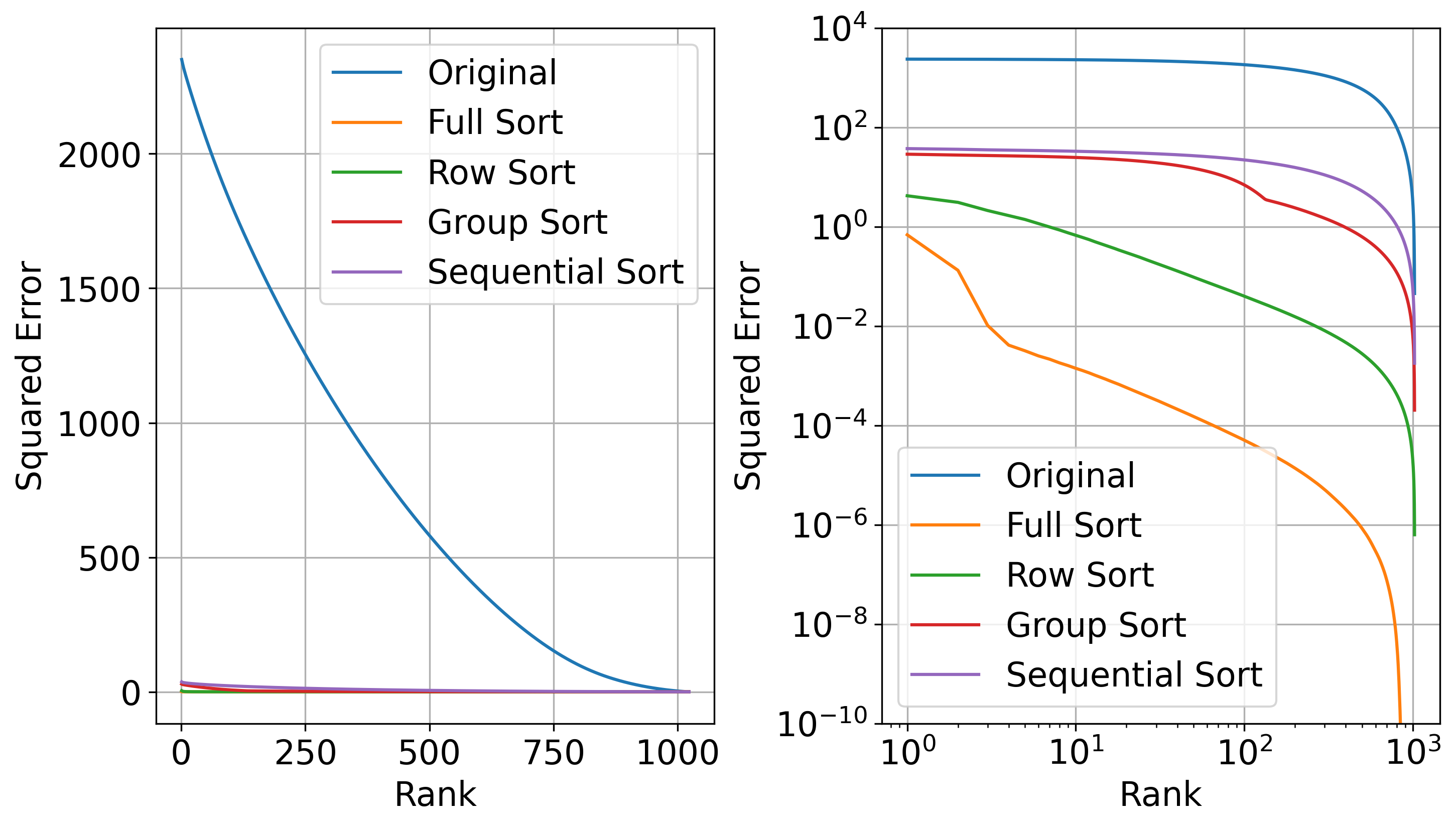}
        \caption{Qwen3-1.7B (layers.0.self\_attn.v\_proj)}
    \end{subfigure}
    \\
    \begin{subfigure}{0.47\linewidth}
        \includegraphics[width=\linewidth]{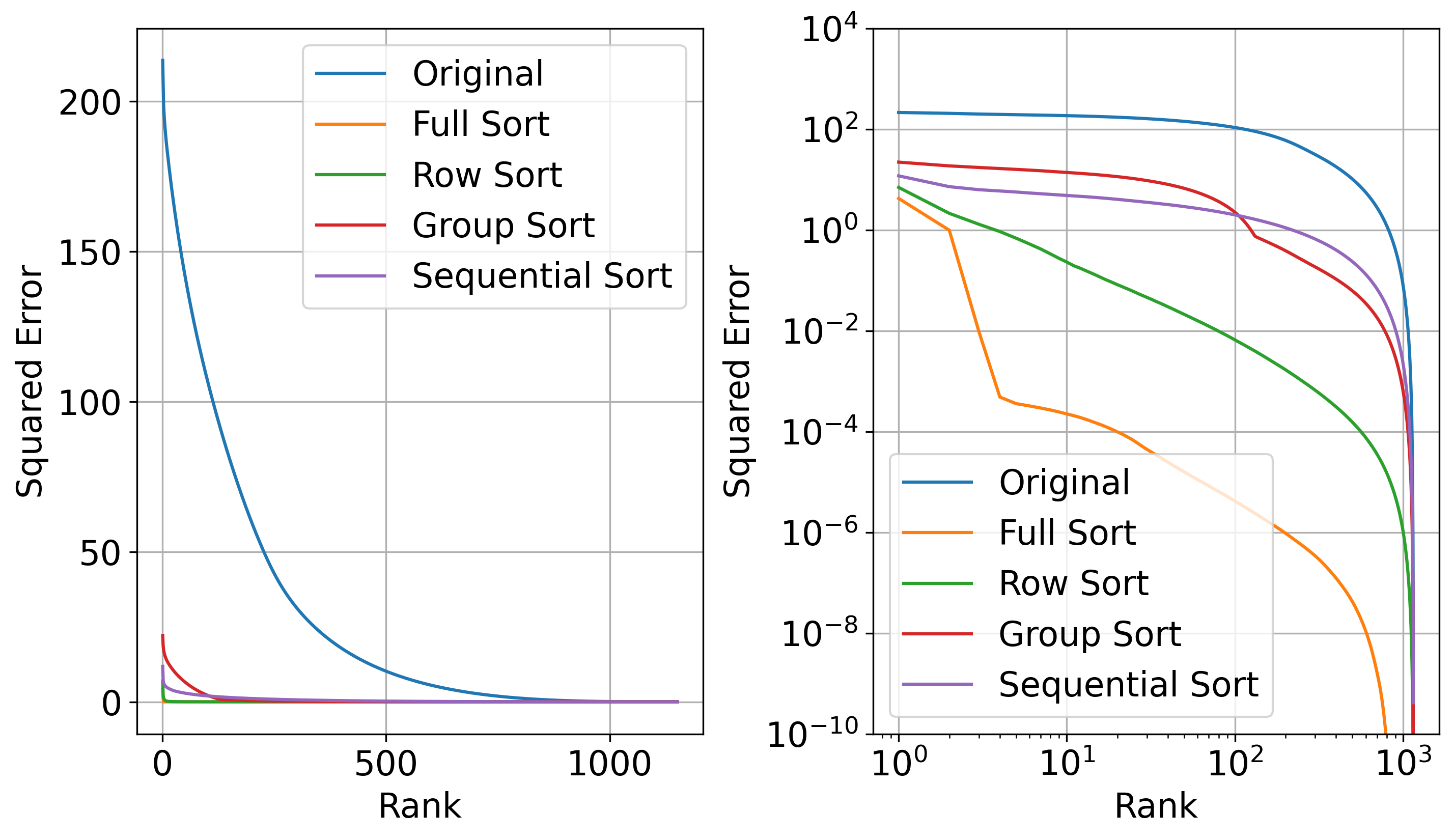}
        \caption{Gemma3-4B (layers.0.self\_attn.v\_proj)}
    \end{subfigure}
    \hfill
    \begin{subfigure}{0.47\linewidth}
        \includegraphics[width=\linewidth]{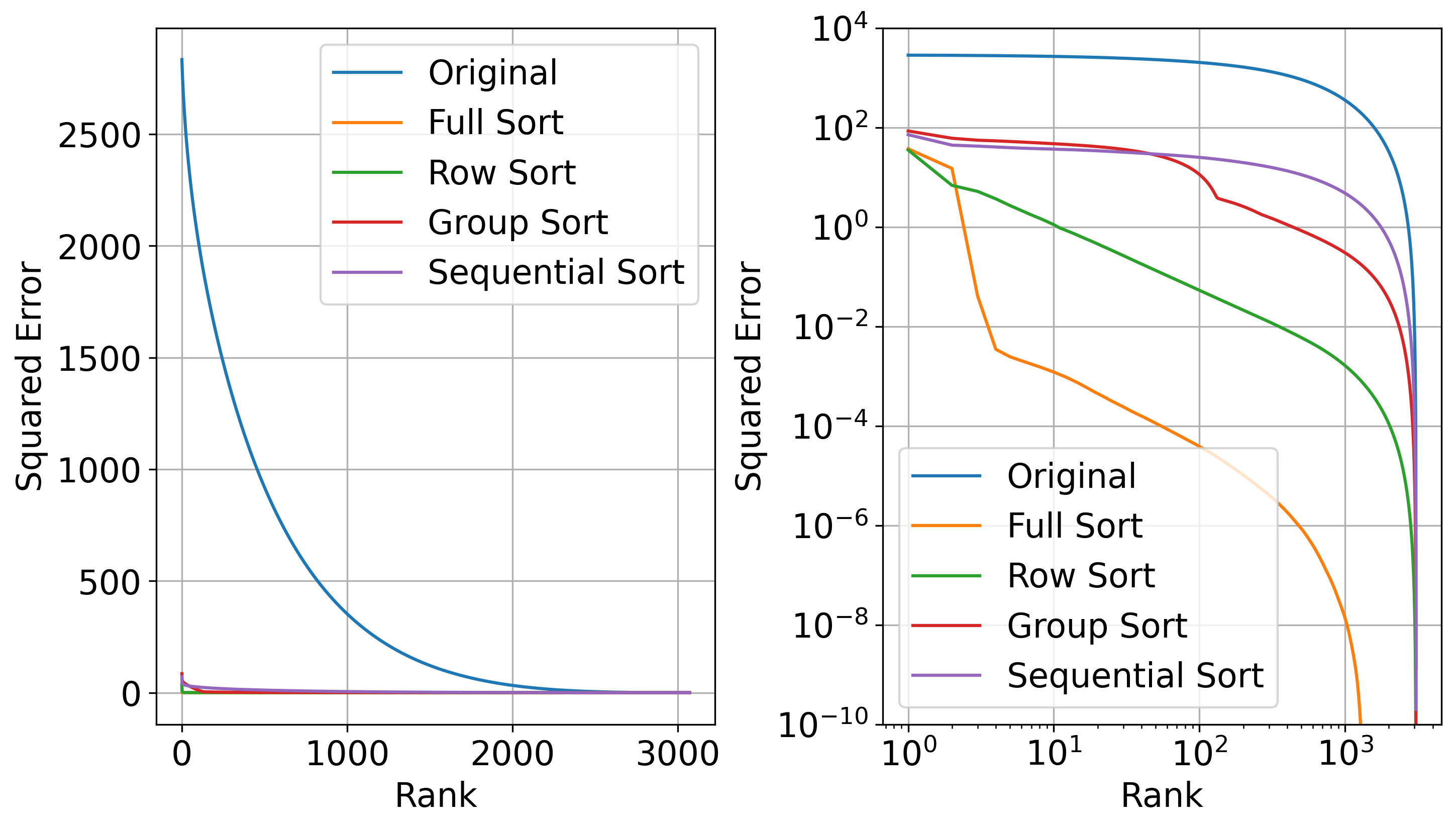}
        \caption{Phi3.5-mini (layers.0.self\_attn.o\_proj)}
    \end{subfigure}

    \caption{Tensorization error vs.\ rank for v\_proj weights at the first layer of Qwen3-0.6B model.
    Sorting can significantly reduce the required rank.
    Left: linear; right: log-log plots.}
    \label{fig:spect-all}
\end{figure}

\section{Theoretical analysis}
\label{sec:theory}

\subsection{Proof of Lemma~\ref{th:uniform}}

\begin{proof}
For $k\in\{0,\ldots,L-1\}$, the order statistic $X_{(k)}$ has the same
distribution as
\begin{align}
a+(b-a)Y_{(k)},
\end{align}
where $Y_{(k)}$ is the $(k+1)$-st order statistic of $L$ i.i.d. uniform
random variables on $[0,1]$. 
It is well known~\citep{david2004order} that $Y_{(k)}$ follows the Beta distribution:
\begin{align}
Y_{(k)}\sim \operatorname{Beta}(k+1,L-k).    
\end{align}
Hence, we have
\begin{align}
\mathbb{E}[X_{(k)}]
&=
a + (b-a)\mathbb{E}[Y_{(k)}]
=
a+(b-a)\frac{k+1}{L+1},
\end{align}
and
\begin{align}
\operatorname{Var}[X_{(k)}]
=
(b-a)^2
\frac{(k+1)(L-k)}{(L+1)^2(L+2)}.
    \end{align}
Since
\begin{align}
(k+1)(L-k)\leq \frac{(L+1)^2}{4},    
\end{align}
we have
\begin{align}
\operatorname{Var}[X_{(k)}]
\leq
\frac{(b-a)^2}{4(L+2)}
=
\mathcal{O}[L^{-1}].
\end{align}
Therefore, by Chebyshev's inequality,
\begin{align}
X_{(k)}
=
a+(b-a)\frac{k+1}{L+1}
+
\mathcal{O}_p[L^{-1/2}].    
\end{align}

Now setting $k=Ni+j$, we can write the matrix entry at $(i,j)$ as
\begin{align}
X_{i,j}
&=
a+(b-a)\frac{(Ni+j)+1}{L+1}
+
\mathcal{O}_p[L^{-1/2}]
\\
&=
N\frac{b-a}{L+1}i
+
\frac{b-a}{L+1}j
+
\Big(a+\frac{b-a}{L+1}\Big)
+
\mathcal{O}_p[L^{-1/2}].
\end{align}
Thus, in matrix form, we can write
\begin{align}
X
&=
N\frac{b-a}{L+1}\ell_M1_N^\top
+
\frac{b-a}{L+1}1_M\ell_N^\top
+
\Big(a+\frac{b-a}{L+1}\Big) 
1_M1_N^\top
+
\mathcal{E},
\end{align}
where every entry of $\mathcal{E}$ is $\mathcal{O}_p[L^{-1/2}]$.

The first three terms are rank-one matrices. Hence their sum has rank at
most $3$. Therefore $X$ is asymptotically approximated by a rank-at-most-$3$
matrix, with entrywise stochastic error $\mathcal{O}_p[L^{-1/2}]$.
\end{proof}

Note that the bias factor in the third term is gone when $a=-b/L$ as $a+(b-a)/(L+1)=0$.

\subsection{Random projection to make uniformity}

The preceding lemma can be extended heuristically beyond the uniform case.
Indeed, many random vectors become approximately Gaussian after an orthogonal
mixing transform due to the central limit effect. Since a Gaussian random
variable can be converted into a uniform random variable through its cumulative
distribution function, the sorted-matrix structure above approximately applies
to a broader class of distributions after an appropriate unitary projection.

More specifically, let
\begin{align}
x=[x_0,\ldots,x_{L-1}]^\top
\end{align}
be a random vector with independent entries having finite variance, not
necessarily uniformly distributed. Let $U\in\mathbb{R}^{L\times L}$ be a fixed
unitary (orthogonal) matrix whose entries are sufficiently delocalized, such
as the normalized Walsh--Hadamard matrix or discrete Fourier transform (DFT)
matrix, and define
\begin{align}
y = Ux.
\end{align}
Each component
\begin{align}
y_i = \sum_{k=0}^{L-1} U_{ik}x_k
\end{align}
is then a weighted sum of many independent random variables. Under standard
Lindeberg-type conditions, the central limit theorem implies that, as
$L\to\infty$, each $y_i$ converges in distribution toward a Gaussian random
variable:
\begin{align}
y_i \xrightarrow{d} \mathcal{N}(\mu_i,\sigma_i^2).    
\end{align}
After normalization,
\begin{align}
z_i = \varPhi\!\left(\frac{y_i-\mu_i}{\sigma_i}\right),    
\end{align}
where $\varPhi$ denotes the standard Gaussian cumulative distribution function,
we obtain approximately uniform random variables
\begin{align}
z_i \approx \mathcal{U}(0,1).    
\end{align}

Consequently, after a sufficiently mixing unitary transform and marginal
Gaussianization, a broad class of random matrices can be reduced
approximately to the uniform setting analyzed in
Lemma~\ref{th:uniform}. Therefore, the asymptotic low-rank structure of sorted
and folded matrices is expected to hold far beyond the exact uniform model.

This argument is heuristic because the transformed variables are generally not
exactly independent, and finite-dimensional convergence from the central limit
theorem does not directly imply joint uniformity. Nevertheless, for highly
mixing transforms such as Walsh--Hadamard matrices, the approximation becomes
accurate in high dimensions and is widely exploited in randomized numerical
linear algebra and signal processing.

\subsection{Row sorting}

Even without assuming uniform distribution or random projection, we can derive the following theory for unknown random parameters for row-wise sorting.

\begin{lemma}[Asymptotic rank-one structure of row-wise sorted matrices]
\label{th:row}
Let $X_{i,j}$ for $i \in\mathbb{Z}_M$ and $j\in\mathbb{Z}_N$ be \textit{i.i.d.} random variables with continuous distribution function $F$. Assume: $F$ has density $f$; there exists $c>0$ such that $f(x)\ge c$ on the support of $F$; and the quantile function $Q(u)=F^{-1}(u)$ for $u\in(0,1)$ satisfies as
\begin{align}
\int_0^1 Q(u)^2\,\mathrm{d}u<\infty.
\end{align}

For each row $i$, let
$X_{i,(1)}\le \cdots \le X_{i,(N)}$
denote the order statistics of $(X_{i,0},\dots,X_{i,N-1})$, and define the row-wise sorted matrix
$
S\in\mathbb R^{M\times N}
$ with
$
S_{i,j}=X_{i,(j+1)}$.    
Define the deterministic vector $q=(q_0,\dots,q_{N-1})^\top$ with   
$
q_j = Q((j+1)/(N+1))$.

Then, the row-wise sorted matrix $S$ is asymptotically rank one as $1_M q^\top$ for $N\gg 1$, i.e.,
\begin{align}
\frac{\|S-1_M q^\top\|_\mathrm{F}}
{\| 1_M q^\top\|_\mathrm{F}}
=
O_p(N^{-1/2}).
\end{align}
\end{lemma}

\begin{proof}
Define
\begin{align}
U_{i,j}=F(X_{i,j}).    
\end{align}
By the probability integral transform, the variables $U_{i,j}$ are \textit{i.i.d.} uniform on $[0,1]$. 
Since $Q=F^{-1}$,
\begin{align}
X_{i,(j+1)} = Q(U_{i,(j+1)}),    
\end{align}
where $U_{i,(j+1)}$ is the $(j+1)$-st order statistic of $N$ \textit{i.i.d.} uniform random variables.

The uniform order statistics satisfy
\begin{align}
\mathbb E[U_{i,(j+1)}]
&=
\frac{j+1}{N+1},    
\\
\operatorname{Var}(U_{i,(j+1)})
&=
\frac{(j+1)(N-j)}
{(N+1)^2(N+2)}
\leq
\frac{1}{4(N+2)}
=
O(N^{-1}).
\end{align}

Since $f(x)\ge c>0$, the quantile function $Q(u)$ is Lipschitz. 
Indeed,
\begin{align}
Q'(u)=\frac{1}{f(Q(u))},    
\end{align}
hence
\begin{align}
|Q'(u)|\le \frac1c.    
\end{align}

Therefore, for some constant $C>0$,
\begin{align}
|Q(u)-Q(v)|\le C|u-v|.    
\end{align}

Using this Lipschitz property,
\begin{align}
\begin{aligned}
\mathbb E\left[
\left(
X_{i,(j+1)}-q_j
\right)^2
\right]
&=
\mathbb E\left[
\left(
Q(U_{i,(j+1)})
-
Q\!\left(\frac{j+1}{N+1}\right)
\right)^2
\right]
\\
&\le
C^2
\operatorname{Var}(U_{i,(j+1)})
\\
&=
O(N^{-1}).
\end{aligned}    
\end{align}

Summing over $j$,
\[
\sum_{j=0}^{N-1}
\mathbb E\left[
\left(
X_{i,(j+1)}-q_j
\right)^2
\right]
=
O(1).
\]

Now define the error matrix
\begin{align}
\mathcal{E}_{i,j}=X_{i,(j+1)}-q_j.    
\end{align}
Then
\begin{align}
\mathbb E\|\mathcal{E}\|_\mathrm{F}^2
&=
\sum_{i=0}^{M-1}
\sum_{j=0}^{N-1}
\mathbb E[\mathcal{E}_{i,j}^2]
\\
&=
M\cdot O(1)
\\
&=
O(M).
\end{align}

Hence,
\begin{align}
\|\mathcal{E}\|_\mathrm{F} = O_p(\sqrt{M}).    
\end{align}

Next,
\begin{align}
\|1_M q^T\|_\mathrm{F}^2
=
M\|q\|_2^2
=
M\sum_{j=0}^{N-1} q_j^2.    
\end{align}

Since $q_j=Q((j+1)/(N+1))$,
\begin{align}
\frac1N\sum_{j=0}^{N-1} q_j^2
\to
\int_0^1 Q(u)^2\,\mathrm{d}u.    
\end{align}

Therefore,
\begin{align}
\sum_{j=0}^{N-1} q_j^2
=
\Theta(N),    
\end{align}
which implies
\begin{align}
\|1_M q^\top\|_\mathrm{F}
=
\Theta(\sqrt{MN}).    
\end{align}

Finally,
\begin{align}
\frac{\|\mathcal{E}\|_\mathrm{F}}
{\|1_M q^\top\|_\mathrm{F}}
=
\frac{O_p(\sqrt M)}
{\Theta(\sqrt{MN})}
=
O_p(N^{-1/2}).    
\end{align}

Since
\begin{align}
S= 1_M q^\top+\mathcal{E},    
\end{align}
the matrix $S$ converges to the rank-one matrix $1_M q^\top$ in relative Frobenius norm.
\end{proof}

\subsection{Independent permutation}

The row sorting may give a rank-one structure as in Lemma~\ref{th:row}.
However, any independent permutation per tensor axis has no benefit.

\begin{proposition}[Invariance of singular values and rank under independent permutations]
Let $X \in \mathbb{R}^{M \times N}$ be a matrix, and let
$P_{\mathrm r} \in \mathbb{R}^{M \times M}$ and
$P_{\mathrm c} \in \mathbb{R}^{N \times N}$
be arbitrary row and column permutation matrices, respectively.
Then,
\begin{align}
\sigma(X) = \sigma(P_{\mathrm r} X P_{\mathrm c}),
\end{align}
where $\sigma(\cdot)$ denotes the set of singular values.
Consequently,
\begin{align}
\operatorname{rank}(X)
=
\operatorname{rank}(P_{\mathrm r} X P_{\mathrm c}).
\end{align}
\end{proposition}

\begin{proof}
Let the SVD of $X$ be
\begin{align}
X = U \Sigma V,
\end{align}
where $U \in \mathbb{R}^{M \times M}$ and
$V \in \mathbb{R}^{N \times N}$ are orthogonal matrices, and
$\Sigma \in \mathbb{R}^{M \times N}$ is diagonal rectangular with the
singular values of $X$ on its diagonal.

Since permutation matrices are orthogonal, we have
\begin{align}
P_{\mathrm r}^\top P_{\mathrm r} = I_M,
\qquad
P_{\mathrm c} P_{\mathrm c}^\top = I_N.
\end{align}
Therefore, $P_{\mathrm r}U$ and $VP_{\mathrm c}$ are also orthogonal, because
\begin{align}
(P_{\mathrm r}U)^\top(P_{\mathrm r}U)
=
U^\top P_{\mathrm r}^\top P_{\mathrm r} U
=
U^\top U
=
I_M,
\\
(VP_{\mathrm c})(VP_{\mathrm c}^\top
=
V P_{\mathrm c} P_{\mathrm c}^\top V^\top
=
V V^\top
=
I_N.
\end{align}

Now,
\begin{align}
P_{\mathrm r} X P_{\mathrm c}
&=
P_{\mathrm r} U \Sigma V P_{\mathrm c} \\
&=
(P_{\mathrm r}U)\Sigma(VP_{\mathrm c}) .
\end{align}
This is an SVD of $P_{\mathrm r} X P_{\mathrm c}$ with the same diagonal
matrix $\Sigma$. Hence, $P_{\mathrm r} X P_{\mathrm c}$ and $X$ have the same
singular values.

Since the rank of a matrix is equal to the number of nonzero singular values,
the rank is also invariant under independent row and column permutations.
\end{proof}

Comparing Lemma~\ref{th:row}, this trivial theory suggests that we should use axis-dependent permutation like CNOT entanglement, i.e., the permutation $\pi$ for the column index $j$ should be dependent on the row index $i$.
From this perspective, TQCompress~\cite{abronin2024tqcompressor} should have no benefit in the sense of singular values and tensor ranks as it uses row and column independent permutations.

\subsection{Reduction impact}

The sliced sorting with reduction can reduce the memory requirement.
However, in return it can degrade the sorting accuracy.
Under some assumptions, we can provide the potential impact when increasing  the reduction size.
\begin{proposition}[Scaling law for mean and product reduction on uniform random variables]
    \label{th:reduction}
    Let $X_{i,j}$ be uniform \textit{i.i.d.} random variables uniformly distributed on $[-1,1]$ for $i\in\mathbb{Z}_M$ and $j\in\mathbb{Z}_N$.
    Let the mean and product reductions after power scaling across column $j$ be defined as $Y_{i} = \sum_{j=1}^M |X_{i,j}|^p /M$ and $Z_{i} = \prod_{j=1}^M |X_{i,j}|^p$, respectively, where $p>0$ is power exponent.
    Then, the reduced random variables $Y_i$ and $Z_i$ conditioned on $X_{i,j}$ have moments:
    \begin{align}
        \mathbb{E}[Y_i \mid X_{i,j}] &=
        \frac{X_{i,j}}{M} + 
        \frac{M-1}{M(p+1)},
        \\
        \mathrm{Var}[Y_i \mid X_{i,j}]
        &=
        \frac{M-1}{M^2}
        \Big(
        \frac{1}{2p+1}
        -
        \frac{1}{(p+1)^2}
        \Big),
        \\
        \mathbb{E}[Z_i\mid X_{i,j}]
        &=
        \frac{X_{i,j}}{(p+1)^{M-1}}
        ,\\
        \mathrm{Var}[Z_i\mid X_{i,j}]
        &=
        X_{i,j}^2
        \Big(
        \frac{1}{(2p+1)^{M-1}}
        -
        \frac{1}{(p+1)^{2(M-1)}}
        \Big)
        .
    \end{align}
    Hence, the error margin to detect $X_{i,j}$, i.e., target factor in mean over standard deviation, will decay as follows:
    \begin{align}
    \frac{
        \mathbb{E}[Y_i\mid X_{i,j}] - (M-1)/M(p+1)
    }{\mathrm{Var}[Y_i\mid X_{i,j}]^{1/2}}
    &=
    \frac{X_{i,j}}{\sqrt{M-1}}
    \Big(
        \frac{1}{2p+1}
        -
        \frac{1}{(p+1)^2}
    \Big)^{-1/2}
    \mathop{\longrightarrow}_{M\gg 1}
    X_{i,j}\mathcal{O}[1/M^{1/2}],
    \\
    \frac{
        \mathbb{E}[Z_i\mid X_{i,j}]
    }{\mathrm{Var}[Z_i\mid X_{i,j}]^{1/2}}
    &=
    \Big(
        \frac{(p+1)^{2(M-1)}}{(2p+1)^{M-1}}
        -
        1
    \Big)^{-1/2}
    \mathop{\longrightarrow}_{M\gg 1}
    \mathcal{O}[((2p+1)/(p+1)^2)^{M/2}].
    \end{align}
    Hence, the mean reduction decays the error margin slower than the product reduction when increasing the reduction size $M$.
\end{proposition}

\begin{proof}
    Let $X$ is the $p$-th absolute power of uniform random variable $U\sim \mathcal{U}(-1,1)$, we have the CDF
    \begin{align}
        F_X(x) = \mathbb{P}(|U|^p<x)
        = x^{1/p}, \qquad 0\leq x\leq 1.
    \end{align}
    Its PDF is then
    \begin{align}
        f_X(x) = \frac{1}{p} x^{1/p-1},
    \end{align}
    which is Beta distribution $\mathrm{Beta}(1/p,0)$.
    The moments are hence given as
    \begin{align}
        \mathbb{E}[X^k] &=
        \mathbb{E}[|U|^{kp}]
        =
        \frac{1}{pk+1}.
    \end{align}

    Let $Y$ be the mean reduction of $M$ \textit{i.i.d.} Beta-distributed variables $X_1,\ldots, X_M$ conditioned on $X_1$,  defined as
    \begin{align}
        Y\mid X_1 = \frac{1}{M}(X_1 + \sum_{i=2}^M X_i).
    \end{align}
    Its mean is thus given as
    \begin{align}
        \mathbb{E}[Y\mid X_1] &=
        \frac{X_1}{M} + \frac{M-1}{M} \mathbb{E}[|U|^p]
        \\
        &=
        \frac{X_1}{M} + \frac{M-1}{M(p+1)}.
    \end{align}
    And, its variance is given as
    \begin{align}
        \mathrm{Var}[Y|X_1] &=
        \frac{M-1}{M}\mathrm{Var}[|U|^p]
        \\
        &=
        \frac{M-1}{M}
        \Big(
        \frac{1}{2p+1} -
        \Big(
        \frac{1}{p+1}\Big)^2
        \Big).
    \end{align}

    Let $Z$ be the product reduction of $X_1,\ldots, X_M$ conditioned on $X_1$, defined as
    \begin{align}
        Z \mid X_1 &=
        X_1 \prod_{i=2}^M X_i.
    \end{align}
    Then, its mean is given as
    \begin{align}
        \mathbb{E}[Z\mid X_1] &=
        X_1 \big( \mathbb{E}[|U|^p] \big)^{M-1}
        \\
        &=
        X_1 \Big(
        \frac{1}{p+1}
        \Big)^{M-1}
        .
    \end{align}
    Also the second moment is
    \begin{align}
        \mathbb{E}[Z^2\mid X_1] &=
        X^2 \big( 
        \mathbb{E}[|U|^{2p}]
        \big)^{M-1}
        \\
        &=
        X^2 \Big(
        \frac{1}{2p+1}
        \Big)^{M-1}
    \end{align}
    Hence, the variance is given as
    \begin{align}
        \mathrm{Var}[Z\mid X_1] &=
        \mathrm{E}[Z^2\mid X_1]
        -
        \big(\mathbb{E}[Z\mid X_1]\big)^2
        \\
        &=
        X_1^2 \Big( 
        \Big(
        \frac{1}{2p+1}
        \Big)^{M-1}
        -
        \Big(
        \frac{1}{p+1}
        \Big)^{M-1}
        \Big).
    \end{align}
\end{proof}

For Gaussian random variables, it appears the similar scaling law.
Note that the mean reduction has a target-dependent margin increasing with $X_{i,j}$, and thus the impact of the stochastic error is more severe at smaller value of $X_{i,j}$.
Whereas, the product reduction has a constant margin across $X_{i,j}$, which may lead more reliable sliced sorting.
For example, when we have $p=1$ and $M=8$, the error margin for the value $X_{i,j}=0.1$ will be about $0.13$ for the mean reduction, while it will be $0.41$ for the product reduction.

\section{EinSort}
\label{sec:einsort}

Those results motivate us to consider sorting for tensor network designs.
We then propose \textbf{EinSort} framework based on the Einstein sorted sum.
Conceptually, we use a reversible  permutation operation $\pi[\cdot]$ inside tensor decomposition.
Let $\mathcal{T}_\theta[\cdot]$ denote the tensor decomposition with hyperparameters $\theta$ which determine einsum equation, shapes, topology, etc.
The sorted tensor decomposition is written as $\mathcal{T}_\theta^\pi :=\pi^{-1}[\mathcal{T}_\theta[\pi[X]]]$.
It generalizes the einsum to find low-rank structure.

Although sorting may reduce the number of parameters for tensor cores with few rank, the memory overhead of storing permutations is not negligible.
Specifically, sorting $L$ values requires at most $\lceil \log_2(L!) \rceil$ bits using factoradic/Lehmer code~\citep{lehmer1960teaching}.
For example, full sorting for a tensor of shape $(1024, 1024)$ requires roughly $19$ bits per entry, which often exceeds the memory footprint of the original tensors in FP16.
Therefore, we need to seek optimizing a tradeoff between sorting accuracy and tensor rank reduction.

We consider several simple sorting methods in this paper, while many other options could be explored.
One of simplest approaches is to reuse a shared permutation for multiple tensors.
To do so, we first need to determine what kind of sorting metrics is relevant for multi-tensor permutation.
Considering a 4-mode tensor $X$ of shape $(d, d, d, d)$,
we may use a shared permutation across the first mode slice: $X_{:,j,k,l}$ to re-order the last three modes for $\{j,k,l\}$.
Then, we can reduce the memory of permutation from $\lceil \log_2(d^4!)\rceil$ to $\lceil \log_2(d^3!) \rceil$ bits, i.e., more than $d$-fold memory reduction.
Therefore, adjusting the slicing dimension, we can easily reduce the total memory footprint, towards even lower than 1-bit per weight.
For example, with $d=32$, we have $\lceil \log_2(d^3!)\rceil / d^4 \simeq 0.4$ bits per weight.
For sorting metric, we consider some options for nonlinear mapping and reduction, including power scaling, and mean/max/std/var/median/prod operations.
An example using power scaling and std reduction is listed in the pseudo code as below.
\begin{lstlisting}
# X (1024, 1024)
X = X.reshape([2] * 20) # folding to 20-mode tensor of shape (2,2,...,2)
S = X.abs().pow(p) # scoring with p-th power mapping
axis = [0,1,2] # slicing modes (the first 3 modes for simplicity) 
S = S.std(dim=axis, keepdim=True).expand_as(X) # std reduction
perm = torch.argsort(S.view(2,2,2,-1), dim=-1) # shared permutation
X = X.view(2,2,2,-1).gather(-1, perm).view([2] * 20) # re-order
\end{lstlisting}

It has a flexibility to adjust the choices of slicing modes, power exponent, and reduction operation. 
We specifically consider six reduction operations:
\begin{itemize}
    \item Mean reduction: \texttt{torch.mean()}
    \item Max reduction: \texttt{torch.max()}
    \item Min reduction: \texttt{torch.min()}
    \item Median reduction: \texttt{torch.median()}
    \item Standard deviation reduction \texttt{torch.std()}
    \item Product reduction: \texttt{torch.prod()}
\end{itemize}
In addition, we consider several nonlinear mapping before reduction:
\begin{itemize}
    \item Linear: no mapping
    \item Power: \texttt{x.pow(p)}
    \item Exponential: \texttt{x.mul(p).clamp(-10,10).exp()}
    \item Additive logarithm: \texttt{x.add(p).clamp(4.4e-5,2.2e4).log()}
    \item Multiplicative logarithm: \texttt{x.mul(p).clamp(4.4e-5,2.2e4).log()}
\end{itemize}

We search for effective folding strategy and mode to increase the sorting accuracy, while keeping lower memory cost.
Appendix~\ref{sec:reduction} compares different mapping and reductions.

In addition to the sorting metric, we also introduce an approach to improve the accuracy by employing non-negative tensor decomposition.
Specifically, we also project the original tensor into non-negative values before tensor decomposition.
To recover the negative values, we keep the sign information of the original tensor, and hence it requires only one additional bit.
It is illustrated in Fig.~\ref{fig:slicing}.

The pseudo code with non-negative tensorization is given as follows.
\begin{lstlisting}
# X (1024, 1024)
X = X.reshape([2] * 20) # folding to 20-mode tensor of shape (2,2,...,2)
S = X.abs().pow(p) # scoring with p-th power mapping

if nn_flag: # non-negative tensorization flag
    X, B = S, X.sign() # X is now non-negative part, and B keeps its sign information

axis = [0,1,2] # slicing modes (the first 3 modes for simplicity) 
S = S.std(dim=axis, keepdim=True).expand_as(X) # std reduction
perm = torch.argsort(S.view(2,2,2,-1), dim=-1) # shared permutation
X = X.view(2,2,2,-1).gather(-1, perm).view([2] * 20) # re-order
\end{lstlisting}
The sorted tensor now has non-negative elements for tensorization.
After tensor reconstruction, the sign information is used to recover the original elements.
Specifically, the reconstruction process is written as follows:
\begin{lstlisting}
# tensor decomposition for sorted X(2,2,...,2)
factors = tensor_decomp(X, gauge=True) # gauge fixing to save memory

# tensor reconstruction: X is approx sorted non-negative tensor
X = tensor_reconst(factors, gauge=True)

# inverse permutation
perm_inv = torch.argsort(perm, dim=-1)
X = X.view(2,2,2,-1).gather(-1, perm_inv).view([2] * 20) 

if nn_flag: 
    X = X.clamp(0).pow(1/p) # approx. inverse nonlinear mapping
    X = X * B # inverse sign

# unfolding back to original tensor shape
X = X.reshape(1024, 1024) 
\end{lstlisting}

Because the permutation is reversible, no information is discarded prior to tensor decomposition.
Note that certain reversible permutations used in EinSort can be interpreted as classical analogues of entangling operations in tensor-network-based quantum circuits as discussed in Appendix~\ref{sec:quantum}.

\section{Nonlinear mapping and reduction}
\label{sec:reduction}

Besides power scaling in Fig.~\ref{fig:pow}, we compare with different nonlinear mapping in Fig.~\ref{fig:nonlinear}.
We observe that none of them achieves good performance except for some reductions of exponential mapping.
Nevertheless, power scaling offers more accurate sorting than exponential mapping.
Table~\ref{tab:nonlinear} lists the PPL score at the best parameter with/without gauge fixing.
For most cases, the product reduction was best except for exponential scaling, which prefers the max reduction. 
Nevertheless, the max and mean reductions with power scaling offer the third and fourth best performance of $40.25$ and $42.38$, respectively.
Appendix~\ref{sec:theory} gives an insight for the difference of reduction operations, where we showed that product reduction is insensitive to the target value, while the mean reduction can be worse in low-magnitude regimes.
All cases, the gauge fixing improves the accuracy.

Fig.~\ref{fig:phi-pow} shows the case for Phi3-mini model at 50\% and 80\% compression rates.
The trend is different from Qwen3 case in Fig.~\ref{fig:pow}.
Specifically, best power exponent is around 0.4 when non-negative tensorization is used, whereas larger than $2$ may be best otherwise.
Also, all reductions are comparable, achieving the best score at the similar exponent.
Nonetheless, the gauge fixing and non-negative tensorizing are effective to compete with full sorting performance.
For Phi3 case, the required memory for shared permutation is $0.35$~bits per weight.

\begin{figure*}[t]
\centering
\begin{subfigure}{0.31\linewidth}
\includegraphics[width=\linewidth]{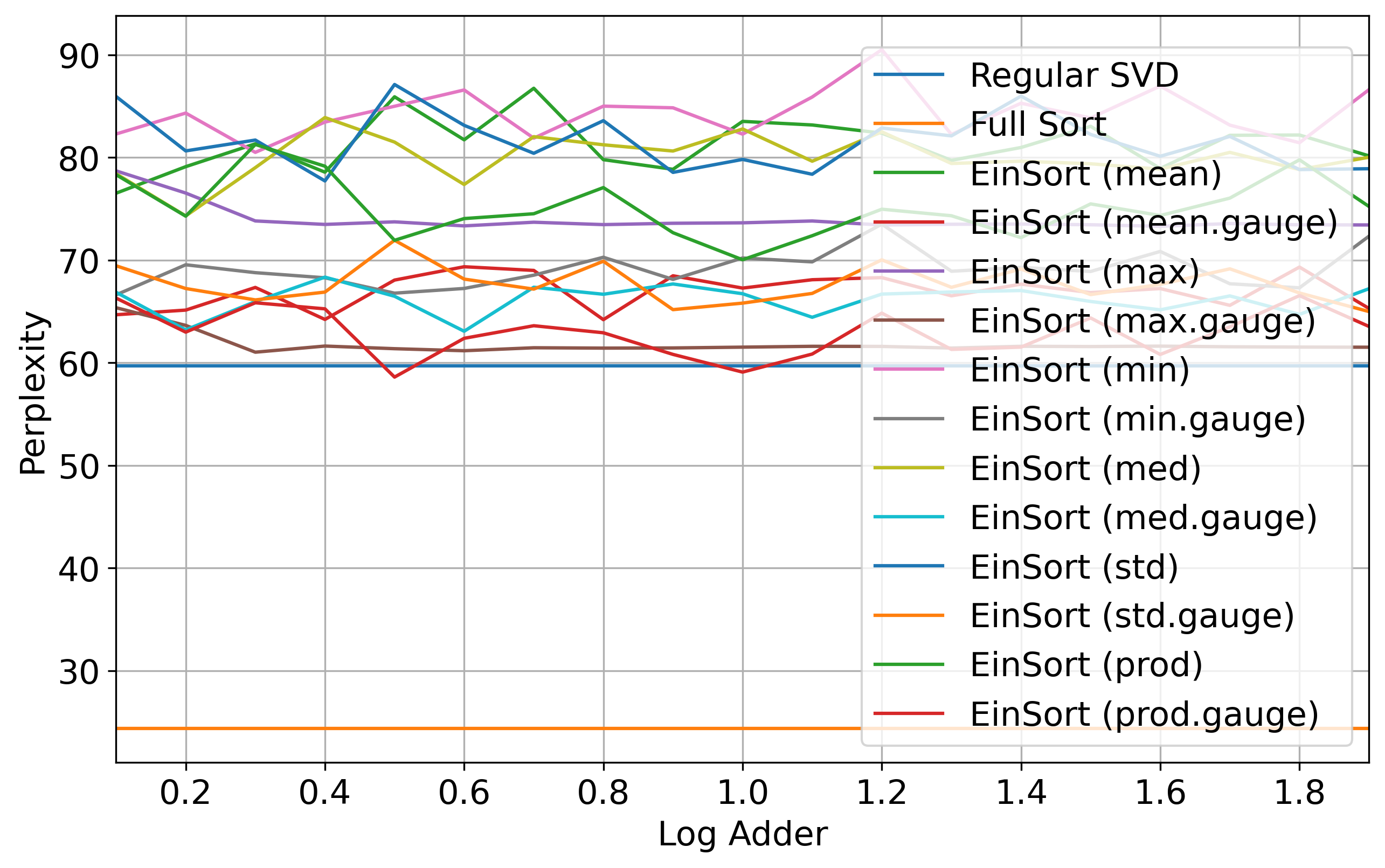}
\caption{Additive Logarithm}    
\end{subfigure}
\begin{subfigure}{0.31\linewidth}
\includegraphics[width=\linewidth]{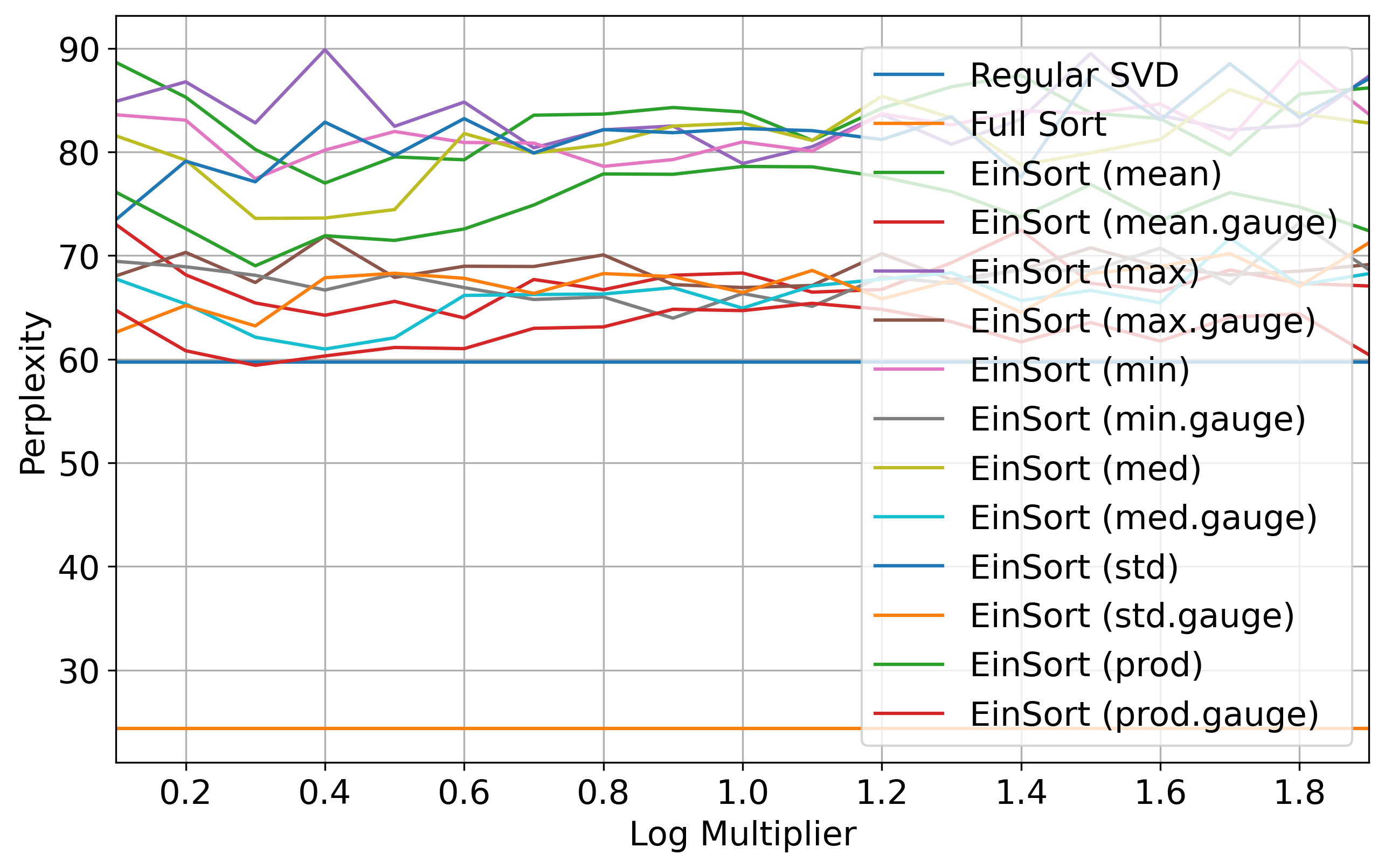}
\caption{Multiplicative Logarithm}    
\end{subfigure}
\begin{subfigure}{0.31\linewidth}
\includegraphics[width=\linewidth]{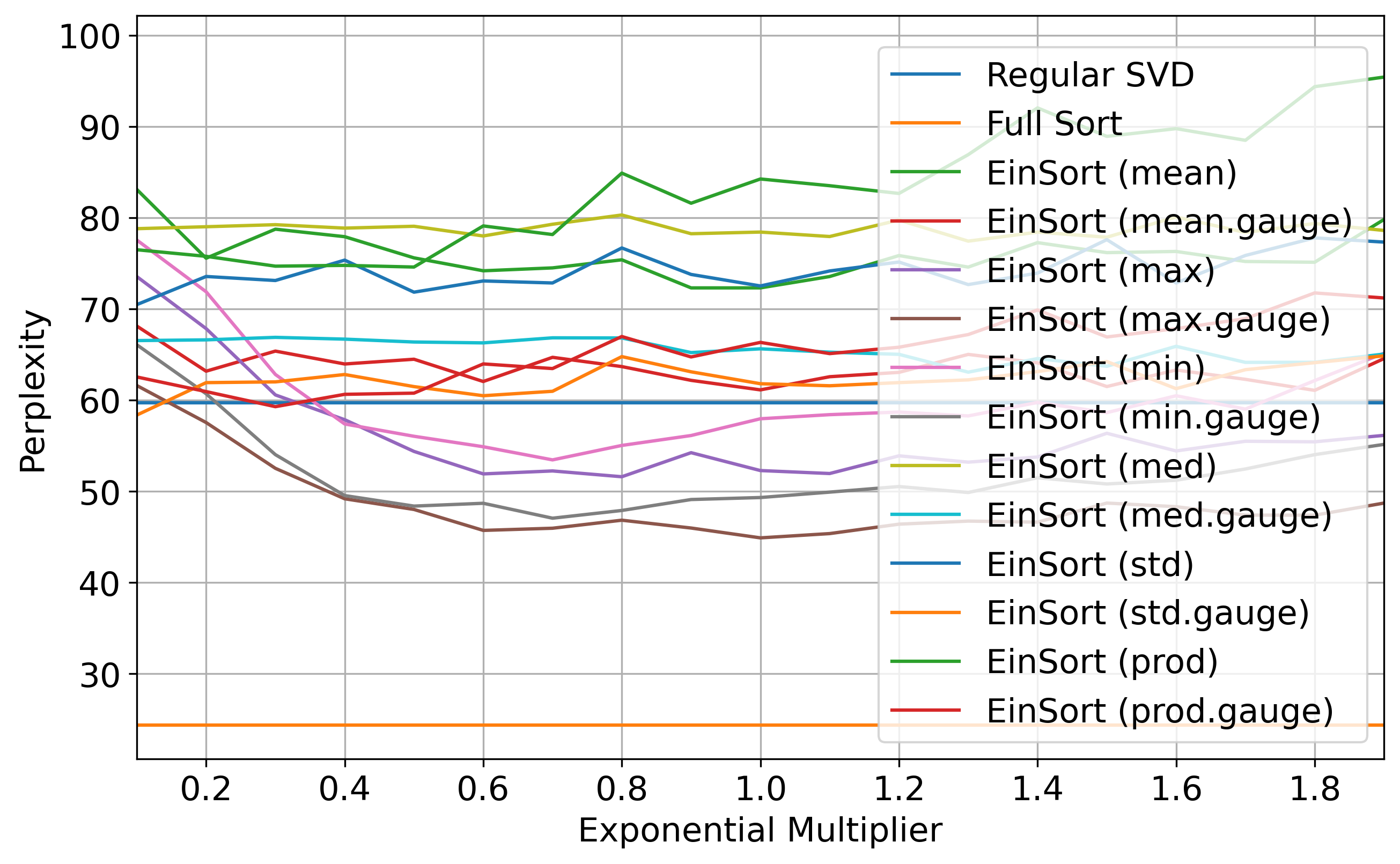}
\caption{Exponentiation}    
\end{subfigure}
\caption{Different nonlinear mapping for Qwen3-1.7B at 50\% compression: additive logarithm; multiplicative logarithm; and exponentiation.}
\label{fig:nonlinear}
\end{figure*}

\begin{table}[t]
    \centering
    \caption{PPL with different nonlinear mapping and reduction for Qwen3-1.7B at 50\% compression. 
    The four best scores are highlighted.
    The regular SVD without sorting has $60.0$ PPL.}
    \label{tab:nonlinear}
    \begin{tabular}{lcc cccc}
    \toprule
       Reduction  & mean & max & min & median & std & prod \\
       \midrule
       linear w/o gauge fixing &
       77.13 & \underline{73.55} & 78.12 & 79.07 & 73.82 & \textbf{59.56}
       \\
       linear w/ gauge fixing &
       62.79 & 61.49 & 66.04 & 66.34 & \underline{61.19} & \textbf{50.84}
       \\
       pow w/o gauge fixing & 
       77.46 & 73.54 & 78.08 & 79.05 & \underline{72.46} & \textbf{59.43}
       \\
       pow w/ gauge fixing & 
       63.60 & 61.44 & 66.03 & 66.30 & \underline{57.42} & \textbf{50.56}
       \\
       abs.pow w/o gauge fixing & 
       \underline{63.75} & 75.25 & 81.11 & 65.47 & 69.91 & \textbf{52.12} 
       \\
       abs.pow w/ gauge fixing & 
       \underline{55.26} & 63.25 & 67.52 & 56.22 & 60.78 & \textbf{44.72}
       \\
       pow.nn w/o gauge fixing & 
       \underline{50.08} & 46.56 & 68.21 & 55.66 & 51.37 & \cellcolor{lgreen}\textbf{36.29}
       \\
       pow.nn w/ gauge fixing & \cellcolor{lgreen}42.38 & \cellcolor{lgreen}\underline{40.25} & 52.70 & 44.21 & 46.79 & \cellcolor{lgreen}\textbf{33.17}         
       \\
       exp w/o gauge fixing & 72.31 & \textbf{51.60} & \underline{53.45} & 77.44 & 70.48 & 74.60
         \\
       exp w/ gauge fixing & 61.07 & \textbf{44.90} & \underline{47.06} & 58.37 & 63.07 & 59.29
       \\
       add.log w/o gauge fixing & 76.52 & \underline{73.30} & 80.50 & 77.72 & 74.31 & \textbf{70.04}
         \\
       add.log w/ gauge fixing & 64.20 & \underline{61.03} & 66.62 & 64.99 & 63.07 & \textbf{58.61}
       \\
       mul.log w/o gauge fixing & 76.99 & 78.88 & 77.42 & \underline{73.48} & 73.58 & \textbf{69.01}
         \\
       mul.log w/ gauge fixing & 63.98 & 66.91 & 63.96 & 62.60 & \underline{60.97} & \textbf{59.41}
       \\
         \bottomrule
    \end{tabular}
\end{table}

\begin{figure}
    \centering
    \begin{subfigure}{0.45\linewidth}
    \includegraphics[width=\linewidth]{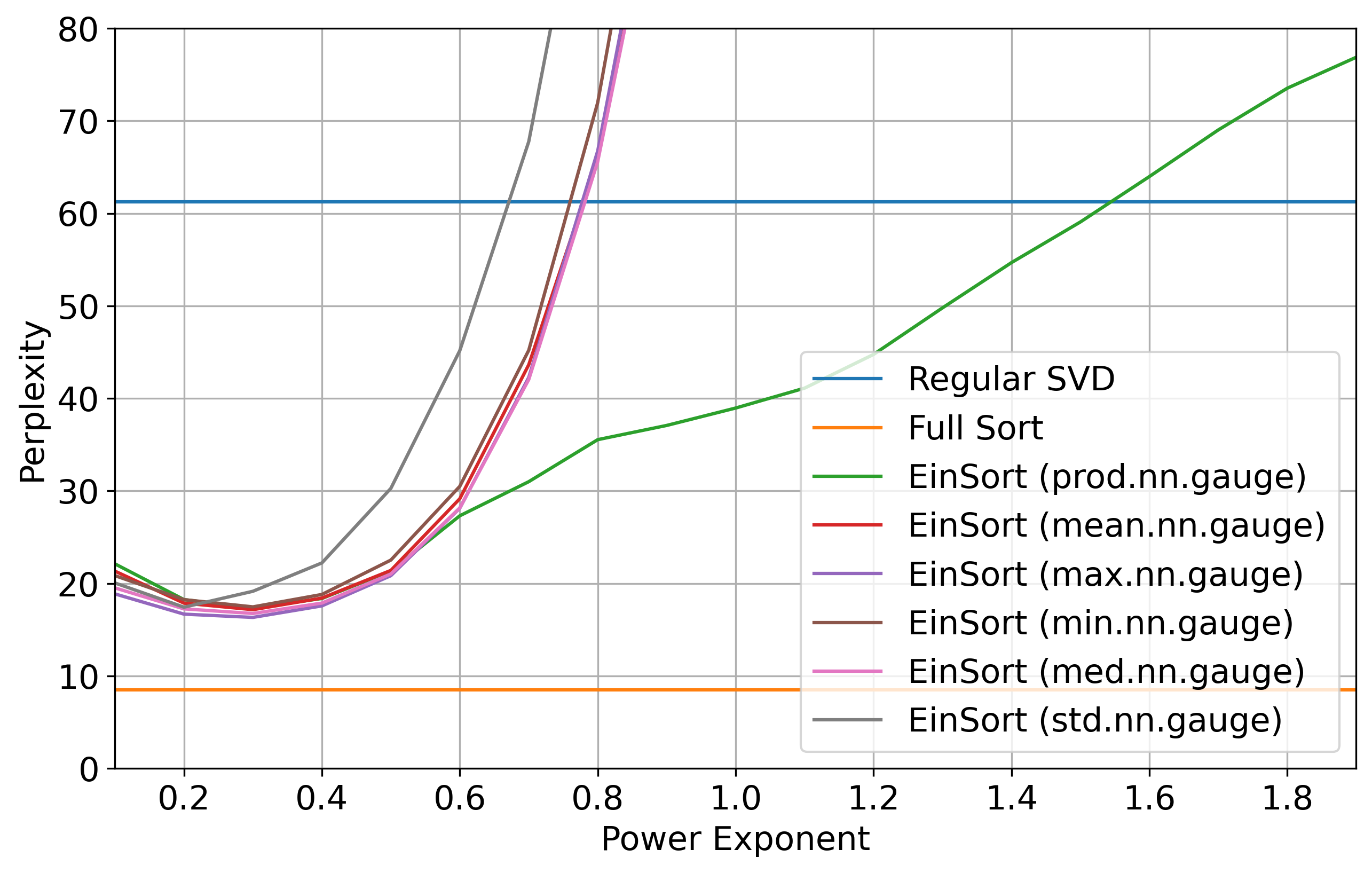}
    \caption{50\% compression}    
    \end{subfigure}
    \begin{subfigure}{0.45\linewidth}
        \includegraphics[width=\linewidth]{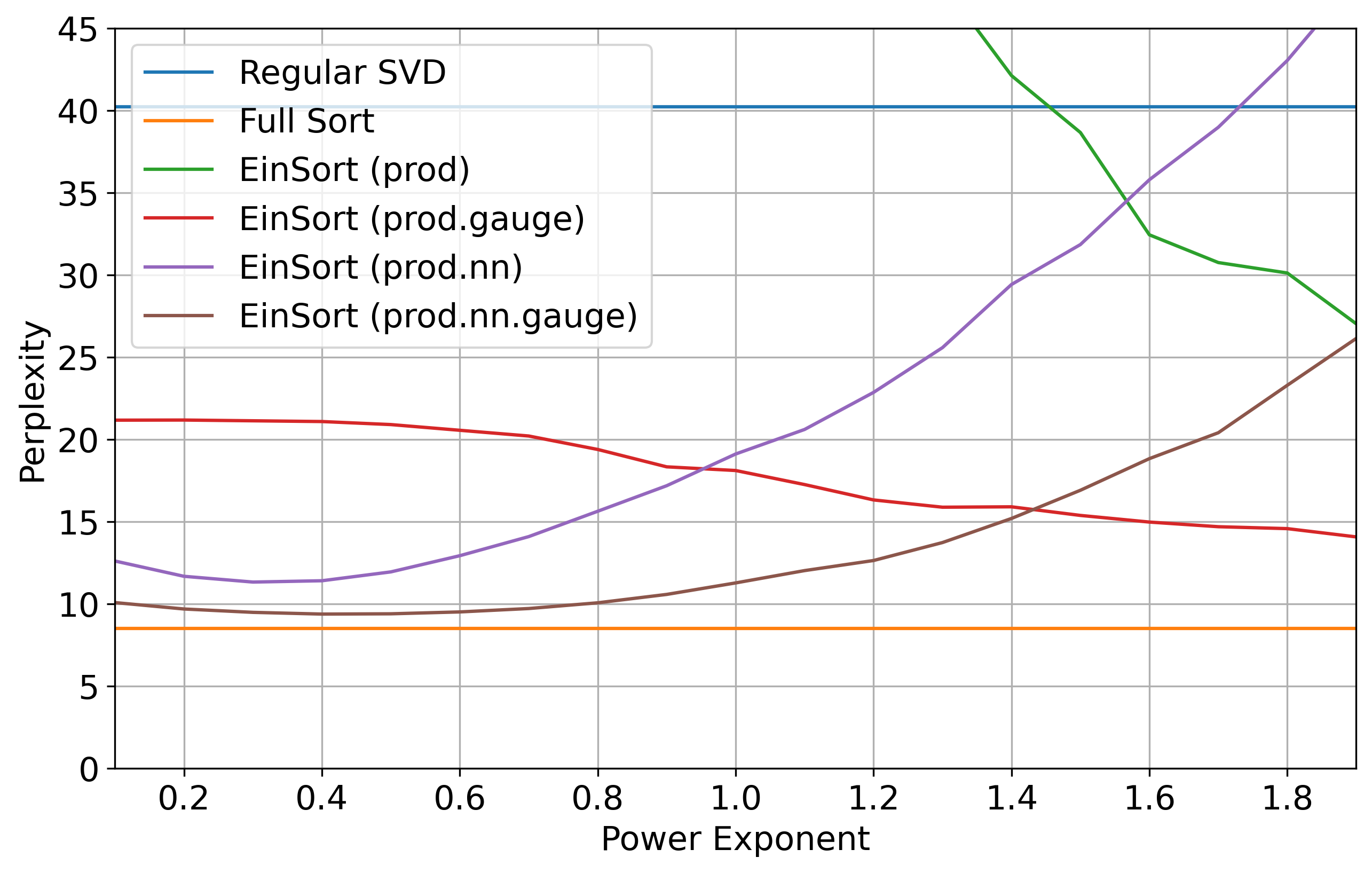}
        \caption{80\% compression}
    \end{subfigure}    
    \caption{Different reduction and nonlinear mapping for Phi3-mini at 50\% and 80\% compression.}
    \label{fig:phi-pow}
\end{figure}

\section{Gauge fixing}
\label{sec:gauge}

Here, we describe the gauge fixing we use.
Consider a 4-mode TT for a tensor $X$ decomposed by four cores $A$, $B$, $C$, and $D$ as depicted in Fig.~\ref{fig:gauge}.
As there are gauge freedom, there are no unique tensor cores to represent $X$.
Specifically, we can inject any full-rank junction matrices together with its inverse between the tensor cuts: $A$--$B$, $B$--$C$, and $C$--$D$.

One of typical methods for canonicalization is based on SVD or QR decomposition to make all cores orthogonal except for one orthogonality center.
For example, we employ an $r$-truncated SVD for a product of $AB$ as
\begin{align}
 \mathsf{svd}_{r}[AB]=U_A S_B V_B   ,
\end{align}
where $U_A$ and $V_B$ are left and right singular vectors, $S_B$ is diagonal singularvalues matrix, and $r$ is the bond rank.
We then assign as new tensor cores for $A$ and $B$ as follows:
\begin{align}
    A & \leftarrow U_A,
    \\
    B & \leftarrow S_B V_B.
\end{align}
Then, we can repeat for a new product $BC$ for $\mathsf{svd}_{r}[BC]=U_B S_C V_C$ to split as $U_B\rightarrow B$ and $S_B V_B \rightarrow C$.
Similarly, the product $CD=U_C S_D V_D$ to split $U_C\rightarrow C$ and $ V_D\rightarrow D$.
The diagonal tensor $S_D$ can be merged either into $C$ or $D$, or left as a new tensor which plays as the orthogonality center.
Doing so, all cores except the orthogonality center will be orthogonal matrices.
This gauge fixing based on SVD canonicalization is illustrated in Fig.~\ref{fig:canonical}(a).

This canonicalization can fix the gauge freedom, and improve the numerical stability to solve the tensor decomposition.
Besides the numerical stability,
\citet{koike2025quantum} discussed the potential to reduce the memory size for orthogonal matrices by parameterizing as the Stiefel manifold, either based on exponential map, Cayley transform, Householder reflections, Givens rotation, Taylor series, or Neumann series.
The Stiefel manifold parameterization can reduce the required number of parameters from $dr$ to $dr-r(r+1)/2$ for a tensor of shape $d\times r$ with $d\geq r$.

However, the regular canonicalization with Stiefel parameterization requires nonlinear mapping to construct orthogonal matrices.
\citet{koike2026latentllm} proposed a simpler method to fix the gauge freedom by using block identity form.
For example, decompose $A$ into LU factorization form as follows:
\begin{align}
    A &= P_A L_A U_A,
\end{align}
where $P_A\in\mathbb{R}^{d\times d}$ is an optional pivoting permutation, $L_A\in\mathbb{R}^{d\times r}$ is lower-triangular matrix and $U_A\in\mathbb{R}^{r\times r}$ is upper-triangular matrix.
As permutation can be represented by index transform with $\lceil \log_2(d!)\rceil$ bits for factradic coding, it does not need $d\times d$ full matrix tensor.
We split the lower triangular into upper square matrix and lower dense matrix:
\begin{align}
    L_A &=
    \begin{bmatrix}
        \bar{L}_A \\
        \underline{L}_A
    \end{bmatrix},
    \qquad
    \bar{L}_A \in \mathbb{R}^{r\times r},
    \qquad
    \underline{L}_A \in \mathbb{R}^{(d-r)\times r}.
\end{align}
Then $L_A$ can be block identity as follows:
\begin{align}
    L_A &=
    \underbrace{
    \begin{bmatrix}
        I_r \\
        \underline{L}_A \bar{L}_A^{-1}
    \end{bmatrix}
    }_{L'_A\in\mathbb{R}^{d\times r}}
    \bar{L}_A.
\end{align}
Note that the permutation $P_A$ is not necessary, while $\bar{L}_A$ can be singular without using a proper permutation.
Now, we map the permuted block-identity as a new $A$ and merge the rest to $B$ as follows:
\begin{align}
    A &\leftarrow P_A L'_A,
    \\
    B &\leftarrow \bar{L}_A U_A B.
\end{align}
Note that the block identity $L'_A$ does not need to keep the upper identity part but only the lower part of $\underline{L}_A\bar{L}^{-1}_A$, and hence the number of parameters is reduced from $dr$ to $dr - r^2$.
Repeating the LU decomposition, the cores $A$, $B$, and $C$ can become block identity. 
In consequence, the total number of parameters for TT will be reduced from $2dr(1+r)$ to $2dr(1+r) - 3r^2$.
Fig.~\ref{fig:canonical}(b) shows LU-based gauge fixing.

The pseudo code is written as follows:
\begin{lstlisting}
def gauge_fixing(A, B): # left core A and right core B with bond rank r
    d, r = A.shape # assuming d > r
    assert B.shape[0] == r # assuming B is in (r, ...)

    # LU decomposition: P @ L @ U = A
    P, L, U = torch.linalg.lu(A) 

    # permutation index from matrix
    perm = P.argmax(-1) 
    assert torch.allclose(A, (L @ U)[perm]) # A = (L @ U)[perm]

    # partition
    L0 = L[:r] # upper part, which is unitriangular
    L1 = L[r:] # lower part

    # block identity's lower part: A = L1 @ L0.inverse()
    A = torch.linalg.solve_triangular(L0, L1, upper=False, left=False, unitriangular=True)

    # New A tensor can be constructed from reduced-parameter A and perm if needed
    # A = torch.cat((torch.eye(r), A))[perm] 

    # New B tensor merging: (L0 @ U) @ B
    B = torch.einsum("ij,j...->i..."), (L0 @ U), B)
    return A, perm, B
\end{lstlisting}

Alternative to LU decomposition, QR decomposition can be used.
However, current implimentation of \texttt{torch.linalg.qr} can be numerically unstable if the first $r$ columns are not independent, which requires a pivotting like the above LU decomposition.

\begin{figure}
    \centering
    \includegraphics[width=0.4\linewidth]{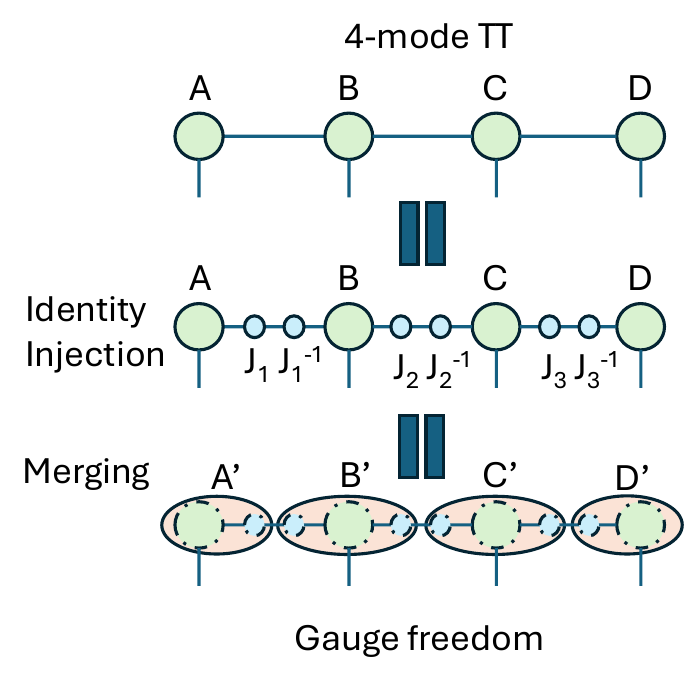}
    \caption{Gauge freedom: Tensor cores can be arbitrary up to any full-rank identity rotation injection at each cuts.}
    \label{fig:gauge}
\end{figure}

\begin{figure}
    \centering
    \begin{subfigure}{0.3\linewidth}
    \includegraphics[width=\linewidth]{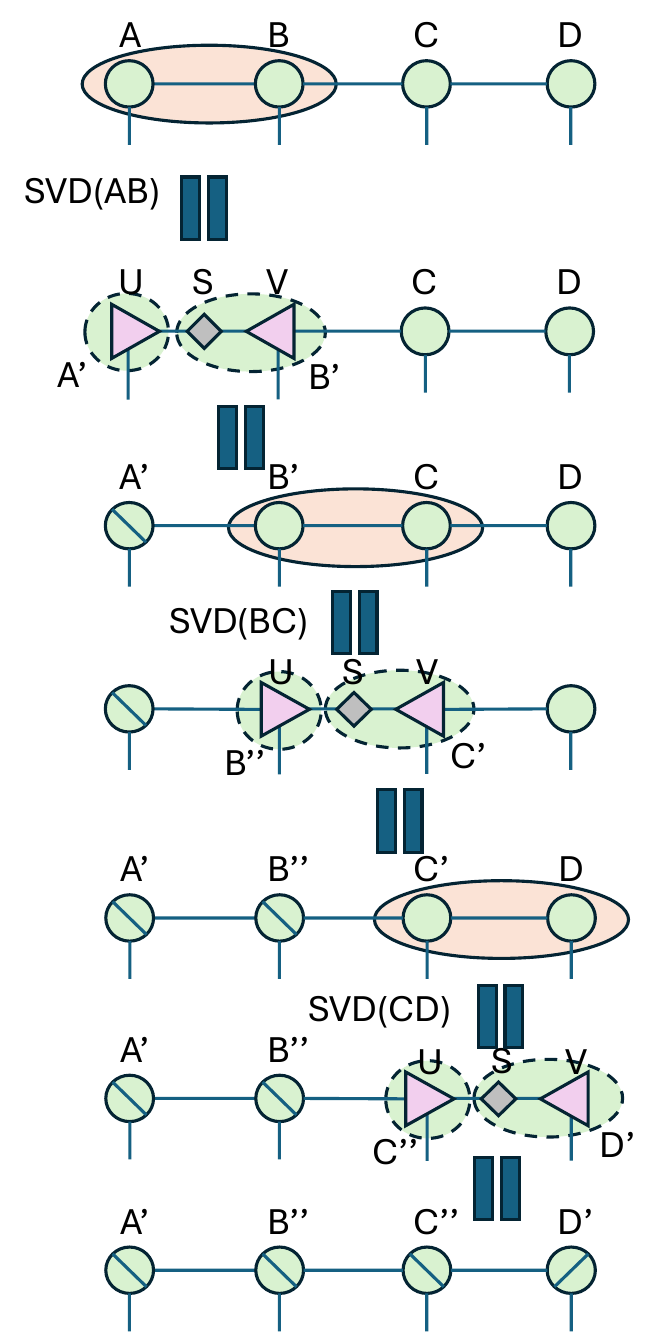}
    \caption{SVD Canonicalization}        
    \end{subfigure}
    \begin{subfigure}{0.3\linewidth}
    \hfill
    \includegraphics[width=\linewidth]{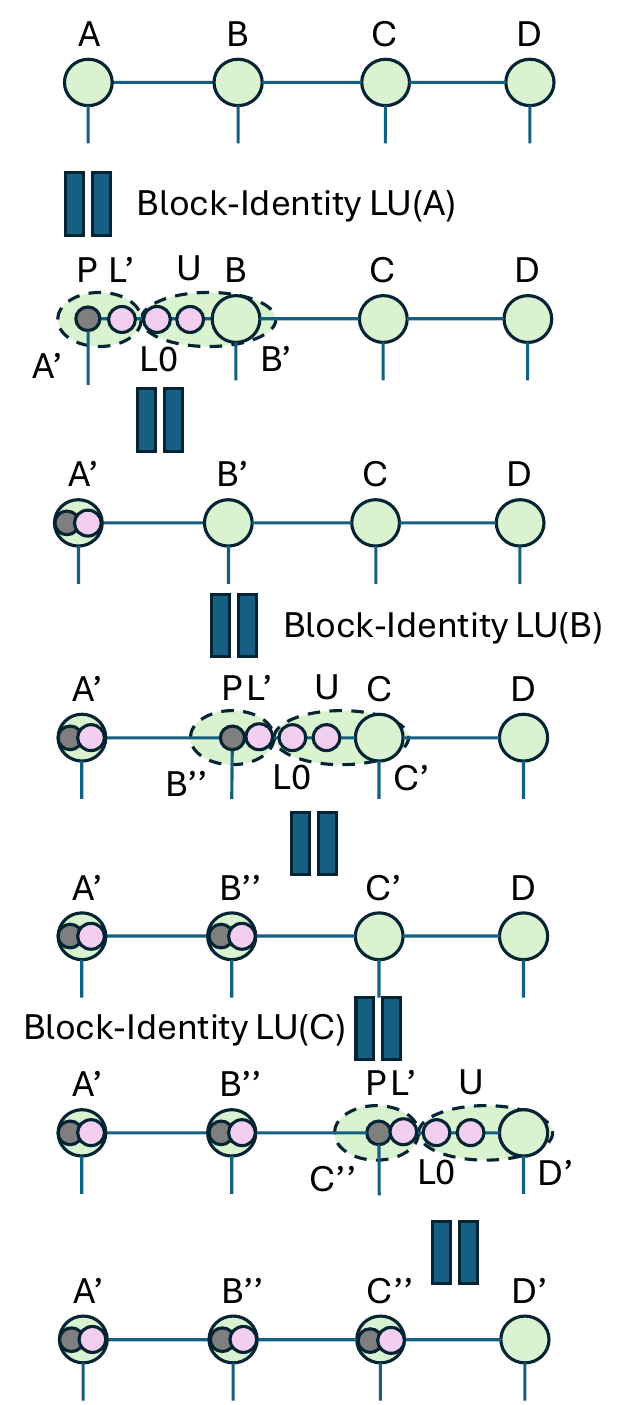}
    \caption{LU Canonicalization}        
    \end{subfigure}
    \caption{Gauge fixing with SVD and LU canonicalization. 
    SVD: Tensor cores can be left- or right-unitary around the orthogonality center. Orthogonal tensors can be represented by lower number of parameters with Stiefel manifold~\citep{koike2025quantum}.
    LU: Tensor cores can be block identity.
    Block identity tensors can be represented with lower number of parameters~\citep{koike2026latentllm}.}
    \label{fig:canonical}
\end{figure}

\section{Activation-aware tensorization}
\label{seq:ext}

We note that EinSort can be integrated with activation-aware decomposition like ASVD~\cite{yuan2023asvd}, AA-SVD~\cite{sinha2026aa} and OBD-SVD~\cite{li2026optimal}.
The core idea of those methods is to use input and/or output preconditioning. 
ASVD uses the following objective function:
\begin{align}
    \hat{W} &= 
    \arg\min_{W'}
    \| W X - W' X \|^2
    +
    \lambda \|W - W'\|^2,
\end{align}
where $X$ is an input activation from calibration data and $\lambda$ is a regularizer to consider activation-unaware loss.
The solution is given as
\begin{align}
\hat{W} &=\mathsf{svd}_r[W C^{1/2}]C^{-1/2},
\end{align}
where $C^{1/2}$ is a preconditioner, defines as $C=XX^\top + \lambda I$.
AA-SVD extends it to deal with error propagation as follows:
\begin{align}
    \hat{W} &= 
    \arg\min_{W'}
    \alpha \| W X - W' X' \|^2
    +
    \beta \| W X - W' X \|^2
    +
    \lambda \|W - W'\|^2,
\end{align}
where $X'$ is a modified input activation after decomposition of preceding layers, $\alpha$ and $\beta$ are regularizers.
The solution is given as
\begin{align}
    \hat{W} &=
    \mathsf{svd}_r[
    W C_1 C_2^{-1/2}
    ]
    C_2^{-1/2},
    \\
    C_1 &= \alpha XX'^\top 
    + \beta XX^\top + \lambda I,
    \\
    C_2 &= \alpha X'X'^\top 
    + \beta XX^\top + \lambda I.
\end{align}
The OBD-SVD extends ASVD to use output preconditioning to minimize:
\begin{align}
    \hat{W} &=\arg\min_{W'}
    \|G^{1/2} (W-W') C^{1/2} \|^2,
\end{align}
where $G$ represents the gradient information to approximate Hessian.
The solution is given in the form:
\begin{align}
    \hat{W} &= 
    G^{-1/2}
    \mathsf{svd}_r[G^{1/2}W C^{1/2}] C^{-1/2}.
\end{align}
All cases including ASVD, AA-ASVD, and SVD-SVD can be expressed in the form with particular preconditioner $G$, $C_A$, and $C_B$:
\begin{align}
    \hat{W} &=
    G^{-1/2}
    \mathsf{svd}_r[G^{1/2}W C_1 C_2^{-1/2}] C_2^{-1/2},
\end{align}

Naturally, sorted tensor decomposition can do the same way to deal with the activation and gradient statistics as follows: 
\begin{align}
    \hat{W} &=
    G^{-1/2}
    \mathcal{T}_\theta^\pi[G^{1/2}W C_1 C_2^{-1/2}] C_2^{-1/2}.
\end{align}

\section{Quantization-aware tensorization}
\label{sec:quant}

The tensor cores can be also quantized or pruned to further save the memory for reduced-rank tensor networks.
As einsum is differentiable, imposing quantization and pruning under gradient optimization is straightforward, e.g., based on straight-through estimation.
Specifically, the tensor cores are quantized (or pruned) as $\hat{W}=\mathcal{Q}[W]$ where $\mathcal{Q}[\cdot]$ denotes the quantization operation (and/or pruning), which itself may be non differentiable.
The quantized tensors are then pass through to the einsum contraction by superposing the non-quantized cores: $W' = \hat{W} + (W - W.\mathrm{detach})$ where $[\cdot].\mathrm{detach}$ denotes cutting the autograd path.
The pseudo code is listed below.
\begin{lstlisting}
def quantize_contract(self, target, *args, **kwargs):
    # quantize cores
    cores = self.quantize(self.cores, **kwargs)
    
    # straight-through estimator
    for k in range(len(cores)):
        cores[k] += self.cores[k] - self.cores[k].detach()
    
    # quantization-aware contraction to use gradient optimization
    return self.expression(*cores, *args, **kwargs)
\end{lstlisting}

\section{Test-time adaptation}
\label{sec:adapt}

Tensor ranks can be optimized across layers and modules~\citep{luo2024adaptive}.
However, most papers consider static rank selections, which do not adaptively change over the inference time.
Inspired by test-time quantization (TTQ)~\citep{koike2026ttq}, we can consider test-time rank adaptation.
TTQ provides a theoretical justification suggesting that test-time weight approximation adaptive to every input prompts can improve the accuracy, compared to offline weight static approximation.

Consider a simple toy example: the weights $W$ are decomposed into 2-rank factors as $\hat{W}= a_1 b_1^\top + a_2 b_2^\top$ for vectors $a_1, a_2, b_1, b_2$, and the input token $X$ is orthogonal to both $b_1$ and $b_2$. 
Then, $\hat{W}X=0$, which implies that those two-rank factors are useless and we should increase the rank for such tokens. 
Whereas, if $X$ is orthogonal to $b_2$ but $b_1$ like $X=b_1 c^\top$, then we have $\hat{W}X=|b_1|^2 a_1 c^\top$, suggesting that the second factor is redundant to be omitted.
This toy example gives a key insight that the weight decomposition should be adaptive to the input tokens at inference time to be more efficient and effective.

We may use a simple scoring network to determine the bond ranks at test time.
Specifically, a small neural network feeding online activation input $X$, and decide the ranks for tensor networks.
The network can be trained through synthetic data or calibration data held out from test samples.
This enables a test-time rank adaptation for tensor networks.

Another alternative is to feed an updated preconditioner at test time, and update the tensor cores in an online fashion.
For example, $W$ of shape $(d^2, d^2)$ uses 4-mode TT of shape $(d,d,d,d)$ with a moderately high rank of $R$, and the online token correlation $C=XX^\top+\lambda I_{d^2}$ is decomposed as another tensor network with 2-mode matrix product operator (MPO) of shape $(d, d)$ with bond rank $R'$ through density matrix renormalization group (DMRG) algorithm~\cite{schollwock2011density}.
Then, we can locally and sequentially update the tensor cores of the 4-mode TT by contracting with MPO to find the dominant eigenspaces with limited bond ranks lower than $R$.
Fig.~\ref{fig:testtime} illustrates the test-time adaptation framework to update the TT depending on the online preconditioner $C$.

\begin{figure}[t]
    \centering
    \includegraphics[width=0.7\linewidth]{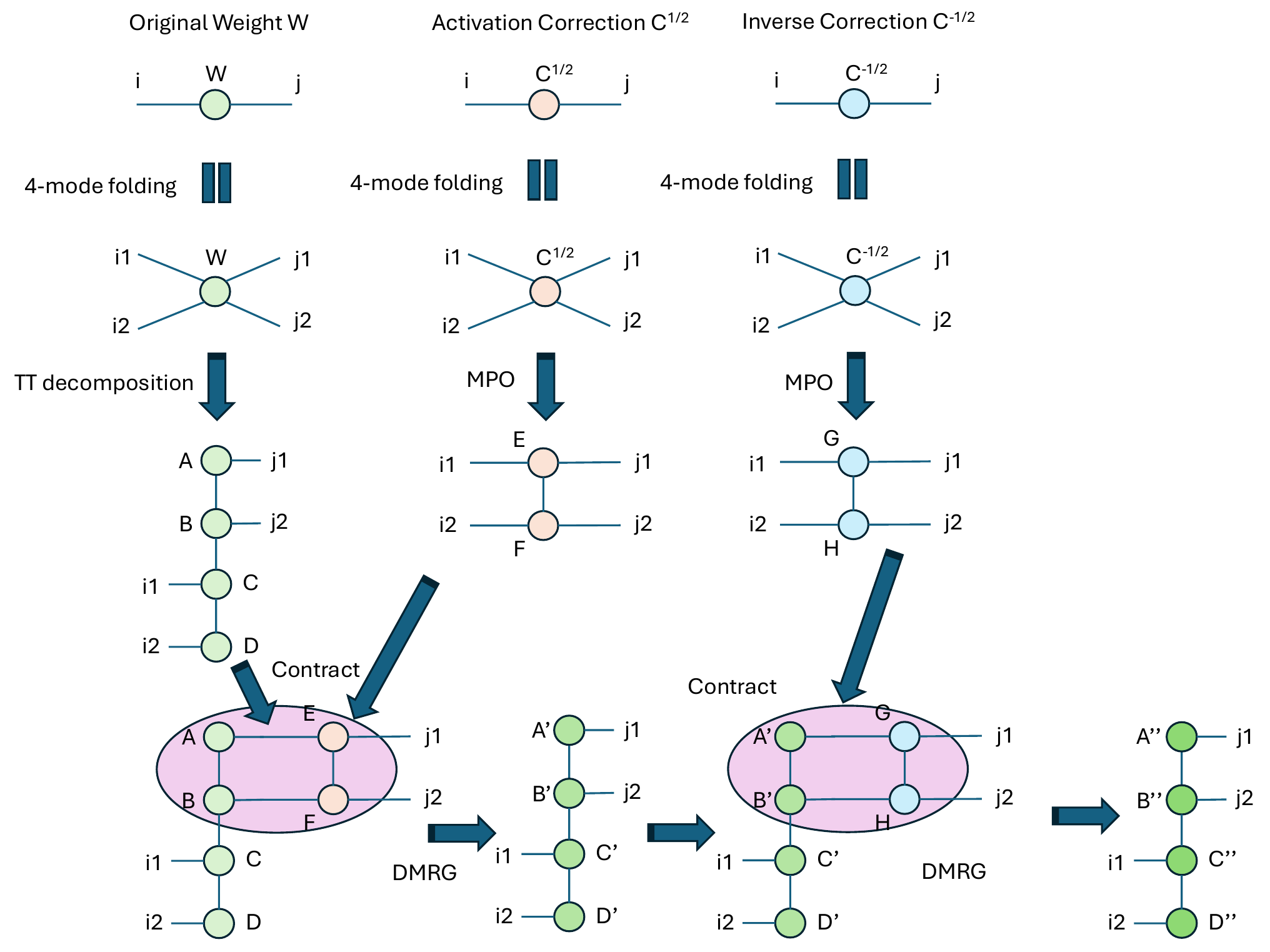}
    \caption{Test-time tensor train adaptation based on online activation preconditioner $C$.}
    \label{fig:testtime}
\end{figure}

\section{Activation compression}
\label{sec:act}

Besides compressing LLM weights or KV cache, we can also compress intermediate activation information, which are required for gradient updates.
For example, consider a linear module taking a forward path: $Y=WX$.
Then the weight update from gradient backpropagation requires:
\begin{align}
    W &\leftarrow
    W - \eta \underbrace{\frac{\partial \mathcal{L}}{\partial Y}}_{G} X^\top,
\end{align}
where $\eta$ is a learning rate, $\mathcal{L}$ is a loss function and $G$ is gradient up to the output $Y$.
This clearly shows the potential memory issue for LLM gradient updates, which require all intermediate activations $X$ across layers.
EinSort can reduce the memory for any tensors including such intermediate activations.

\section{Connection with quantum gates}
\label{sec:quantum}

Tensor index permutations in EinSort are closely related to entangling operations in quantum tensor networks. 
In particular, reversible Boolean permutations can be interpreted as classical analogues of quantum logic gates including controlled-NOT (CNOT) and Toffoli gates. 
For binary indices as in example of (\ref{eq:2x2}), the transformation $(i,j)\mapsto (i,i\oplus j)$ corresponds exactly to the action of a CNOT gate, where one index acts as a control bit and the other is conditionally flipped. 
Indeed, CNOT corresponds to a permutation matrix:
\begin{align}
    \mathsf{CNOT} &=
    \begin{bmatrix}
        1 & 0 & 0 & 0 \\
        0 & 1 & 0 & 0 \\
        0 & 0 & 0 & 1 \\
        0 & 0 & 1 & 0
    \end{bmatrix}.
\end{align}

As an another example, consider a $4\times 4$ matrix $W_0$ to apply row-wise sorting:
\begin{align}
    W_0 =
    \begin{tikzpicture}[baseline=(m.center),remember picture]
    \node (m) at (0,0) {
    $
    \begin{bmatrix}
        4.0 & 1.0 & 2.0 & 3.0 \\
        2.1 & 3.1 & 4.1 & 1.1 \\
        1.2 & 3.2 & 4.2 & 2.2 \\
        3.3 & 2.3 & 1.3 & 4.3 
    \end{bmatrix}
    $
    };
    \draw[->, blue] (-1.7,0.9) -- (-1.7,-0.9) node[midway,left] {$i$};
    \draw[->, blue] (-1.5,1.0) -- (1.5,1.0) node[midway,above] {$j$};
    \end{tikzpicture}
    ,
    \qquad
    \pi &=
    \begin{tikzpicture}[baseline=(m.center),remember picture]
    \node (m) at (0,0) {
    $
    \begin{bmatrix}
        0 & 3 & 2 & 1\\
        2 & 1 & 0 & 3 \\
        2 & 1 & 3 & 0 \\
        3 & 0 & 1 & 2\\
    \end{bmatrix}
    $
    };
    \draw[->, blue] (-1.1,0.9) -- (-1.1,-0.9) node[midway,left] {$i$};
    \draw[->, blue] (-0.9,1.0) -- (0.9,1.0) node[midway,above] {$j_1\oplus j_0\oplus i_0, j_0\oplus i_1$};
    \end{tikzpicture}
    ,
    \qquad
    W_\pi =
    \begin{bmatrix}
        4.0 & 3.0 & 2.0 & 1.0 \\
        4.1 & 3.1 & 2.1 & 1.1 \\
        4.2 & 3.2 & 2.2 & 1.2 \\
        4.3 & 3.3 & 2.3 & 1.3 
    \end{bmatrix}
    ,
    \label{eq:4x4}
\end{align}
where the row index $i\in\mathbb{Z}_4$ can be represented with two binary indices $i_1\in\mathbb{Z}_2$ and $i_0\in\mathbb{Z}_2$ as $i=2i_1 + i_0$.
Similarly the column index $j\in\mathbb{Z}_4$ has 2 binary indices as $j_1\in\mathbb{Z}_2$ and $j_0\in\mathbb{Z}_2$ such that $j=2j_1 + j_0$.
Then the row-wise permutation $\pi$ can be expressed as: 
$j_1\leftarrow j_1+j_0+i_0\pmod 2$, $j_0\leftarrow j_0+i_1\pmod 2$.
The column index $j$ is now dependent on row index $i$, which is the important factor to reduce the rank as discussed in Appendix~\ref{sec:theory}.
This is entangling operator realized by CNOT chain as depicted in Fig.~\ref{fig:cnot}.
The singular values for the unsorted matrix $W_0$ is about $[10.6, 4.1, 1.6, 0.6]$ and that for $W_\pi$ is about $[11.5, 0.2, 0.0, 0.0]$, which is nearly a rank of 2.
And, hence the sorted matrix can be well approximated by lower-rank tensor decomposition.
Note that the CNOT chain for permutation is reversible, and we can use conjugate CNOT chain to contract with an activation $X$ as in this figure.

\begin{figure}[t]
    \centering
    \includegraphics[width=0.9\linewidth]{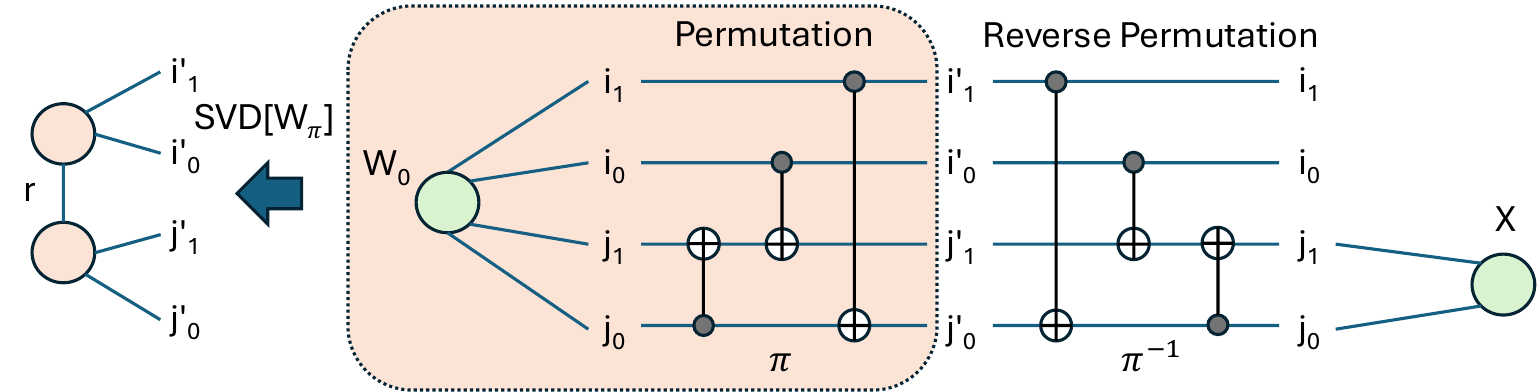}
    \caption{CNOT chain to realize the row-wise sorting $\pi$ for $W_0$ in (\ref{eq:4x4}).
    The sorted matrix $W_\pi$ can then be decomposed in low-rank tensors.
    The input $X$ is contracted with the decomposed $W_\pi$ through reverse permutation $\pi^{-1}$ with conjugate CNOT chain.}
    \label{fig:cnot}
\end{figure}

However, CNOT plus NOT gates can only realize even affine permutation.
More generally, arbitrary reversible permutation can be realized by CNOT plus Toffoli gates, e.g., $(i,j,k)\mapsto (i,j,k\oplus ij)$, where two control bits jointly determine the target transformation. 
From the tensor-network perspective, these reversible index transformations may convert highly entangled tensors into representations with lower effective bond dimensions, thereby exposing latent low-rank structure prior to tensor decomposition. 
This connection suggests broader links between tensor compression, reversible computation, and quantum-inspired representations, where index permutations play a role to basis transformations that simplify entanglement patterns.

Motivated by the interpretation of permutations as compositions of Toffoli-like reversible gates, we can further reduce the memory footprint required to represent sorting permutations. 
Specifically, after obtaining a target permutation vector, we convert the permutation into a sequence of Toffoli gate chains, which induces an equivalent Boolean function representation. 
To reduce the Toffoli gates, we can 
Using ancilla bits can further reduce the reversible circuits via Pebble game theory~\cite{li1998reversible}.
The resulting Boolean function can then be compressed into a compact sum-of-products form using standard logic synthesis techniques such as the Quine–McCluskey algorithm~\citep{mccluskey1956minimization} or Espresso~\citep{brayton1984logic}. 
Using PyEDA\footnote{\url{https://github.com/cjdrake/pyeda}}, we empirically confirmed that the number of permutation gates can be substantially reduced, particularly when introducing don't-care sets corresponding to low-score permutations that can be ignored with minimal impact on sorting quality. For example, the total number of gates was reduced from $1976$ to $744$ when allowing approximately $50\%$ permutation cancellation through don't-care optimization.
We also use RevKit\footnote{\url{https://github.com/msoeken/revkit}} to reduce the reversible gates for permutations.

\section{Algebraic polynomial permutation}
\label{sec:poly}

CNOT gates work binary axis, while it can be extended to arbitrary dimension by modular additions like: $(i,j) \mapsto (i,j+i\mod d_j)$ for the second dimension of $d_j$.
A related method to parameterize permutation is to use algebraic polynomial permutation.
For example, quadratic polynomial permutation (QPP)~\cite{takeshita2007permutation} uses the following conversion:
\begin{align}
    \pi(i) & = a + b\, i + c\,i^2 \mod n,
\end{align}
where $i\in\mathbb{Z}_n$ is an index of size $n$.
The bijective conditions are $\mathrm{gcd}(b,n)=1$ and $\mathrm{rad}(n)\mid c$ where $\mathrm{gcd}$ denotes the greatest common multiple and $\mathrm{rad}$ denotes the radical.
We may optimize the coefficients of polynomials to improve the tensorization accuracy.
However, polynomial over a finite ring does not cover all permutations unless $n$ is a prime number.
For arbitrary $n$, algebraic polynomial over a finite field such as Galois field can realize any reversible permutation instead.
We may use Lagrange interpolation over Galois field as $\pi(x)=\sum_i y_i \prod _{j\neq i} (x-x_j)/(x_i-x_j)$ to pass through major sorted index conversion $x_i\mapsto y_i$.
To design polynomial permutation on Galios field, we use \texttt{galios.laglange\_poly}\footnote{\url{https://github.com/mhostetter/galois}}.

\section{Learned permutation operations}
\label{sec:learn}

In this paper, we focus on practical permutations based on sliced sorting with nonlinear mapping and reduction variety.
Even though sorting has a strong theoretical justification as discussed in Appendix~\ref{sec:theory}, it is just one of heuristics we discovered.
For more general cases, the permutations and nonlinear mappings can be learned to optimize them.
For example, permutation operations can be optimized through gradient updates, e.g., using NeuralSort~\citep{grover2019stochastic}, SortSoft~\cite{prillo2020softsort}, or Gumbel--Sinkhorn~\citep{mena2018learning}.
Our objective could be formulated as follows:
\begin{align}
    \alpha \|W-\mathcal{T}_\theta^\pi[W]\|^2 
    +
    \lambda |[\theta,\pi]|,
\end{align}
where $|[\theta,\pi]|$ denotes the total memory cost for tensorization hyperparameters $\theta$ and permutation $\pi$.
Note that the first term can be extended to deal with input/output preconditioners as discussed in Appendix~\ref{sec:act}.

We can also include nonlinear mapping in $\theta$ not only for sorting but also tensorization.
Specifically, we apply nonlinear mapping to $W$ like exponentiation, and after decomposition, we use inverse nonlinear mapping to reconstruct.
More specifically, we can write as
\begin{align}
    \mathcal{T}_\theta^\pi[W]
    &=
    \phi^{-1}[\pi^{-1}[\mathcal{T}_\theta [\pi[\phi[W]]]],
\end{align}
where $\phi[\cdot]$ denotes reversible nonlinear mappings such as exponentiation, and $\phi[\cdot]^{-1}$ is its inverse.
Note that some function like absolute operation $|\cdot|$ can be reversible if we retain the sign information along with the absolute value, which we call non-negative tensorization.
When some hyperparameters are not differentiable, we could employ reinforcement learning to design permutations and nonlinear mappings as well as tensor ranks, tensor shapes, and topologies.

\section{Folding and ordering}
\label{sec:fold}

Consider a mode-4 tensor $X$ of shape $[d_1, d_2, d_3, d_4]$, e.g., the number of layers $d_1$, binary axis indicating key or value for $d_2$, number of heads for $d_3$, and the head dimension for $d_4$.
There are a few steps to use different tensor ordering, folding, and unfolding.
We can decompose it into many different ways.
For example, we may use:
\begin{itemize}
    \item 4-mode TT for $X$ as is;
    \item 4-mode TT for axis-swapped version of $X$, like \texttt{X.permute(3,1,0,2)};
    \item Use the second axis as a batch dimension, and apply 3-mode TT for the rest 3 axis;
    \item Unfold into $[d_1 \times d_2, d_3, d_4]$, and apply 3-mode TT: \texttt{X.reshape(-1, d3, d4)}; 
    \item Fold $X$ into 6-mode tensor like \texttt{X.reshape(a,b,c,d,e,f)}, and apply 5-mode TT for the last 5 axis.
    
\end{itemize}
As the tensor decomposition depends on the mode ordering, axis swapping and its folding shape can be adjusted so that the reconstruction is minimized.
Similarly, before decomposition, we use sorting operations and we have flexibility to choose which axis to sort.
Note that sorting axis and decomposition axis need not be identical or perfectly overlapped.
In addition, we have additional degrees of freedom to choose the axis for slice reduction.
Fig.~\ref{fig:fold} shows the EinSort framework using different folding/unfolding and axis ordering for sorting, reduction, and decomposition.
We also use different nonlinear mappings for sorting $\phi$ and decomposition $\varphi$, in this figure.

In addition, the sorting axis can be sequential as discussed in Appendix~\ref{sec:sorting}, e.g., axis $[3,5]$ to sort first, and then $[2,4]$ to sort later, rather than sorting all $[2,3,4,5]$ axis.
If the size of axis is all $d$, the whole sorting for the last 4 modes require $\lceil \log_2(d^4!)\rceil/d^4$ bits, whereas the sequential sorting requires $2\lceil \log_2(d^2!) \rceil/d^2$ bits.
For example with $d=32$, it will reduce from $18.6$ to $17.1$ bits.
When doing 4 times in each axis sequentially, then it will be reduced to $14.7$ bits.

\begin{figure}[t]
    \centering
    \includegraphics[width=\linewidth]{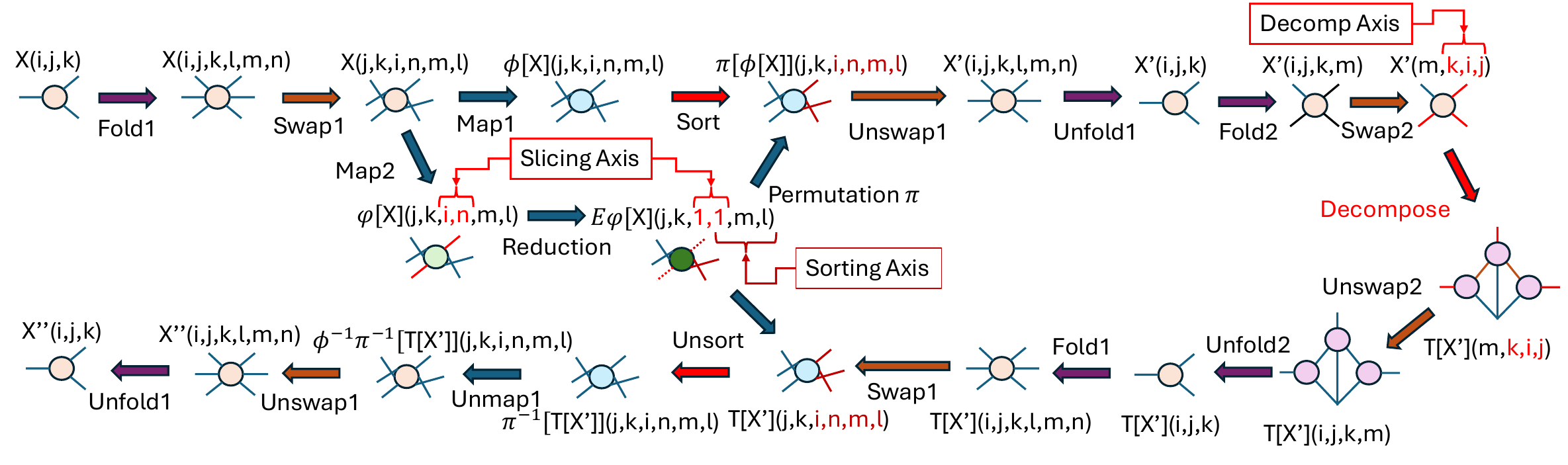}
    \caption{EinSort pipeline with generalized folding/unfolding, ordering, slicing, sorting, and mapping.}
    \label{fig:fold}
\end{figure}

\section{Benefits and limitations}
\label{sec:limit}

EinSort has some benefits:
\begin{itemize}
    \item There is a theoretical foundation suggesting that sorting operation has a potential benefit to reduce tensor ranks.
    \item Sorting operation has a good connection with quantum tensor networks, using entanglement operations.
    \item Einsum-based formulation gives a powerful tool to design almost arbitrary tensor network topology.
    \item EinSort provides a wide range of flexibility to adjust the hyperparameters, including slicing order, mode selection, folding combination, nonlinear mapping, and reduction operations.
    \item It can be used for any tensor compression: e.g., LLM weights decomposition; KV cache reduction; forward activation compression used for backward gradient calculation.
\end{itemize}

Nonetheless, it has some drawbacks and limitations.
\begin{itemize}
    \item Finding best tensor topology, folding shapes, index ordering, etc. is not straightforward.
    \item Computing large tensor factorization is computationally challenging.
    \item Permutation memory is not easy to reduce.
    \item Compatible high-performance CUDA kernels which can optimize contraction and index reordering are not available.
\end{itemize}

\section{Experiments setup}
\label{sec:setup}

We conduct experiments for LLM benchmarks to evaluate the effectiveness of our method.
Our experiments are based on the same setting of SparseLLM~\cite{bai2024sparsellm} and their code base\footnote{\url{https://github.com/BaiTheBest/SparseLLM}}.
Following existing work~\cite{sun2023simple}, we decompose all linear layers in LLM transformers.
We implemented EinSort in PyTorch~\cite{paszke2019pytorch} and used the HuggingFace Transformers library \cite{wolf2019huggingface} for handling models and datasets. 
All experiments are conducted on NVIDIA A40 or A100 GPUs. 

For LLM experiments, we consider the Qwen3~\cite{yang2025qwen3}, Gemma3~\cite{team2025gemma}, and Phi3~\cite{haider2024phi} models.
We show results on different sizes of models to provide a broader picture for the performance of EinSort. 
We measure perplexity score for one of the most widely used benchmarks: raw-WikiText2 (WT2)~\cite{merity2016pointer}.
Details of LLMs and datasets we used are found in Appendices~\ref{sec:model} and \ref{sec:dataset}.

We use the head axis for reduction and $64$-stacked head dimension for sorting, and thus the sliced sorting requires $\lceil \log_2((64\times d)!)\rceil/(64\times d\times h)$ bits, where $d$ and $h$ are head dimension and the number of heads, respectively.
For Qwen3-0.6B/1.7B, the sliced sorting requires $\lceil \log_2((64\times 128)!)\rceil/(64\times 128\times 8)\simeq 1.44$ bits per KV cache parameter.
For Gemma3-4B, the sorting requires $\lceil \log_2((64\times 256)!)\rceil/(64\times 256\times 4)\simeq 3.14$ bits per cache parameter.
For Phi3-mini, the permutation requires
$\lceil \log_2((64\times 96)!)\rceil/(64\times 96\times 32)\simeq 0.35$ bits per parameter.
Note that we can freely adjust the slicing operation to control the required memory.

\section{LLM benchmark}
\label{sec:bench}

\subsection{LLM model variants}
Besides Qwen3-0.6B and Gemme3-4B models in Fig.~\ref{fig:rate2}, we add WT2 perplexity evaluations for various LLM models including Qwen3-1.7B and Phi-3.5-mini.
Fig.~\ref{fig:qwen2b-rate} shows the PPL for Qwen3-1.7B model when KV cache is compressed by regular SVD, full-sort SVD, and EinSort.
The trend is similar to Qwen3-0.6B model in Fig.~\ref{fig:wt2}.

Fig.~\ref{fig:phi3-rate} shows the PPL for Phi3-mini model when KV cache is compressed by regular SVD, full-sort SVD, and EinSort.
We observe that EinSort can significantly improve the tensorization accuracy towards full-sort case (which requires $\lceil \log_2((64\times 96)!)\rceil/(64\times 96) \simeq 11.1$ bits for permutation memory) while the permutation memory is kept small as $0.35$~bits.

\begin{figure}
    \centering
    \begin{subfigure}{0.45\linewidth}
    \includegraphics[width=\linewidth]{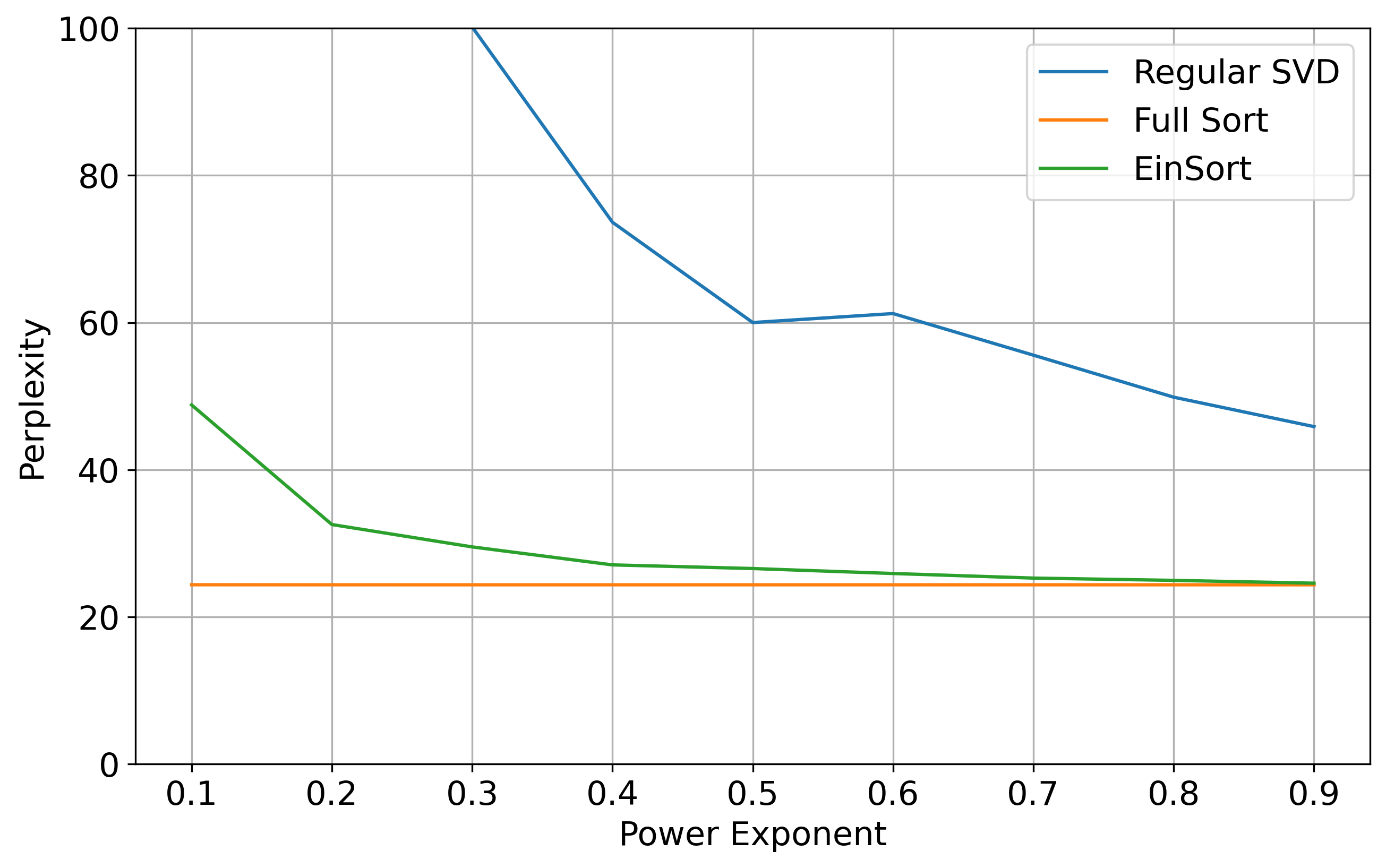}
    \caption{Qwen3-1.7B model.}
    \label{fig:qwen2b-rate}
    \end{subfigure}
    \begin{subfigure}{0.45\linewidth}    \includegraphics[width=\linewidth]{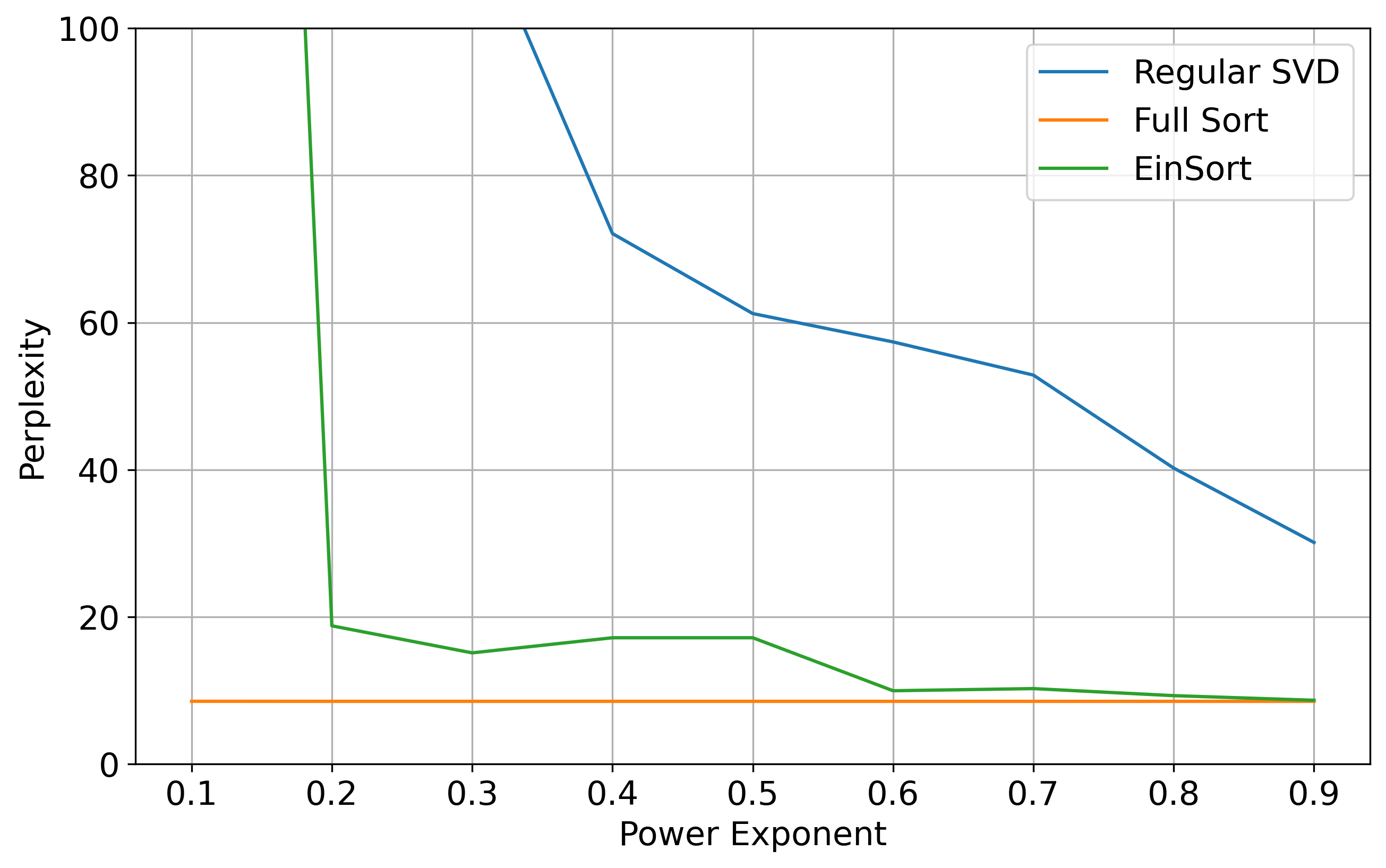}
    \caption{Phi3-mini (3.8B) model.}
    \label{fig:phi3-rate}
    \end{subfigure}
    \caption{PPL over compression rate for Qwen3-1.7B and Phi3-mini (3.8B) models.}
    \label{fig:rate2}
\end{figure}

\subsection{LLM math reasoning}
Besides perplexity evaluations, we add more practical LLM benchmark, specifically GSM8K~\cite{cobbe2021training} for mathematical reasoning.
Fig.~\ref{fig:gsm8k} shows the accuracy for Phi-4-mini when we employ KV cache compression.
As well as the average accuracy, we plot a shaded band to represent the Wilson confidence interval at $99\%$.
Similar to WT2 perplexity results, we observe that EinSort can significantly improve the accuracy over the regular SVD compression.

\begin{figure}[t]
    \centering
    \includegraphics[width=0.7\linewidth]{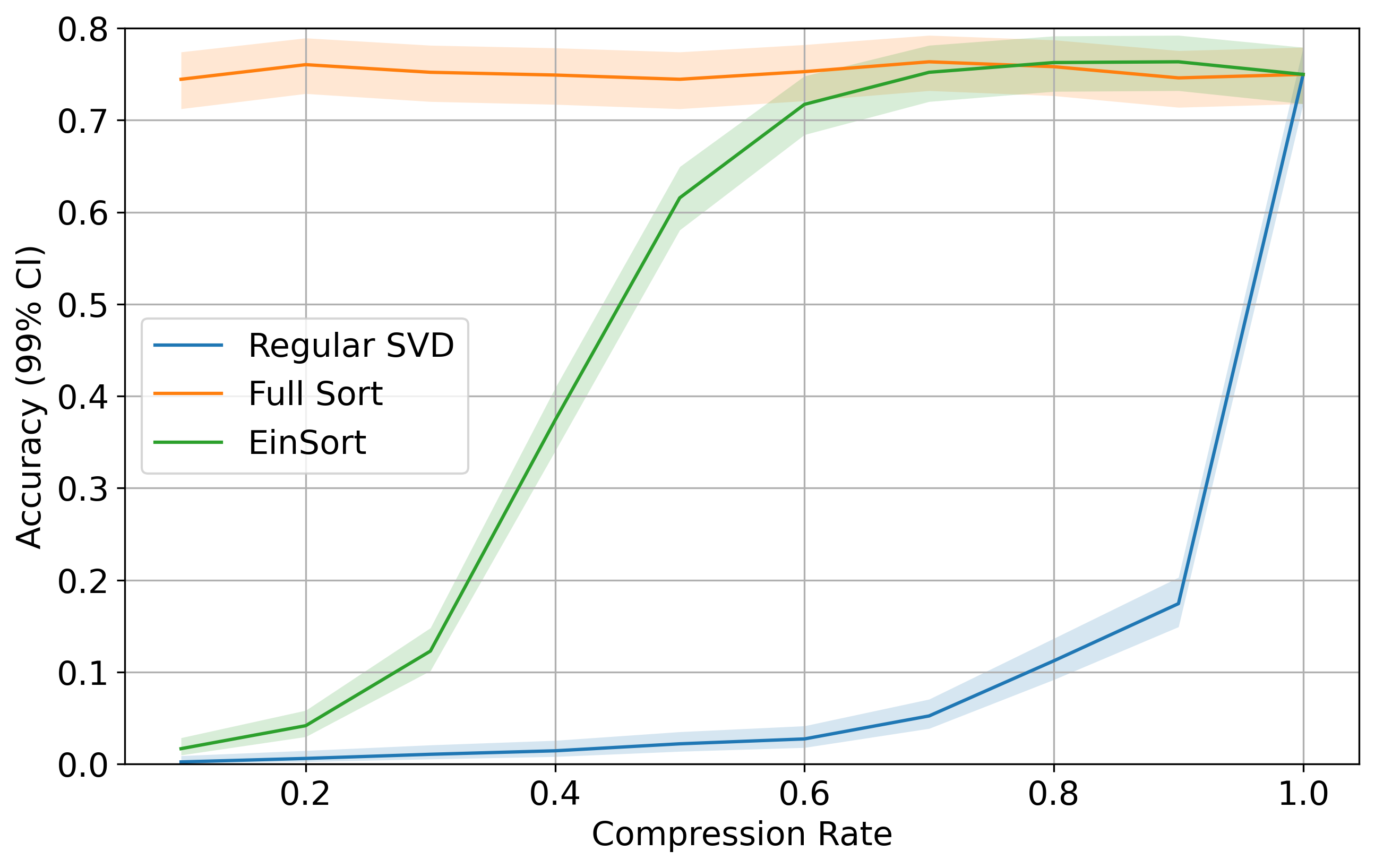}
    \caption{GSM8K accuracy for Phi-4-mini as a function of KV cache compression ratio.
    Shaded zone is Wilson's $99\%$ confidence interval.}
    \label{fig:gsm8k}
\end{figure}

\subsection{VLM visual reasoning}

We next consider practical applications of VLMs on visual reasoning tasks.
Specifically, we use Qwen3-VL~\cite{bai2025qwen3} models on TextVQA~\cite{singh2019towards} benchmark.
Rather than KV cache reduction, we compress weights of linear modules for the LLM backbone in VLM models.
We compare plain SVD, ASVD~\cite{yuan2023asvd}, SVD-LLM~\cite{wang2024svd}, and LatentLLM~\cite{koike2026latentllm} as baselines.
All methods except plain SVD use activation aware preconditioning based on 64 calibration samples in TextVQA train splits.
Fig.~\ref{fig:tvqa} shows performance for Qwen3-VL-4B and 8B models. 
We confirm that the proposed EinSort can keep high accuracy even at 20\% compression, whereas the other baselines show steep degradation when we compress weights.

\begin{figure}
    \centering
    \begin{subfigure}{0.49\linewidth}
        \includegraphics[width=\linewidth]{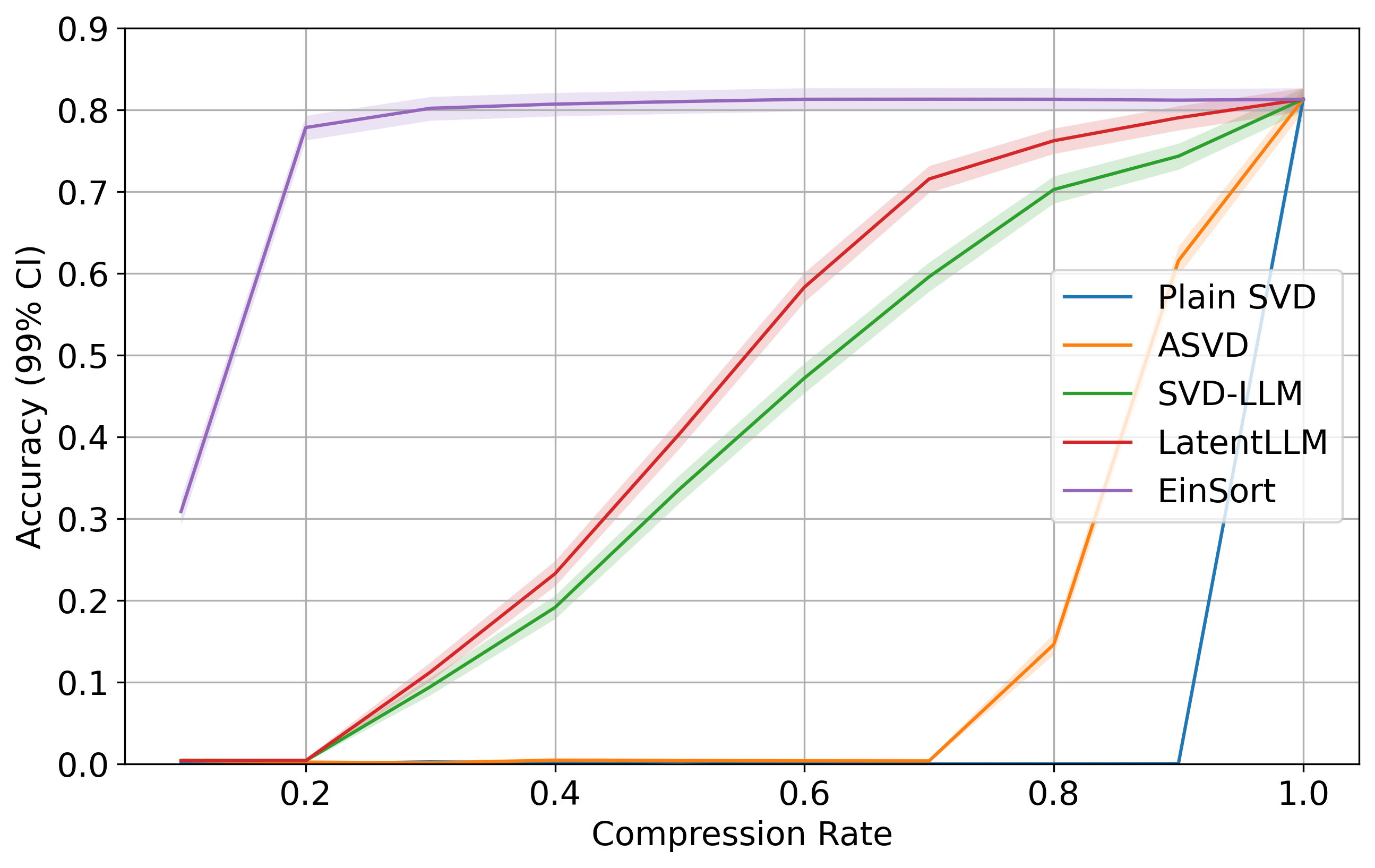}
        \caption{Qwen3-VL-4B-Instruct}
    \end{subfigure}
    \hfill
    \begin{subfigure}{0.49\linewidth}
        \includegraphics[width=\linewidth]{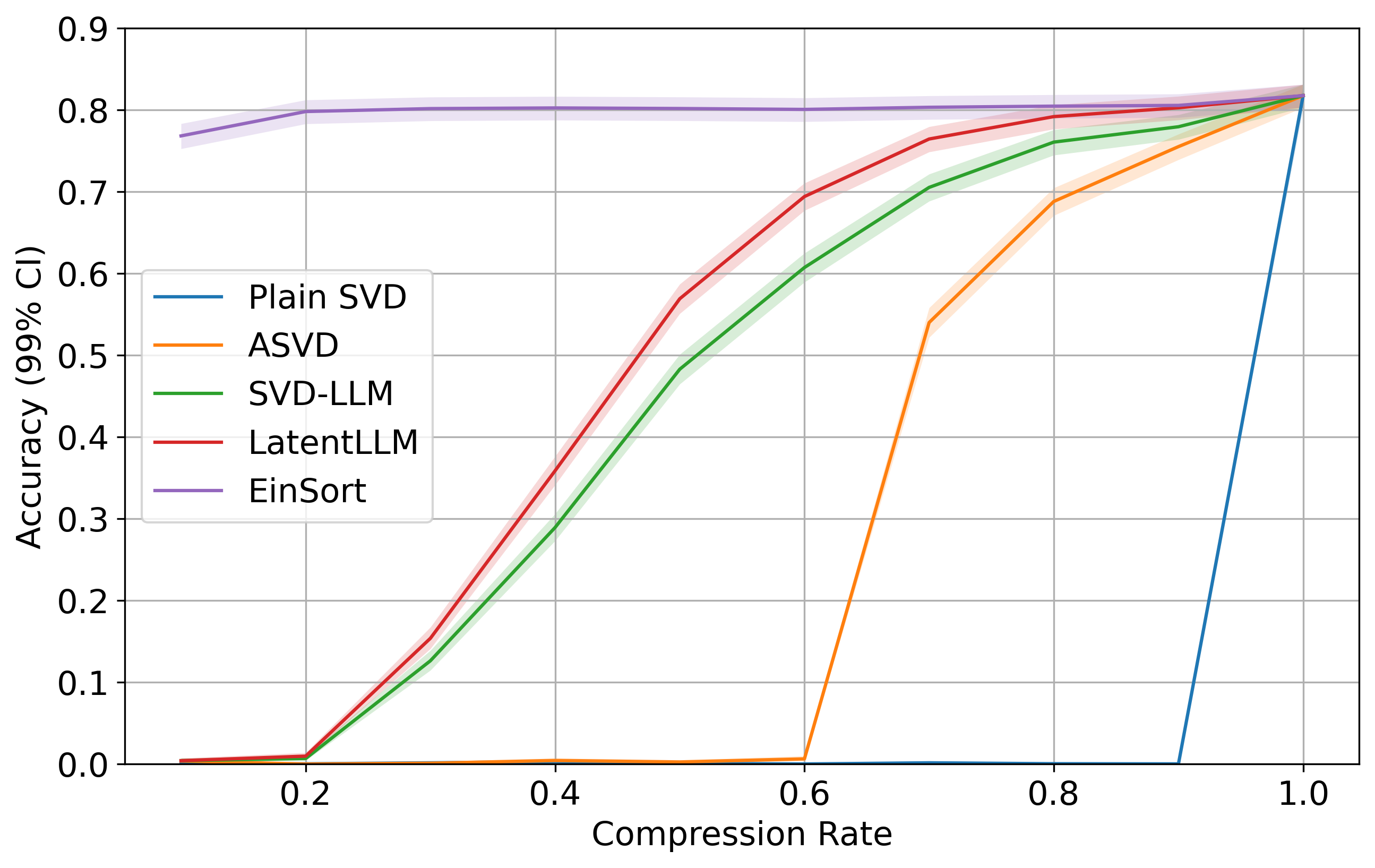}
        \caption{Qwen3-VL-8B-Instruct}
    \end{subfigure}
    \caption{Accuracy for compressed Qwen3-VL models on TextVQA benchmark.}
    \label{fig:tvqa}
\end{figure}

\subsection{VLA robot manipulation}

We further consider practical applications of LLMs on robot manipulation tasks using several VLAs.
Specifically, we evaluate LIBERO~\cite{liu2023libero} robot manipulation benchmarks for X-VLA~\cite{zheng2025x}, $\pi_{0.5}$~\cite{intelligence2025pi05}, and VLA-JEPA~\cite{sun2026vla} foundation models, which use Florence-2~\cite{xiao2024florence}, PaliGemma~\cite{beyer2024paligemma}, and Qwen3-VL~\cite{bai2025qwen3} as VLM backbone, respectively.
Here, we focus on model weight compression via EinSort, rather than KV cache reduction.
We compress all linear modules in attention and multi-layer perceptron (MLP) for the LLM in VLA.

We compare plain SVD, ASVD~\cite{yuan2023asvd}, SVD-LLM~\cite{wang2024svd}, and LatentLLM~\cite{koike2026latentllm} as baselines.
Except plain SVD, we use 64 episodes in LIBERO-Spatial task suite for activation calibration to compute preconditioning.
We evaluate success rates over 800 rollouts on four LIBERO benchmarks: Spatial, Object, Goal, and Long.
Tables~\ref{tab:xvla-libero}, \ref{tab:pi05-libero}, and \ref{tab:jepa-libero} show the success rates for X-VLA, $\pi_{0.5}$, and VLA-JEPA models, respectively.
We verify that our EinSort method significantly outperforms the state-of-the-art weight decomposition methods across all LIBERO benchmarks.
It is interesting to see that EinSort can occasionally perform better than un-compressed VLA models.
It is potentially because of regularization benefit, while we note that the gaps are within the Wilson confidence interval. 
Fig.~\ref{fig:libero} shows example video snapshots using the compressed VLA-JEPA model on LIBERO benchmarks.

\begin{table}[t]
    \centering
    \caption{Success rate ($\uparrow$) of X-VLA model with different quantization methods on LIBERO robot manipulation benchmark at a compression rate of $0.1$.
    \textbf{Bold} and \underline{underline} denote the best and second best, respectively.
    Asterisk ``*'' indicates reaching competitive performance to the original un-compressed VLA.
    }
    \label{tab:xvla-libero}
    \newcolumntype{g}{>{\columncolor{lblue}}r}
    \begin{tabular}{lrrrrgr}
    \toprule
    Benchmark & 
     Spatial &  Object &  Goal & Long & Avg & (95\% CI) \\
    \midrule
    Original &
    95.0\% & 98.5\% & 93.5\% & 87.5\% & 93.63\% & (91.72--95.12\%)
    \\
    \addlinespace
    Plain SVD & 7.5\% & 1.0\% & 0.0\% & 2.0\% & 2.63\% & (1.72--3.98\%)
    \\
    ASVD & 32.5\% & 21.0\% & \underline{5.0}\% & 13.0\% & 17.88\% & (15.38--20.68\%)
    \\
    SVD-LLM & 41.0\% & \underline{26.0}\% & 4.0\% & 31.0\% & 25.50\% & (22.60--28.63\%) %
    \\
    LatentLLM & \underline{43.0}\% & \underline{26.0}\% & 3.5\% & \underline{32.5}\% & \underline{26.25}\% & (23.32--29.41\%) %
    \\
    \rowcolor{lgreen}
    \rowcolor{lgreen}
    \textbf{EinSort} &
    \textbf{93.5}\% & \textbf{97.5}\% & \textbf{90.5}\% & *\textbf{88.0}\% & \textbf{92.38}\% & (90.33--94.02\%)
    \\
    \bottomrule
    \end{tabular}
\end{table}

\begin{table}[t]
    \centering
    \caption{Success rate ($\uparrow$) of $\pi_{0.5}$ VLA model with different quantization methods on LIBERO robot manipulation benchmark at a compression rate of $0.4$. 
    \textbf{Bold} and \underline{underline} denote the best and second best, respectively.
    Asterisk ``*'' indicates reaching competitive performance to the original un-compressed VLA.
    }
    \label{tab:pi05-libero}    
    \newcolumntype{g}{>{\columncolor{lblue}}r}

    \begin{tabular}{lrrrrgr}
    \toprule
    Benchmark & 
     Spatial &  Object &  Goal & Long & Avg & (95\% CI) \\
    \midrule
    Original &
    97.5\% & 100.0\% & 97.0\% & 96.5\% & 97.75\% & (96.47--98.57\%)
    \\
    \addlinespace
    Plain SVD & 27.0\% & 34.5\% & 8.5\% & 1.0\% & 17.75\% & (15.26--20.55\%)
    \\
    ASVD & 4.0\% & 5.5\% & 7.5\% & 0.0\% & 4.25\% & (3.06--5.88\%)
    \\
    SVD-LLM & 60.0\% & 79.0\% & 32.0\% & 27.0\% & 49.50\% & (45.05--52.96\%)
    \\
    LatentLLM & \underline{66.5}\% & \underline{88.5}\% & \underline{37.0}\% & \underline{37.5}\% & \underline{57.38}\% & (53.92--60.76\%)
    \\
    \rowcolor{lgreen}
    \rowcolor{lgreen}
    \textbf{EinSort} &
    \textbf{96.0}\% & \textbf{99.5}\% & *\textbf{97.5}\% & \textbf{95.5}\% & \textbf{97.13}\% & (95.72--98.08\%)
    \\
    \bottomrule
    \end{tabular}    
\end{table}

\begin{table}[t]
    \centering
    \caption{Success rate ($\uparrow$) of VLA-JEPA model with different quantization methods on LIBERO robot manipulation benchmark at a compression rate of $0.8$.
    \textbf{Bold} and \underline{underline} denote the best and second best, respectively.
    Asterisk ``*'' indicates reaching competitive performance to the original un-compressed VLA.
    }
    \label{tab:jepa-libero}
    \newcolumntype{g}{>{\columncolor{lblue}}r}
    \begin{tabular}{lrrrrgr}
    \toprule
    Benchmark & 
     Spatial &  Object &  Goal & Long & Avg & (95\% CI) \\
    \midrule
    Original &
    97.5\% & 100.0\% & 99.0\% & 93.5\% & 97.50\% & (96.17--98.38\%)
    \\
    \addlinespace
    Plain SVD & 0.0\% & 0.0\% & 0.0\% & 0.0\% & 0.00\% & (0.00--0.48\%) %
    \\
    ASVD & 62.0\% & \underline{75.5}\% & 49.5\% & 21.0\% & 52.00\% & (48.54--55.44\%) %
    \\
    SVD-LLM & 18.5\% & 16.0\% & 19.0\% & 0.5\% & 13.50\% & (11.31--16.04\%)
    \\ %
    LatentLLM & \underline{87.0}\% & 64.5\% & \underline{59.0}\% & \underline{28.5}\% & \underline{59.75}\% & (56.31--63.09\%) %
    \\
    \rowcolor{lgreen}
    \rowcolor{lgreen}
    \textbf{EinSort} &
    *\textbf{98.0}\% & *\textbf{100.0}\% & \textbf{98.0}\% & *\textbf{95.0}\% & *\textbf{97.75}\% & (96.47--98.57\%) %
    \\
    \bottomrule
    \end{tabular}
\end{table}

\begin{figure}[t]
    \centering
    \begin{subfigure}{\linewidth}
        \includegraphics[width=\linewidth]{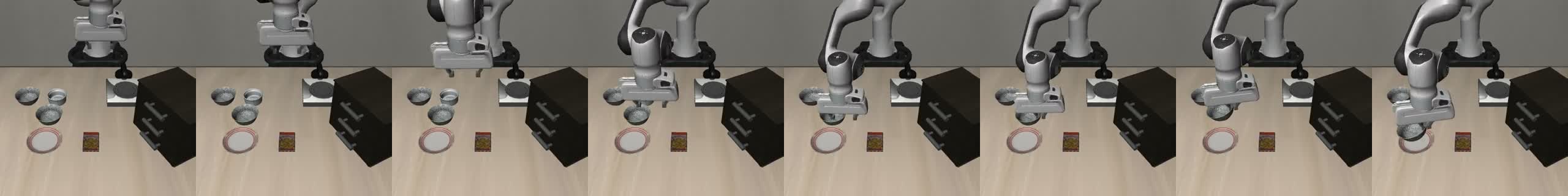}
        \caption{LIBERO-Spatial task 0: ``pick up the black bowl between the plate and the ramekin and place it on the plate''}
    \end{subfigure}
    \begin{subfigure}{\linewidth}
        \includegraphics[width=\linewidth]{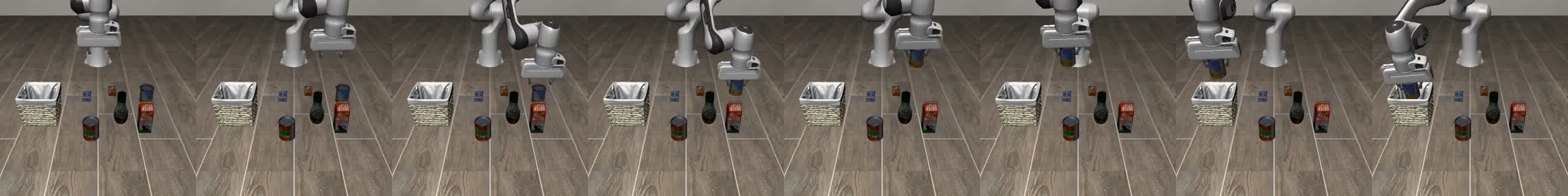}
        \caption{LIBERO-Object task 0: ``pick up the alphabet soup and place it in the basket''}
    \end{subfigure}
    \begin{subfigure}{\linewidth}
        \includegraphics[width=\linewidth]{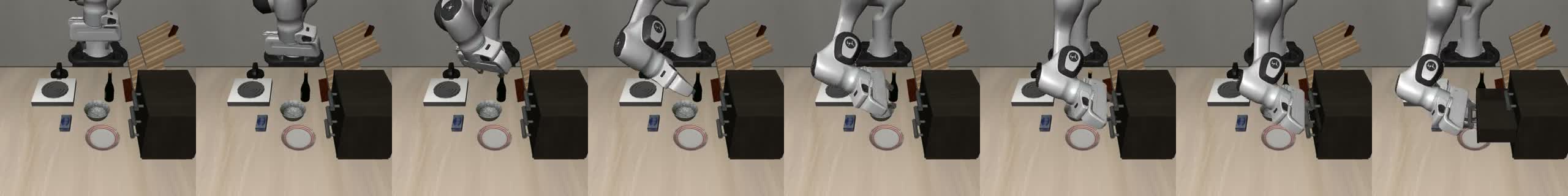}
        \caption{LIBERO-Goal task 0: ``open the middle drawer of the cabinet''}
    \end{subfigure}
    \begin{subfigure}{\linewidth}
        \includegraphics[width=\linewidth]{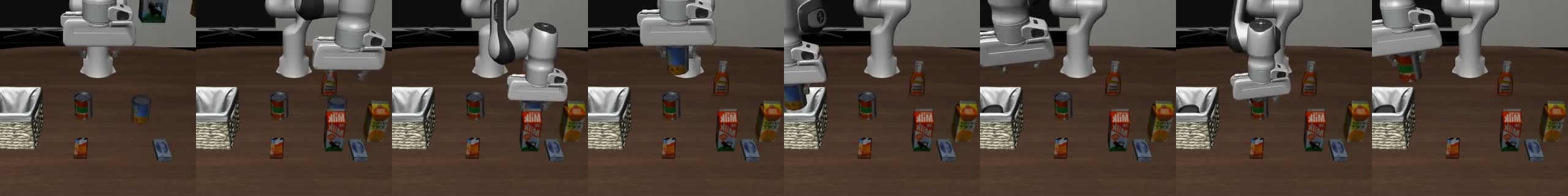}
        \caption{LIBERO-Long task 0: ``put both the alphabet soup and the tomato sauce in the basket''}
    \end{subfigure}
    \caption{Robot manipulation video snapshots for the compressed VLA-JEPA model on LIBERO benchmarks.}
    \label{fig:libero}
\end{figure}

\section{Large foundation models}
\label{sec:model}

\paragraph{Qwen3}
We use Qwen3~\citep{yang2025qwen3} dense models, which are decoder-only transformers spanning 270M to 30B parameters, built on a consistent architecture with RMSNorm, SwiGLU feed-forward layers, and rotary positional embeddings. 
All variants employ grouped-query attention (GQA) with a fixed small number of key–value heads while scaling the number of query heads with model width, reducing KV-cache cost. 
Importantly, the hidden size is decoupled from the attention projection width, providing additional flexibility.
Parameters are listed in Table~\ref{tab:qwen3-dense}.
It is under Apache-2.0 license.

\paragraph{Gemma3}
Gemma3 models~\citep{team2025gemma} are decoder-only transformer architectures released across a wide range of scales, from 270M to 27B parameters, and include both text-only and multimodal variants. 
Similar to Qwen3, all Gemma3 models adopt RMSNorm, SwiGLU feed-forward networks, and rotary positional embeddings, as well as grouped-query attention (GQA). 
The per-head dimension is fixed at 256 across model sizes. 
Parameters are listed in Table~\ref{tab:gemma3-it}.
It is licensed under the Gemma terms of use.

\paragraph{Phi-3}
Phi-3 models~\citep{haider2024phi} are language models developed for efficient on-device and edge AI. 
Phi-3 models achieve strong reasoning and coding performance through high-quality training data and optimized transformer architectures.
Parameters are listed in Table~\ref{tab:phi3}.
It is under MIT license.

\paragraph{Phi-4}
Phi-4 models~\citep{abdin2024phi} are a family of compact language and multimodal foundation models for efficient reasoning, coding, and on-device AI. 
The family includes text-only, reasoning-specialized~\cite{aneja2026phi}, and multimodal variants~\cite{abouelenin2025phi} supporting language, vision, and audio inputs, while maintaining strong efficiency relative to model size. 
Parameters are listed in Table~\ref{tab:phi4}. 
It is released under the MIT license.

\begin{table}[t]
\centering
\caption{Architecture parameters of Qwen3 dense models~\cite{yang2025qwen3}}
\label{tab:qwen3-dense}
\small
\begin{tabular}{rrrrrrrl}
\toprule
Model
& layers
& heads
& KV heads
& hidden size
& head dim
& MLP dim 
& Huggingface ID \\
\midrule
0.6B
& 28 & 16 & 8
& 1024 & 128 
& 3072 &
\href{https://huggingface.co/Qwen/Qwen3-0.6B}
  {Qwen/Qwen3-0.6B}\\

1.7B
& 28 & 16 & 8
& 2048 & 128 
& 6144 &
\href{https://huggingface.co/Qwen/Qwen3-1.7B}
  {Qwen/Qwen3-1.7B}\\

4B
& 36 & 32 & 8
& 2560 & 128 
& 9728 &
\href{https://huggingface.co/Qwen/Qwen3-4B}
  {Qwen/Qwen3-4B}\\

8B
& 36 & 32 & 8
& 4096 & 128 
& 12288 &
\href{https://huggingface.co/Qwen/Qwen3-8B}
  {Qwen/Qwen3-8B}\\

14B
& 40 & 40 & 8
& 5120 & 128 
& 17408 &
\href{https://huggingface.co/Qwen/Qwen3-14B}
  {Qwen/Qwen3-14B}\\

32B
& 64 & 64 & 8
& 5120 & 128 
& 25600 &
\href{https://huggingface.co/Qwen/Qwen3-32B}
  {Qwen/Qwen3-32B}\\
\bottomrule
\end{tabular}
\end{table}

\begin{table}[t]
\centering
\caption{Gemma3 instruction-tuned text transformer parameters~\cite{team2025gemma}}
\label{tab:gemma3-it}
\small
\begin{tabular}{rrrrrrrl}
\toprule
Model
& layers
& heads
& KV heads
& hidden size
& head dim
& MLP dim
& Huggingface ID \\
\midrule
270M
& 18 & 4 & 1
& 640 & 256 & 2048
& \href{https://huggingface.co/google/gemma-3-270m-it}
  {google/gemma-3-270m-it} \\

1B
& 26 & 4 & 1
& 1152 & 256 & 6912
& \href{https://huggingface.co/google/gemma-3-270m-it}
  {google/gemma-3-1b-it} \\

4B
& 34 & 8 & 4
& 2560 & 256 & 10240
& \href{https://huggingface.co/google/gemma-3-270m-it}
  {google/gemma-3-4b-it} \\

12B
& 48 & 16 & 8
& 3840 & 256 & 15360
& \href{https://huggingface.co/google/gemma-3-270m-it}
  {google/gemma-3-12b-it} \\

27B
& 62 & 32 & 16
& 5376 & 256 & 21504
& \href{https://huggingface.co/google/gemma-3-270m-it}
  {google/gemma-3-27b-it} \\
\bottomrule
\end{tabular}
\end{table}

\begin{table}[t]
\centering
\caption{Phi-3 instruction-tuned text transformer parameters~\cite{haider2024phi}}
\label{tab:phi3}
\small
\begin{tabular}{rrrrrrrl}
\toprule
Model
& layers
& heads
& KV heads
& hidden size
& head dim
& MLP dim
& Huggingface ID \\
\midrule
mini (3.8B)
& 32 & 32 & 32
& 3072 & 96 & 8192
& \href{https://huggingface.co/microsoft/Phi-3.5-mini-instruct}
  {microsoft/Phi-3.5-mini-instruct} \\

small (7B)
& 32 & 32 & 8
& 4096 & 128 & ---
& \href{https://huggingface.co/microsoft/Phi-3-small-4k-instruct}
  {microsoft/Phi-3-small-8k-instruct} \\

medium (14B)
& 40 & 40 & 10
& 5120 & 128 & 17920
& \href{https://huggingface.co/microsoft/Phi-3-medium-4k-instruct}
  {microsoft/Phi-3-medium-4k-instruct} \\
\bottomrule
\end{tabular}
\end{table}

\begin{table}[t]
\centering
\caption{Phi-4 instruction-tuned text transformer parameters~\cite{abouelenin2025phi,aneja2026phi}}
\label{tab:phi4}
\small
\begin{tabular}{rrrrrrrl}
\toprule
Model
& layers
& heads
& KV heads
& hidden size
& head dim
& MLP dim
& Huggingface ID \\
\midrule
mini (3.8B)
& 32 & 32 & 8
& 3072 & 128 & 8192
& \href{https://huggingface.co/microsoft/Phi-4-mini-instruct}
  {microsoft/Phi-4-mini-instruct} \\

reasoning (14B)
& 40 & 40 & 10
& 5120 & 128 & 17920
& \href{https://huggingface.co/microsoft/Phi-4-reasoning}
  {microsoft/Phi-4-reasoning} \\
\bottomrule
\end{tabular}
\end{table}

\paragraph{Qwen3-VL}
Qwen3-VL~\cite{bai2025qwen3} is a family of multimodal LLMs that extend the Qwen3 transformer with vision–language capabilities. 
The models integrate a SigLIP-based vision encoder~\cite{zhai2023sigmoid} with the Qwen3 language backbone~\cite{yang2025qwen3} through a projection module, enabling joint reasoning over images, videos, and text for tasks such as visual question answering, captioning, and document understanding.
Parameters are listed in Table~\ref{tab:qwen3vl}.
It is licensed under Apache-2.0.

\begin{table}[t]
\centering
\caption{Qwen3-VL transformer parameters~\cite{bai2025qwen3}}
\label{tab:qwen3vl}
\small
\setlength{\tabcolsep}{4pt}
\begin{tabular}{rrrrrrrl}
\toprule
Model
& layers
& heads
& KV heads
& hidden size
& head dim
& MLP dim
& Hugginface ID
\\
\midrule
2B	& 24 & 16 & 8 &
2048 & 128 & 11008 &
\href{https://huggingface.co/Qwen/Qwen3-VL-2B-Instruct}
  {Qwen/Qwen3-VL-2B-Instruct}
\\
4B & 32 & 32 & 8 &
4096 & 128 & 22016 &
\href{https://huggingface.co/Qwen/Qwen3-VL-4B-Instruct}
  {Qwen/Qwen3-VL-4B-Instruct}
\\
8B & 36 & 32 & 8 &
4096 & 128 & 22016 &
\href{https://huggingface.co/Qwen/Qwen3-VL-8B-Instruct}
  {Qwen/Qwen3-VL-8B-Instruct}
\\
32B & 64 & 40 & 8 &
5120 & 128 & 27392 &
\href{https://huggingface.co/Qwen/Qwen3-VL-32B-Instruct}
  {Qwen/Qwen3-VL-32B-Instruct}
\\
\bottomrule
\end{tabular}
\end{table}

\paragraph{X-VLA}
X-VLA (Cross-modal Vision–Language–Action)~\cite{zheng2025x} is a unified framework that jointly models visual perception, language understanding, and action generation within a shared representation space. 
Leveraging multimodal pretraining, X-VLA captures compositional and temporal structure, allowing robust generalization across tasks and environments for embodied AI systems.
It uses Florence-2~\cite{xiao2024florence} as a VLM backbone, which uses DaViT-B vision encoder~\cite{ding2022davit}.
The parameter is listed in Table~\ref{tab:xvla}.
It is released under the Apache-2.0 license.

\paragraph{$\boldsymbol{\pi_{0.5}}$}

The $\pi_{0.5}$ model~\cite{intelligence2025pi05} is the state-of-the-art VLA transformer having 2.3B parameters, that maps visual observations and language instructions directly to continuous robot actions through flow-matching diffusion policy. 
It uses PaliGemma~\cite{beyer2024paligemma} as a VLM backborne, and Gemma~\cite{team2024gemma} as a flow-matching diffusion policy.
Trained via distillation from a larger foundation model on diverse robot interaction data, it achieves strong zero-shot generalization while remaining lightweight and deployable for real-world manipulation tasks.
Parameters are listed in Table~\ref{tab:pi05}.
It is licensed under the Gemma terms of use.

\paragraph{VLA-JEPA}
VLA-JEPA (Vision--Language--Action with Joint Embedding Predictive Architecture)~\cite{sun2026vla} augments VLA policies with a latent predictive world model based on V-JEPA2~\cite{assran2025v}, enabling the agent to learn action-relevant future representations without reconstructing pixels. 
Built on a Qwen3-VL-2B VLM backbone~\cite{bai2025qwen3} and a V-JEPA2 latent encoder based on ViT-L, it predicts future latent states and generates robot actions using a diffusion policy based on DiT-B~\cite{peebles2023scalable}. 
By jointly learning perception, language grounding, and latent dynamics, VLA-JEPA achieves strong generalization and robustness across manipulation tasks while improving data efficiency and transfer to unseen environments. 
Parameters are listed in Table~\ref{tab:vlajepa}.
It is released under the Apache-2.0 license.

\begin{table}[t]
\centering
\caption{X-VLA transformer parameters~\cite{zheng2025x}:
Huggingface ID \href{https://huggingface.co/lerobot/xvla-libero}
{lerobot/xvla-libero}}
\label{tab:xvla}
\small
\begin{tabular}{lrrrrrr}
\toprule
Module
& layers
& heads
& KV heads
& hidden size
& head dim
& MLP dim
\\
\midrule
VLM: Florence-2-Large
& 24 & 16 & 16
& 1024 & 64 & 4096
\\

Vision Encoder: DaViT
& 12 & 64 & 64
& 2048 & 32 & 8192
\\

\bottomrule
\end{tabular}
\end{table}

\begin{table}[t]
\centering
\caption{$\pi_{0.5}$ VLA transformer parameters~\cite{intelligence2025pi05}:
Huggingface ID \href{https://huggingface.co/lerobot/pi05_libero_finetuned}
  {lerobot/pi05\_libero\_finetuned}}
\label{tab:pi05}
\small
\begin{tabular}{lrrrrrr}
\toprule
Module
& layers
& heads
& KV heads
& hidden size
& head dim
& MLP dim
\\
\midrule
LLM: Gemma-2B
& 18 & 8 & 1
& 2048 & 256 & 16384
 \\

Vision Encoder: SigLIP ViT-L
& 24 & 16 & 16
& 1024 & 64 & 4096
\\
Diffusion Policy: Gemma-300M
&
18 & 8 & 1 &
1024 & 128 & 4096
\\
\bottomrule
\end{tabular}
\end{table}

\begin{table}[t]
\centering
\caption{VLA-JEPA transformer parameters~\cite{sun2026vla}:
Huggingface ID \href{https://huggingface.co/lerobot/VLA-JEPA-LIBERO}
{lerobot/VLA-JEPA-LIBERO}}
\label{tab:vlajepa}
\small
\begin{tabular}{lrrrrrr}
\toprule
Module
& layers
& heads
& KV heads
& hidden size
& head dim
& MLP dim
\\
\midrule
VLM: Qwen3-VL-2B
& 36 & 16 & 8
& 1536 & 96 & 8960
\\

World model: V-JEPA2 ViT-L
& 24 & 16 & 16
& 1024 & 64 & 4096
\\

Diffusion Policy: DiT-B
& 12 & 12 & 12
& 768 & 64 & 3072
\\
\bottomrule
\end{tabular}
\end{table}

\section{Datasets}
\label{sec:dataset}

\paragraph{Wikitext-2 (WT2)}
The WikiText language modeling dataset~\citep{merity2016pointer} is a collection of over 100 million tokens extracted from the set of verified good and featured articles on Wikipedia. 
The dataset is available under the CC BY-SA-4.0 license.
The wikitext-2-raw-v1 contains 36{,}718, 3{,}760, and 4{,}358 samples for train, validation, and test splits, respectively.
We use \url{https://huggingface.co/datasets/mindchain/wikitext2}.

\paragraph{GSM8K}
The Grade School Math 8K (GSM8K) dataset~\citep{cobbe2021training} is a benchmark for evaluating mathematical reasoning in LLMs.
It consists of 8{,}500 high-quality grade-school-level word problems written by human annotators, each paired with a detailed natural-language solution and a final numeric answer. 
The dataset is designed to assess multi-step arithmetic reasoning rather than factual recall, requiring models to generate intermediate reasoning steps to solve problems correctly. 
GSM8K is released under the MIT license and contains 7{,}473 training samples and 1{,}319 test samples. 
We use \url{https://huggingface.co/datasets/openai/gsm8k}.

\paragraph{TextVQA}

TextVQA~\cite{singh2019towards} requires VLM models to read and reason about text in images to answer questions about them. 
Specifically, models need to incorporate the new modality of text present in the images and reason over it to answer TextVQA questions. 
TextVQA dataset contains 45{,}336 questions over 28{,}408 images from the OpenImages dataset. 
We use \url{https://huggingface.co/datasets/lmms-lab/textvqa}, licensed under CC-BY-4.0.

\paragraph{LIBERO}
The LIBERO dataset~\cite{liu2023libero} is a benchmark for robotic vision-language-action (VLA) learning that evaluates long-horizon, compositional manipulation in simulated household environments. 
It provides five benchmark suites: spatial; object; goal; long, and short.
The short task suite has 90 task variants and the other suites have 10 different tasks. 
Each task includes multimodal data---video observations, language instructions, and low-level control trajectories---enabling end-to-end learning from vision and language to robot actions.
We use \url{https://huggingface.co/datasets/lerobot/libero},
 licensed under Apache-2.0.

\section{Libraries}
\label{sec:library}

We partly use some public libraries as below.

\begin{itemize}
    \item Pytorch \url{https://github.com/pytorch/pytorch}; version 2.9.1; BSD-3-Clause license
    \item opt\_einsum \url{https://github.com/dgasmith/opt_einsum}; version 3.3.0; MIT license
    \item einops \url{https://github.com/arogozhnikov/einops}; version 0.8.2; MIT license
    \item tensorly \url{https://github.com/tensorly/tensorly}; version 0.9.0; BSD-3 Clause license
    \item transformers \url{https://github.com/huggingface/transformers}; version 5.8.0; Apache-2.0 license
    \item datasets \url{https://github.com/huggingface/datasets}; version 4.85; Apache-2.0 license
    \item SciPy \url{https://github.com/scipy/scipy}; version 1.17.1; BSD-3-Clause license
    \item PyEDA \url{https://github.com/cjdrake/pyeda}; version 0.29.0; BSD-2-Clause license
    \item RevKit \url{https://github.com/msoeken/revkit}; MIT license
    \item Galois \url{https://github.com/mhostetter/galois}; version 0.4.11; MIT license
    \item LeRobot \url{https://github.com/huggingface/lerobot}; version 0.5.1; Apache-2.0 license

\end{itemize}

\end{document}